%% file: main.tex
\documentclass{article}
\usepackage{microtype}
\usepackage{amsmath,amssymb,booktabs,graphicx,tabularx,adjustbox,multirow}
\usepackage[hidelinks]{hyperref}
\usepackage{xurl}
\usepackage[accepted]{icml2025}
\renewcommand{\algorithmicrequire}{\textbf{Input.}}
\usepackage{xcolor}
\usepackage{amsthm}
\usepackage{caption,subcaption}
\usepackage{needspace,placeins}
\usepackage[framemethod=default]{mdframed}
\definecolor{examplecolor}{HTML}{199D9A}
\definecolor{caseblue}{HTML}{607FA8}
\definecolor{casegreen}{HTML}{199D9A}
\definecolor{codepink}{HTML}{C23C7B}
\newmdenv[linecolor=examplecolor,backgroundcolor=examplecolor!4,linewidth=0.6pt,innerleftmargin=10pt,innerrightmargin=10pt,innertopmargin=8pt,innerbottommargin=8pt,skipabove=12pt,skipbelow=12pt,frametitlefont=\bfseries,frametitlebackgroundcolor=examplecolor!10]{examplebox}
\newtheorem{proposition}{Proposition}[section]
\usepackage[T1]{fontenc}
\usepackage[utf8]{inputenc}

\setcitestyle{authoryear,round,citesep={,},aysep={,},yysep={,}}
\icmltitlerunning{JevSpawn: Adaptive Agentic Inference through Compositional Action Spaces}
\hypersetup{pdftitle={JevSpawn: Adaptive Agentic Inference through Compositional Action Spaces},pdfauthor={Haoyang Su, Weiran Huang},pdfsubject={Preprint},hidelinks}
\begin{document}
\twocolumn[{
\icmltitle{JevSpawn: Adaptive Agentic Inference through Compositional Action Spaces}
\begin{icmlauthorlist}
\begin{tabular}{@{}c@{\hspace{2.5em}}c@{}}
\icmlauthor{Haoyang Su}{} & \icmlauthor{Weiran Huang}{} \\
{\small Fudan University} & {\small Shanghai Jiao Tong University} \\
{\small Shanghai Innovation Institute} & {\small Shanghai Innovation Institute}
\end{tabular}
\end{icmlauthorlist}
\vskip 0.12in
\begin{minipage}{\textwidth}
\captionsetup{type=figure,hypcap=false}
\centering
\begin{subfigure}[t]{0.55\linewidth}
\centering
\includegraphics[width=\linewidth]{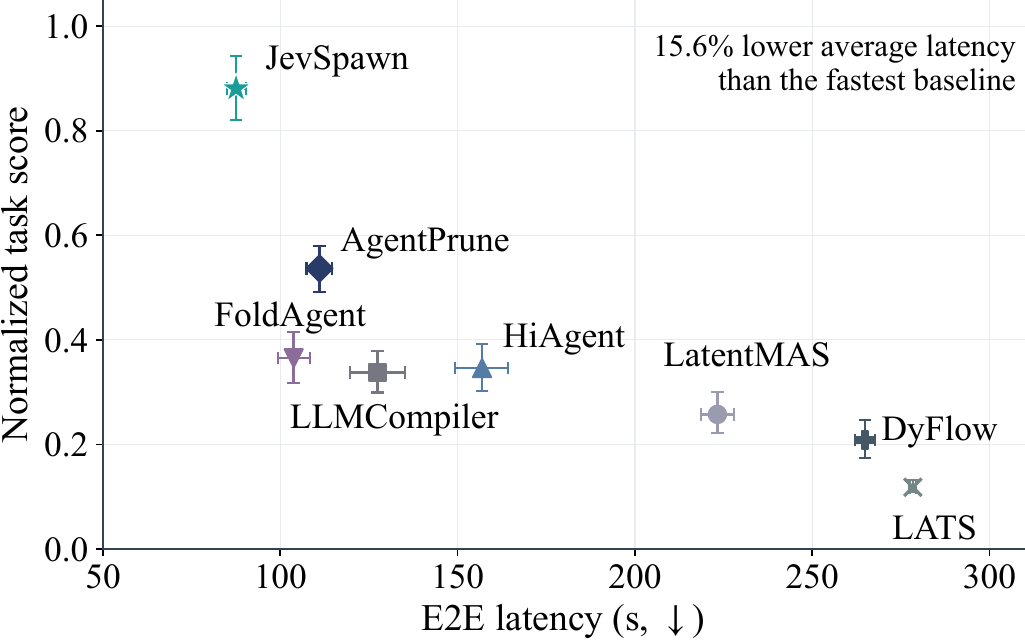}
\caption{Quality and latency}
\end{subfigure}
\hfill
\begin{subfigure}[t]{0.42\linewidth}
\centering
\includegraphics[width=\linewidth]{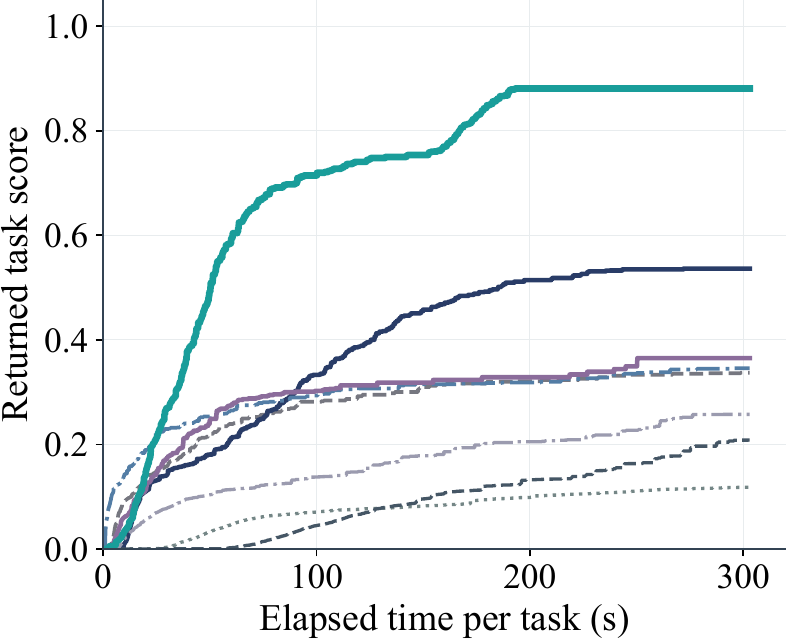}
\caption{Task scores over time}
\end{subfigure}
\par\medskip
\includegraphics[scale=0.625000]{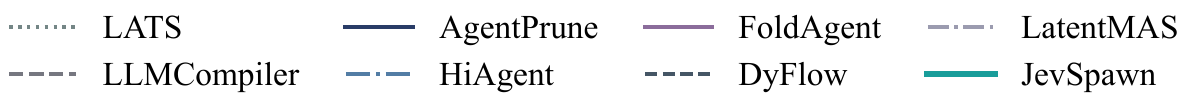}
\caption{Quality and latency across eight benchmark tasks for JevSpawn and seven agent baselines. (a) Average quality and latency with 95\% paired bootstrap intervals. (b) Scores accumulated as tasks finish.}
\label{fig:overview}
\end{minipage}
\vskip 0.15in
}]

\begin{abstract}
LLM agents generate intermediate reasoning and actions token by token, making extended interactions slow and computationally expensive. Jev-style models offer fast probabilistic predictions over finite fields, but require those fields to be specified in advance. This requirement limits autonomous task solving, where the available actions must be derived from natural language instructions and adapted through interaction. We introduce JevSpawn, a compositional policy that connects natural language task specifications to finite probabilistic exploration. Parallel action spawning is coupled with feedback driven branch selection, representation revision, and recovery from retained alternatives. Shared action structure and model prefixes reduce repeated generation and context computation without additional training. Evaluations on eight benchmark tasks against seven agent baselines and a TypeSafe Jev variant establish JevSpawn as a promising approach to structured agentic inference, with improved task performance and faster navigation. Code is available at \textcolor{codepink}{\url{https://github.com/Hoyant-Su/JevSpawn}}.
\end{abstract}

\section{Introduction}

LLM agents solve tasks through tool use and repeated interaction with an environment \citep{yao2023react,NEURIPS2023_1b44b878,Guo_2024}. Reasoning and actions are generated token by token, while observations enlarge the context for later steps. Exploring alternative actions adds further cost, limiting both the breadth and duration of interaction under a fixed inference budget.

Existing methods use this budget more effectively by organizing model calls and reducing intermediate information \citep{https://doi.org/10.48550/arxiv.2601.14192,https://doi.org/10.48550/arxiv.2502.14321}. Search favors promising trajectories, while workflow optimization removes redundant calls and parallelizes independent tools \citep{zhou2024lats,kim2024llmcompiler,ICLR2025_5492ecbc,ICLR2025_36b7acf6,wang2025dyflow}. Communication pruning, memory management, and latent exchange reduce the information passed between steps \citep{zhang2025agentprune,wang-etal-2025-agentdropout,hu-etal-2025-hiagent,sun2026scaling,zou2026latent}. Further savings are possible within the actions themselves, since alternatives often share a form and differ in only a few arguments.

To exploit this shared structure, repeated action descriptions can be replaced by predictions over finite fields. Constrained decoding enforces output rules and can skip deterministic continuations \citep{Beurer_Kellner_2023,geng-etal-2023-grammar,MLSYS2025_5c20ca4b,arxiv231207104}. TypeSafe Jev returns typed values and probabilities for supplied questions without generating response strings \citep{almeida2026jev}. The application must, however, define the fields in advance. For autonomous agents, suitable action fields must be inferred from instructions and revised as new observations arrive. The central challenge is to retain the speed of finite prediction while allowing the action space to change during interaction.

JevSpawn addresses this challenge by making the action fields part of the policy. Field values are combined into executable actions, and joint probabilities guide parallel spawning. Execution feedback guides branch selection and field revision, while retained branches support recovery from unsuccessful choices. Shared action structure and model context reduce repeated generation and computation. The resulting policy extends Jev-style prediction to sustained task solving without additional training.

Our contributions are threefold. First, a compositional action representation connects natural language goals and interaction rules to executable finite fields, extending Jev-style prediction to tasks including spatial navigation. Second, a unified Jev-style architecture couples parallel spawning with feedback driven state transitions, representation revision, and recovery through retained branches. Third, parallel causal evaluation and shared prefix reuse provide a common computational basis for the finite policy, reducing repeated action generation and context processing.

\input{related}
\input{method}
\input{experiments}

\section{Conclusion}

JevSpawn extends Jev-style prediction to agentic inference by making the finite action space adaptive. A compositional policy couples parallel spawning with feedback driven exploration, representation revision, and branch recovery. Shared structure and prefix computation reduce the cost of evaluating alternatives without additional training. Experiments demonstrate higher scores across several tasks and reduced latency in navigation.

\bibliographystyle{references}
\bibliography{references}
\clearpage
\onecolumn
\raggedbottom
\appendix
\input{derivations}
\section{Experimental Details}
\input{setup}

\input{ablations}
\clearpage
\input{cases}
\end{document}

%% file: related.tex
\section{Related Work}

\label{sec:related}

\subsection{Efficient Agentic Inference}

Efficient agentic inference depends on both the computation requested by an agent and the cost of executing those requests \citep{Guo_2024,https://doi.org/10.48550/arxiv.2601.14192,https://doi.org/10.48550/arxiv.2502.14321}. At the agent level, existing methods organize exploration or reduce intermediate information. Search and workflow optimization determine where model calls are spent \citep{yao2023treeofthoughts,ICLR2025_5492ecbc,ICLR2025_36b7acf6}. LATS uses environment feedback to guide search \citep{zhou2024lats}, LLMCompiler parallelizes tool calls \citep{kim2024llmcompiler}, and DyFlow adapts workflows during execution \citep{wang2025dyflow}. Fewer or better scheduled calls reduce wasted computation, but each interaction still carries messages and history needed for later actions.

Reducing this information burden is the focus of communication and memory management \citep{https://doi.org/10.48550/arxiv.2308.08155,https://doi.org/10.48550/arxiv.2402.14034,https://doi.org/10.48550/arxiv.2304.03442,packer2024memgptllmsoperatingsystems,NEURIPS2023_1b44b878,https://doi.org/10.48550/arxiv.2409.07429}. AgentPrune removes communication edges \citep{zhang2025agentprune}, while AgentDropout removes unnecessary agents \citep{wang-etal-2025-agentdropout}. HiAgent organizes subgoal memory \citep{hu-etal-2025-hiagent}, and Context Folding summarizes history \citep{sun2026scaling}. LatentMAS further reduces reliance on text by exchanging latent representations \citep{zou2026latent}. Together, call scheduling and information reduction determine how much work reaches the model. The cost of that remaining work depends on how context is processed and outputs are generated.

At this computational level, inference methods accelerate individual model requests. Transformer inference \citep{https://doi.org/10.48550/arxiv.1706.03762} benefits from efficient attention \citep{https://doi.org/10.48550/arxiv.2205.14135,https://doi.org/10.48550/arxiv.2307.08691,https://doi.org/10.48550/arxiv.2407.08608,https://doi.org/10.48550/arxiv.2501.01005}, compact key and value representations \citep{https://doi.org/10.48550/arxiv.1911.02150,https://doi.org/10.48550/arxiv.2305.13245}, and improved scheduling and memory management \citep{https://doi.org/10.48550/arxiv.2401.09670,https://doi.org/10.48550/arxiv.2311.18677,https://doi.org/10.48550/arxiv.2308.16369,https://doi.org/10.48550/arxiv.2403.02310,https://doi.org/10.48550/arxiv.2405.04437}. Partitioning, offloading, and activation sparsity address additional resource constraints \citep{https://doi.org/10.48550/arxiv.2207.00032,https://doi.org/10.48550/arxiv.2303.06865,https://doi.org/10.48550/arxiv.2312.12456}. For repeated agent interactions, these per-request improvements can be extended by reusing the context shared across successive calls.

Prefix and text caching provide such reuse \citep{https://doi.org/10.48550/arxiv.2405.16444,https://doi.org/10.48550/arxiv.2311.04934,https://doi.org/10.48550/arxiv.2402.05099,https://doi.org/10.48550/arxiv.2402.15220}, with prefix aware scheduling extending reuse across workers \citep{https://doi.org/10.48550/arxiv.2407.00023}. Cache compression and selective retention or retrieval reduce the cost of accumulated history \citep{https://doi.org/10.48550/arxiv.2309.17453,https://doi.org/10.48550/arxiv.2306.14048,https://doi.org/10.48550/arxiv.2305.17118,https://doi.org/10.48550/arxiv.2404.14469,https://doi.org/10.48550/arxiv.2406.02069,https://doi.org/10.48550/arxiv.2406.10774,https://doi.org/10.48550/arxiv.2402.02750,https://doi.org/10.48550/arxiv.2405.14366,https://doi.org/10.48550/arxiv.2310.07240,https://doi.org/10.48550/arxiv.2404.04793,https://doi.org/10.48550/arxiv.2310.01801,https://doi.org/10.48550/arxiv.2312.04985}, while alternative context models use recurrence, compression, sparse attention, or linear attention \citep{https://doi.org/10.48550/arxiv.1901.02860,https://doi.org/10.48550/arxiv.1911.05507,https://doi.org/10.48550/arxiv.2004.05150,https://doi.org/10.48550/arxiv.2007.14062,https://doi.org/10.48550/arxiv.2001.04451,https://doi.org/10.48550/arxiv.2009.14794,https://doi.org/10.48550/arxiv.2006.04768,https://doi.org/10.48550/arxiv.2006.16236}. These methods reduce the cost of shared inputs, but exploring alternative actions can still require repeated output generation. JevSpawn targets this remaining cost by reusing compositional action fields across branches and selecting field values through finite probabilities.

\subsection{Constrained Decoding}

Replacing repeated action generation requires a representation that preserves executable outputs. Constrained decoding provides a basis for such representations through fixed or input dependent grammars \citep{Beurer_Kellner_2023,geng-etal-2023-grammar}. Syntactic and semantic constraints improve program validity \citep{https://doi.org/10.48550/arxiv.2109.05093,https://doi.org/10.48550/arxiv.2201.11227}, although valid structure alone does not ensure accurate predictions \citep{https://doi.org/10.48550/arxiv.2501.10868,https://doi.org/10.48550/arxiv.2408.02442}. Formatting, option order, and context position can still affect the selected output \citep{https://doi.org/10.48550/arxiv.2310.11324,https://doi.org/10.48550/arxiv.2307.03172,https://doi.org/10.48550/arxiv.2102.09690,https://doi.org/10.48550/arxiv.2104.08786,https://doi.org/10.48550/arxiv.2309.03882}. A useful structured policy must therefore preserve both output validity and meaningful choices within the permitted space.

Meeting these requirements during decoding raises two further issues, the cost of enforcing constraints and the effect of constraints on output probabilities. XGrammar largely hides this overhead through cached token checks and overlap with GPU execution \citep{MLSYS2025_5c20ca4b}, while SGLang skips deterministic grammar paths and reuses context \citep{arxiv231207104}. Constraint overhead is reduced, while distributional fidelity requires separate treatment. DOMINO preserves valid continuations across tokenization boundaries \citep{pmlr-v235-beurer-kellner24a}, and grammar aligned decoding corrects distributional distortion from local token masking \citep{NEURIPS2024_2bdc2267}. Even with efficient enforcement and improved fidelity, choices not fixed by the grammar are still generated sequentially.

This remaining sequential cost motivates methods that produce or verify several tokens at once. Speculative decoding accelerates generation through parallel verification \citep{https://doi.org/10.48550/arxiv.2211.17192,https://doi.org/10.48550/arxiv.2302.01318,https://doi.org/10.48550/arxiv.2503.01840,https://doi.org/10.48550/arxiv.2402.02057,https://doi.org/10.48550/arxiv.2305.09781}, using drafts from auxiliary models, prediction heads, retrieval, or earlier layers \citep{https://doi.org/10.48550/arxiv.2401.10774,https://doi.org/10.48550/arxiv.2401.15077,https://doi.org/10.48550/arxiv.2406.16858,https://doi.org/10.48550/arxiv.2311.08252,https://doi.org/10.48550/arxiv.2309.08168,https://doi.org/10.48550/arxiv.2404.16710,https://doi.org/10.48550/arxiv.2402.11131,https://doi.org/10.48550/arxiv.2403.09919}. Nonautoregressive and insertion based models change the output factorization \citep{https://doi.org/10.48550/arxiv.1711.02281,https://doi.org/10.48550/arxiv.1904.09324,https://doi.org/10.48550/arxiv.1905.11006,https://doi.org/10.48550/arxiv.1902.03249,https://doi.org/10.48550/arxiv.1808.08583}, while diffusion models refine sequences iteratively \citep{https://doi.org/10.48550/arxiv.2205.14217,https://doi.org/10.48550/arxiv.2310.16834,https://doi.org/10.48550/arxiv.2406.07524,https://doi.org/10.48550/arxiv.2502.09992}. These methods accelerate the production of output sequences. When the possible values are already known, the task can instead be expressed as finite prediction over those alternatives.

TypeSafe Jev provides this finite interface by evaluating supplied questions in parallel and returning typed values with probabilities \citep{almeida2026jev}. Response strings need not be generated, but the application must define the fields being predicted. JevSpawn addresses this dependence on predefined fields by making the action space part of the policy. Compositional fields are inferred from task context and revised through feedback, coupling finite prediction with parallel spawning and recovery across interaction rounds.

%% file: method.tex
\section{Method}
\label{sec:method}

JevSpawn extends Jev-style prediction to agentic tasks by allowing the model to define and revise finite action fields. Spawning explores combinations of field values in parallel, with actions selected by joint probability. Execution feedback guides branch selection and field revision. Figure~\ref{fig:method} presents the method.

\begin{figure*}[t]
\centering
\includegraphics[width=\linewidth,trim={33.75bp 20.21bp 33.75bp 36.75bp},clip]{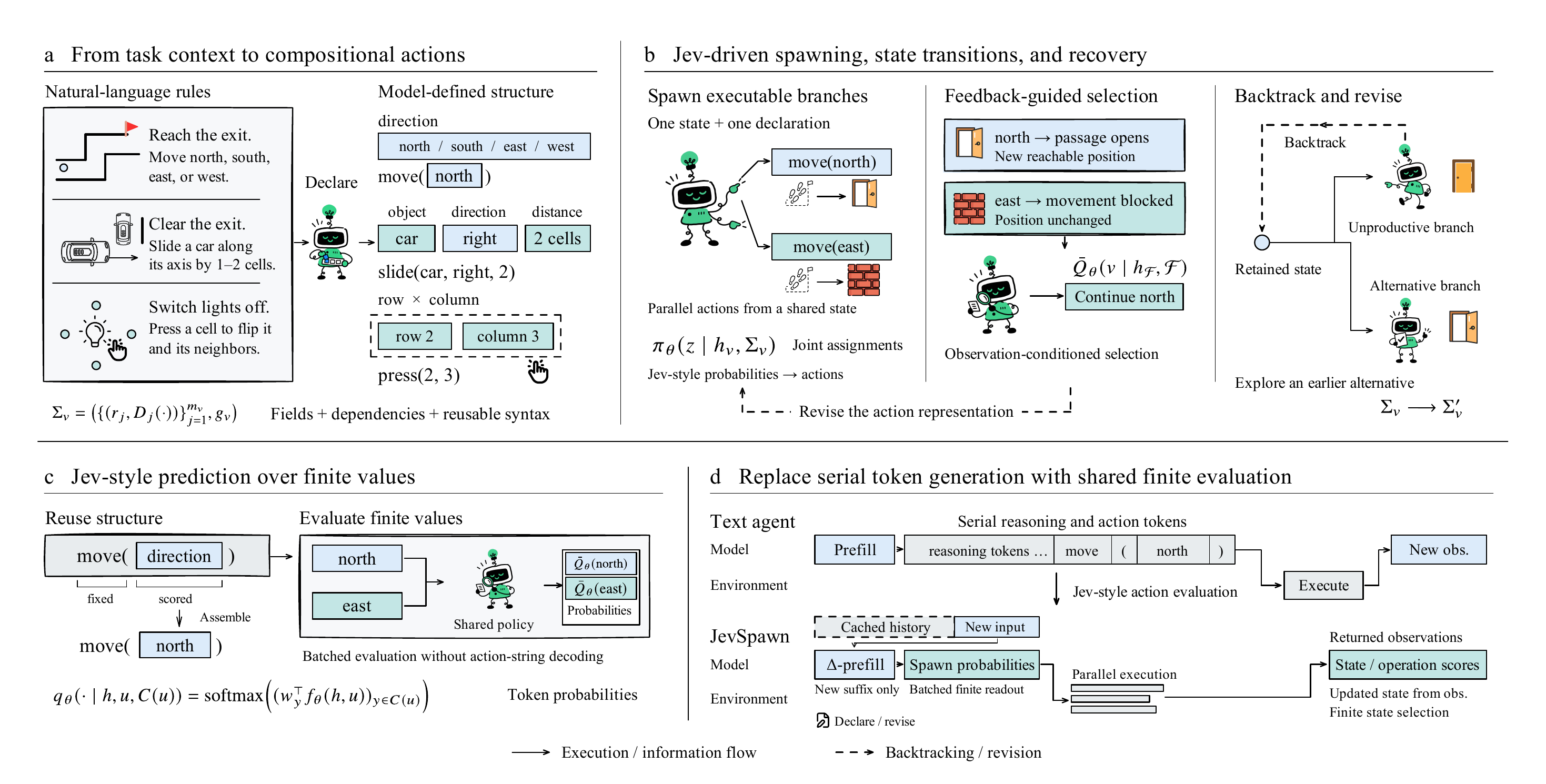}
\caption{JevSpawn architecture. (a) Compositional actions derived from natural language. (b) Probabilistic spawning, feedback-driven transitions, and recovery. (c) Finite action probabilities. (d) Shared computation in the interaction loop.}
\label{fig:method}
\end{figure*}

\subsection{Compositional Jev Policies}

The action representation is inferred as part of the policy. Let $v$ index a branch, $x$ contain the task context and interaction rules, $H_v$ denote the action, observation, and declaration feedback history of $v$, and $M$ collect observations shared across branches. The model receives $h_v=\operatorname{prompt}(x,M,H_v)$ together with the current declaration and finite alternatives. A pretrained model with fixed parameters $\theta$ constructs
\begin{equation}
 \Sigma_v=\bigl(\{(r_j,D_j(\cdot))\}_{j=1}^{m_v},g_v\bigr),
 \label{eq:declaration}
\end{equation}
Here $\Sigma_v$ is the action declaration for branch $v$, $m_v$ is the number of fields, and $j$ indexes a field. Each $r_j$ specifies a field name and meaning, while $D_j(\cdot)$ gives the finite values permitted under preceding assignments. For a complete assignment $z=(z_1,\ldots,z_{m_v})$, $g_v(z)$ renders the field values as an executable action. Shared syntax follows the task rules, while missing fields are generated from context. Appendix~\ref{app:representation} gives a worked construction.

Fields are evaluated in conditional blocks $b_1,\ldots,b_L$, where $b_\ell$ indexes the fields evaluated jointly at composition step $\ell$ and $L$ is the maximum composition depth. The tuple $z_{b_\ell}$ contains the values assigned to that block, and $z_{<b_\ell}$ contains all preceding assignments. Completed assignments receive only an empty extension of probability one. The compositional policy assigns probability $\pi_\theta(z\mid h_v,\Sigma_v)$ to a complete assignment,
\begin{equation}
 \pi_\theta(z\mid h_v,\Sigma_v)=
 \prod_{\ell=1}^{L}\bar Q_\theta\bigl(z_{b_\ell}\mid h_v^{z_{<b_\ell}},\mathcal D_{b_\ell}(z_{<b_\ell})\bigr),
 \label{eq:proposal-policy}
\end{equation}
The context $h_v^{z_{<b_\ell}}$ augments $h_v$ with the preceding assignments. The joint domain $\mathcal D_{b_\ell}(z_{<b_\ell})$ contains the block assignments allowed by the field domains $D_j$. The finite distribution $\bar Q_\theta$ gives the probability of a block assignment within that domain and is normalized over a value tree, with equal probability shares for duplicate serialized values. Equation~\ref{eq:proposal-policy} defines a policy over the declared action space without prescribing a trajectory. The distribution and absorbing empty extension are derived in Appendix~\ref{app:value-tree}.

\subsection{Adaptive Branch Exploration}

Spawning expands a state into executable alternatives under the compositional policy. A beam of width $K$ retains up to $K$ assignments. Each retained assignment $z_w$ defines a child branch $w$ and an action $a_w=g_v(z_w)$. Let $s_v$ be the parent environment state. The interaction operator $\mathcal E$ executes $a_w$ in an independent copy of $s_v$, producing the child state $s_w$ and observation $o_w$,
\begin{equation}
 (s_w,o_w)=\mathcal E(s_v,a_w),\qquad
 H_w=H_v\mathbin{\Vert}(a_w,o_w).
 \label{eq:environment-step}
\end{equation}
Here $\Vert$ appends the action--observation pair to the parent history, giving the child history $H_w$. The continuation policy scores retained branches using returned observations. Action preferences can therefore be revised after execution, and earlier branches can be resumed without regenerating the preceding trajectory.

The policy can also revise $\Sigma_v$ from accumulated feedback. Recovery explores retained states, while revision changes the available action space. Answer submission or budget exhaustion terminates interaction. Algorithm~\ref{alg:jevspawn} specifies branch retention, operation selection, and expansion.

\subsection{Parallel Policy Evaluation}

Declared values are known token sequences, whose causal hidden states can be evaluated in parallel. JevSpawn projects onto vocabulary entries that distinguish continuations and combines the resulting probabilities into $\bar Q_\theta$. The same finite evaluation supports action proposals, branch selection, and operation choice. Appendix~\ref{app:derivations} derives the distribution and computation.

Shared action structure avoids regenerating complete descriptions, while prefix reuse avoids recomputing common context. For $B$ evaluations with prefix length $n_p$ and suffix lengths $n_1,\ldots,n_B$, processed input tokens decrease from $Bn_p+\sum_{b=1}^{B}n_b$ to $n_p+\sum_{b=1}^{B}n_b$, before padding and value prefix evaluation. Suffix attention to the shared prefix is retained. Appendix~\ref{app:attention} establishes the conditions for equivalent attention and recurrent computation.

Text generation supplies missing declarations, revisions, and final answers when required. Repeated exploration uses the finite policy, with the pretrained parameters unchanged.

%% file: experiments.tex
\section{Experiments}
\label{sec:results}

\input{tables/main_scores}
\input{tables/main_time}

\begin{figure*}[!t]
\centering
\includegraphics[scale=0.615385]{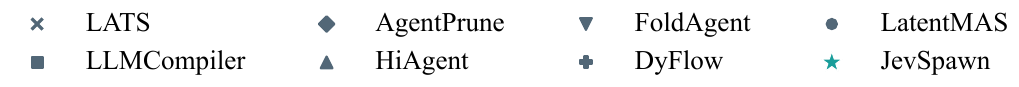}
\par\medskip
\begin{subfigure}[t]{0.265015\linewidth}
\centering
\includegraphics[width=\linewidth]{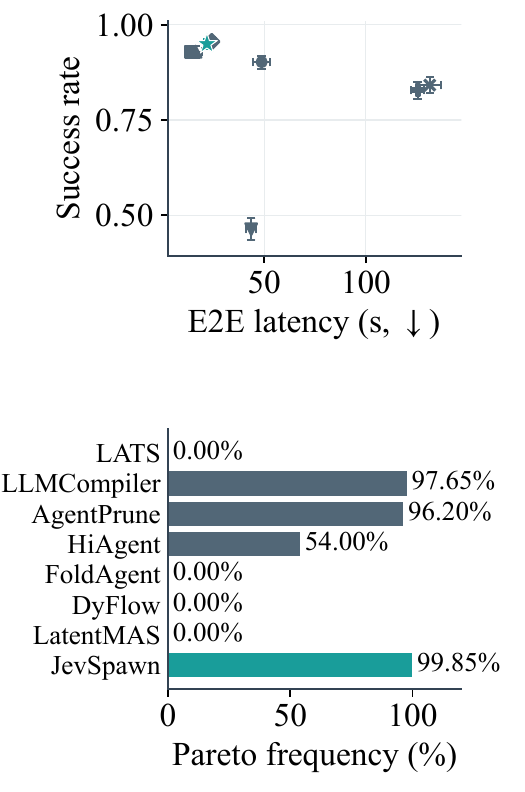}
\caption{PPNL}
\end{subfigure}
\hfill
\begin{subfigure}[t]{0.231662\linewidth}
\centering
\includegraphics[width=\linewidth]{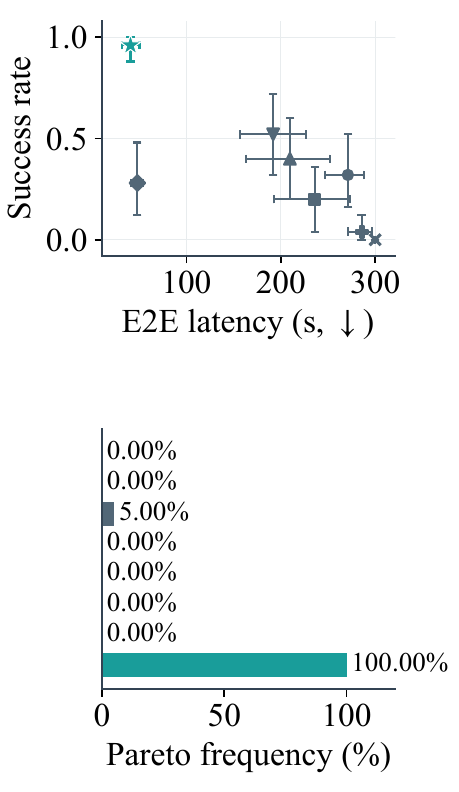}
\caption{Maze}
\end{subfigure}
\hfill
\begin{subfigure}[t]{0.231662\linewidth}
\centering
\includegraphics[width=\linewidth]{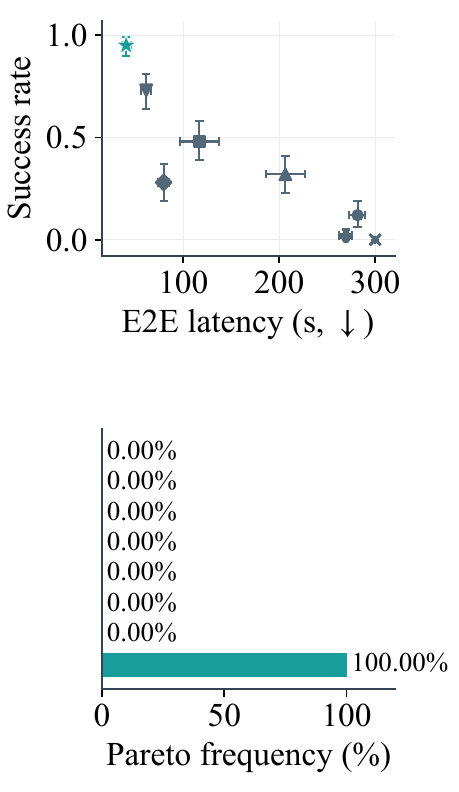}
\caption{Grid}
\end{subfigure}
\hfill
\begin{subfigure}[t]{0.231662\linewidth}
\centering
\includegraphics[width=\linewidth]{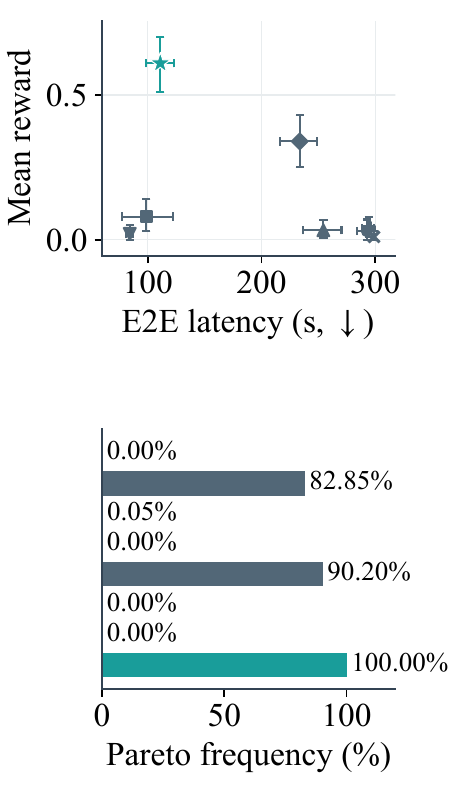}
\caption{LightsOut}
\end{subfigure}
\par\bigskip
\begin{subfigure}[t]{0.265015\linewidth}
\centering
\includegraphics[width=\linewidth]{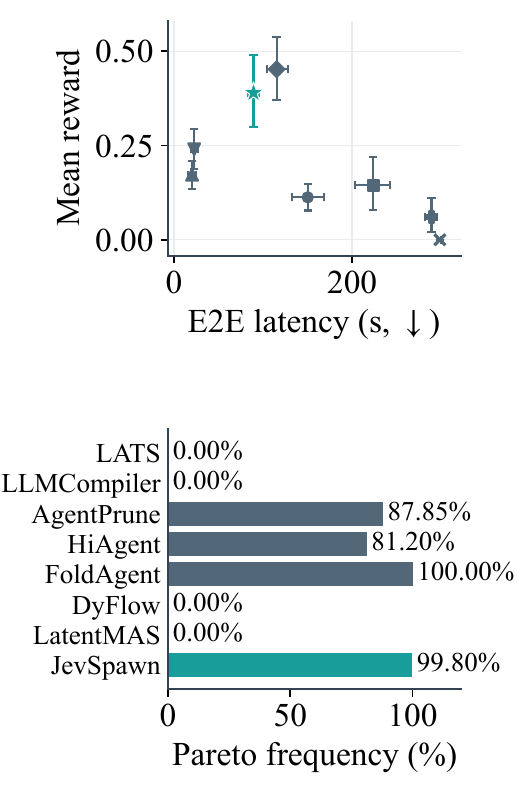}
\caption{RushHour}
\end{subfigure}
\hfill
\begin{subfigure}[t]{0.231662\linewidth}
\centering
\includegraphics[width=\linewidth]{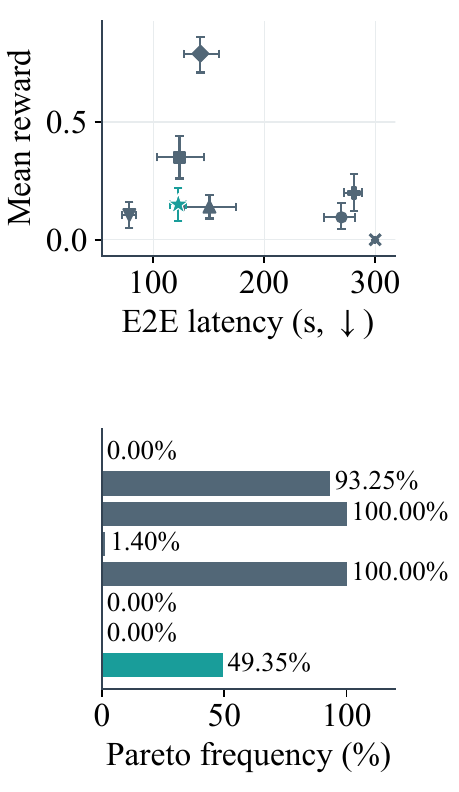}
\caption{Sokoban}
\end{subfigure}
\hfill
\begin{subfigure}[t]{0.231662\linewidth}
\centering
\includegraphics[width=\linewidth]{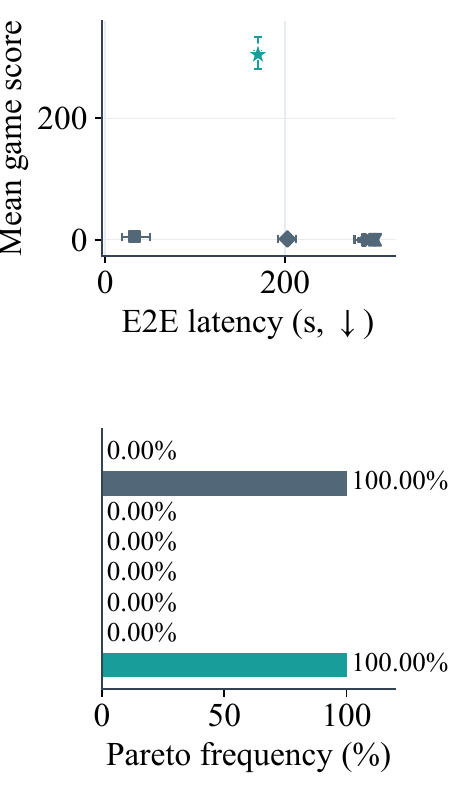}
\caption{2048}
\end{subfigure}
\hfill
\begin{subfigure}[t]{0.231662\linewidth}
\centering
\includegraphics[width=\linewidth]{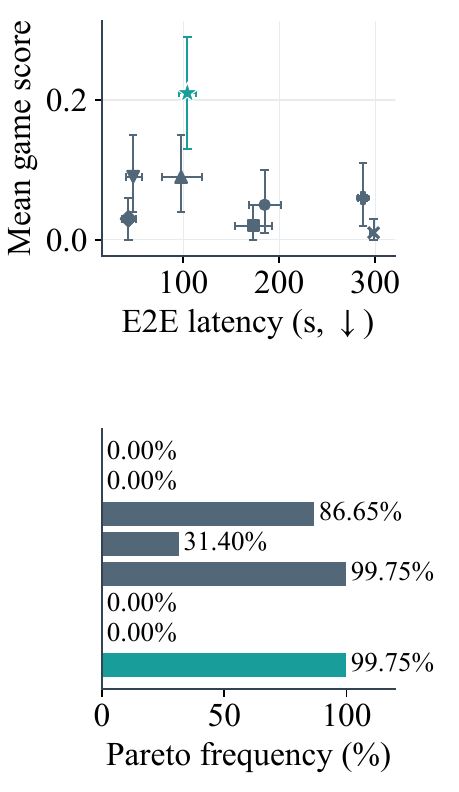}
\caption{Nullify}
\end{subfigure}
\caption{Stability of the quality--latency tradeoff under paired bootstrap resampling. Symbols mark paired bootstrap medians, and error bars show 95\% marginal intervals from 2,000 resamples. Horizontal bars show Pareto frontier frequency under resampling of paired instances.}
\label{fig:tradeoff}
\end{figure*}

\input{tables/main_throughput}

\begin{figure*}[!t]
\centering
\includegraphics[scale=0.500000]{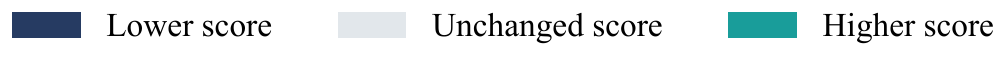}\par\smallskip
\begin{subfigure}[t]{0.288449\linewidth}
\centering
\includegraphics[width=\linewidth]{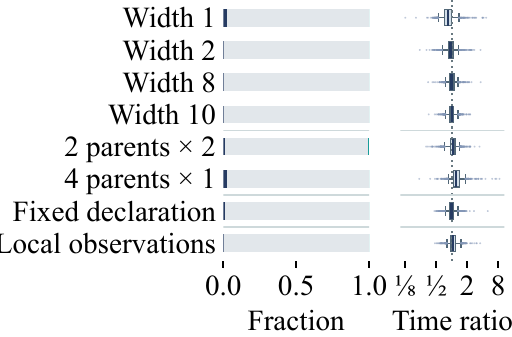}
\caption{PPNL}
\end{subfigure}
\hspace{0.025\linewidth}
\begin{subfigure}[t]{0.176410\linewidth}
\centering
\includegraphics[width=\linewidth]{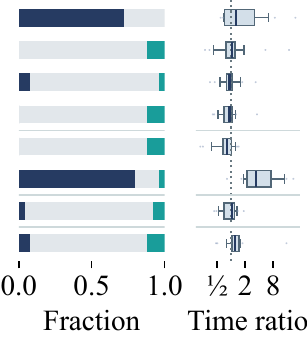}
\caption{Maze}
\end{subfigure}
\hspace{0.025\linewidth}
\begin{subfigure}[t]{0.176410\linewidth}
\centering
\includegraphics[width=\linewidth]{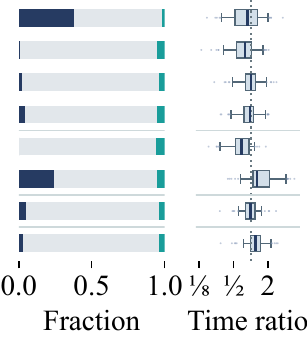}
\caption{Grid}
\end{subfigure}
\hspace{0.025\linewidth}
\begin{subfigure}[t]{0.176410\linewidth}
\centering
\includegraphics[width=\linewidth]{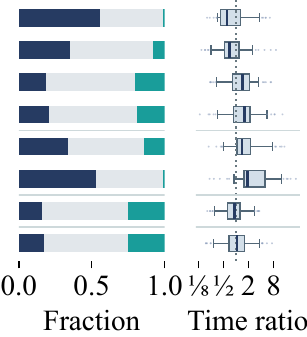}
\caption{LightsOut}
\end{subfigure}
\par\smallskip
\begin{subfigure}[t]{0.288449\linewidth}
\centering
\includegraphics[width=\linewidth]{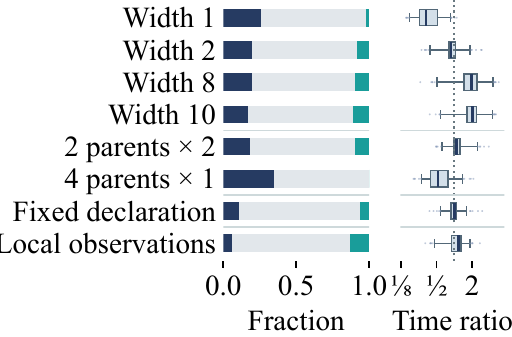}
\caption{RushHour}
\end{subfigure}
\hspace{0.025\linewidth}
\begin{subfigure}[t]{0.176410\linewidth}
\centering
\includegraphics[width=\linewidth]{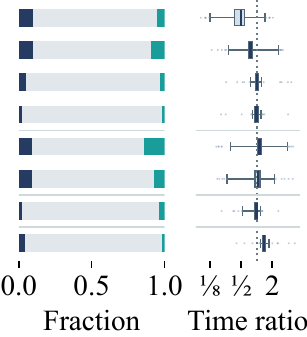}
\caption{Sokoban}
\end{subfigure}
\hspace{0.025\linewidth}
\begin{subfigure}[t]{0.176410\linewidth}
\centering
\includegraphics[width=\linewidth]{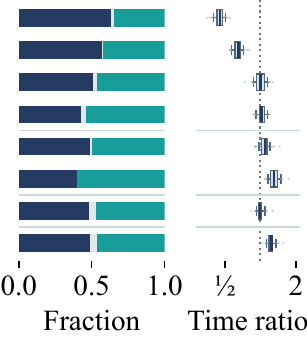}
\caption{2048}
\end{subfigure}
\hspace{0.025\linewidth}
\begin{subfigure}[t]{0.176410\linewidth}
\centering
\includegraphics[width=\linewidth]{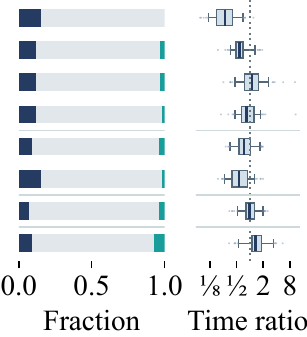}
\caption{Nullify}
\end{subfigure}
\caption{Instance-level effects of component ablations. Stacked bars show fractions with lower, equal, and higher scores than complete JevSpawn. Boxes summarize paired latency ratios with medians, interquartile ranges, and fifth to ninety-fifth percentile whiskers. Ratios below one indicate shorter execution.}
\label{fig:ablations}
\end{figure*}

\subsection{Experimental Setup}

JevSpawn and the seven agent baselines use Qwen3.8-27B as the pretrained policy and run on four H100 GPUs with tensor parallelism and bfloat16 precision. The batch size is 8, the context limit is 16,384 tokens, and each generation is limited to 2,048 tokens. Tasks receive up to 36 exploration rounds and a final submission step, within a 300 second limit. Temperature is zero, thinking mode is disabled, and the seed is 42. JevSpawn selects one parent and spawns up to four distinct actions per expansion. A further variant replaces finite scoring with TypeSafe Jev and retains Qwen for declarations and text generation. End-to-end (E2E) latency measures elapsed time from task submission to termination. Appendix~\ref{sec:setup} provides timing definitions and execution details.

Eight tasks from five benchmarks are used to compare JevSpawn with LATS \citep{zhou2024lats}, LLMCompiler \citep{kim2024llmcompiler}, AgentPrune \citep{zhang2025agentprune}, HiAgent \citep{hu-etal-2025-hiagent}, FoldAgent \citep{sun2026scaling}, DyFlow \citep{wang2025dyflow}, and LatentMAS \citep{zou2026latent}. PPNL evaluates path planning from natural language specifications \citep{ppnl}. Maze from LMRL Gym and Grid from LLF Bench evaluate navigation with action feedback \citep{lmrlgym,llfbench}. LightsOut, RushHour, and Sokoban from TextArena assess sequential puzzle solving \citep{textarena}, while 2048 and Nullify from KORGym assess numerical game play \citep{korgym}. Evaluation uses the public PPNL test collections, fixed Maze initial states, and environment seeds 42 through 141 for the remaining tasks.

\paragraph{Ablation settings.}

Component ablations vary expansion width over $\{1,2,4,8,10\}$, distribute four actions across parent--child allocations of $1\times4$, $2\times2$, or $4\times1$, retain the first valid declaration, or restrict observations to the current branch. Each variant is compared with a separate run of the full method on the same 1,761 instances.

Long horizon experiments form a separate series with a 131,072 token context and round limits of $\{36,54,72,108\}$, each evaluated with and without the 300 second deadline. Batch scaling retains the 16,384 token context and varies physical batch size over $\{8,16,32,64,128\}$ until memory is exhausted. Other settings are fixed within each comparison. Appendix~\ref{app:ablations} reports the complete results.

\Needspace{6\baselineskip}
\subsection{Main Results}

\paragraph{Task performance.}

Against the seven agent baselines, JevSpawn achieves the highest scores on five of eight tasks in Table~\ref{tab:main-scores}. Success rates of 0.96 and 0.95 are achieved on Maze and Grid, compared with 0.52 and 0.73 for FoldAgent. LightsOut reward reaches 0.61 against AgentPrune's 0.34. Nullify reward reaches 0.21 against 0.09 for HiAgent and FoldAgent, and the 2048 score reaches 305.12 against LLMCompiler's 5.00. AgentPrune remains strongest on PPNL, RushHour, and Sokoban.

Within the same JevSpawn architecture, TypeSafe Jev outperforms Qwen scoring on PPNL, LightsOut (0.66 versus 0.61), and Sokoban, matches Maze, and scores lower on Grid, RushHour, 2048, and Nullify.

\paragraph{Execution efficiency.}

JevSpawn achieves joint quality and latency gains on Maze and Grid. Alongside the success rate improvements, E2E latency is reduced from the fastest baseline values of 47.91 and 61.44 seconds to 40.91 and 40.58 seconds. Figure~\ref{fig:tradeoff} places JevSpawn on the Pareto frontier in nearly all bootstrap resamples for seven tasks. Sokoban frontier membership is less stable, reflecting a 1.50 second latency margin over LLMCompiler and a lower score.

On LightsOut and 2048, higher scores accompany longer execution. JevSpawn takes 111.05 seconds on LightsOut against 84.17 for FoldAgent, and 169.94 seconds on 2048 against 33.51 for LLMCompiler. The TypeSafe Jev variant takes 1.4 to 2.1 times as long as Qwen scoring across all eight tasks, including API communication. Appendix~\ref{app:performance} compares per-instance scores and latencies with AgentPrune.

Text decoding throughput reaches 544 and 580 tokens per second on PPNL and LightsOut, respectively, the highest values in Table~\ref{tab:throughput}. Maze and Grid use finite evaluation throughout. On Nullify, text generation and finite evaluation overlap in every decoding interval. Separate finite evaluation measurements are given in Appendix~\ref{sec:setup}.

\Needspace{8\baselineskip}
\subsection{Ablation Studies}

Figure~\ref{fig:ablations} compares expansion width, parent allocation, declaration revision, and branch information sharing against the full JevSpawn method. Appendix~\ref{app:ablations} reports scores and E2E latency for every condition.

Parallel spawning contributes strongly to navigation and LightsOut performance. Reducing width from four to one lowers Maze success from 0.88 to 0.16, Grid success from 0.94 to 0.58, and LightsOut reward from 0.61 to 0.06. Increasing width to ten solves every Maze instance, but halves Nullify reward and approximately doubles RushHour latency.

Parent allocation matters at a fixed budget of four actions. Success rates of 1.00 are achieved on both Maze and Grid with two parents and two actions each, compared with 0.88 and 0.94 for one parent. Four parents with one action each reduce these rates to 0.12 and 0.75, while increasing the 2048 score.

Retaining the first valid declaration improves scores on Maze, LightsOut, Sokoban, and 2048 and reduces latency on six tasks. Adaptive declarations score higher on the other four tasks. Restricting observations to the current branch improves Maze, Grid, LightsOut, and RushHour scores, whereas shared observations yield higher scores on the remaining tasks and lower latency on all eight.

Without a deadline, increasing the round limit from 36 to 108 raises the 2048 score from 315.16 to 1102.24 and LightsOut reward from 0.61 to 0.876. Under the 300 second deadline, 85 percent of 2048 tasks time out at 108 rounds. Figure~\ref{fig:turns} and Table~\ref{tab:turns} compare both time settings.

%% file: tables/main_scores.tex
\begin{table*}[t]
\centering
\caption{Task performance ($\uparrow$) across eight benchmark tasks. PPNL, Maze, and Grid use success rates on $[0,1]$, and the remaining tasks use benchmark rewards or scores. Bold and underlined values indicate the best and second best results, including ties. The TypeSafe Jev variant replaces finite scoring within JevSpawn.}
\label{tab:main-scores}
\small\begin{adjustbox}{max width=\linewidth}
\begin{tabular}{lrrrrrrrr}
\toprule
Method &  PPNL &  Maze &  Grid &  LightsOut &  RushHour &  Sokoban &  2048 &  Nullify \\
\midrule
LATS & 0.842 & 0.000 & 0.000 & 0.010 & 0.000 & 0.000 & 0.000 & 0.010 \\
LLMCompiler & 0.929 & 0.200 & 0.480 & 0.080 & 0.147 & \underline{0.350} & 5.000 & 0.020 \\
AgentPrune & \textbf{0.956} & 0.280 & 0.280 & 0.340 & \textbf{0.453} & \textbf{0.790} & 1.240 & 0.030 \\
HiAgent & 0.930 & 0.400 & 0.320 & 0.033 & 0.172 & 0.140 & 0.000 & 0.090 \\
FoldAgent & 0.464 & \underline{0.520} & 0.730 & 0.020 & 0.240 & 0.105 & 0.000 & 0.090 \\
DyFlow & 0.828 & 0.040 & 0.020 & 0.040 & 0.060 & 0.200 & 0.120 & 0.060 \\
LatentMAS & 0.902 & 0.320 & 0.120 & 0.030 & 0.113 & 0.095 & 0.000 & 0.050 \\
TypeSafe Jev & \underline{0.955} & \textbf{0.960} & \underline{0.930} & \textbf{0.660} & 0.290 & 0.250 & \underline{296.000} & \underline{0.170} \\
JevSpawn & 0.950 & \textbf{0.960} & \textbf{0.950} & \underline{0.610} & \underline{0.390} & 0.150 & \textbf{305.120} & \textbf{0.210} \\
\bottomrule
\end{tabular}\end{adjustbox}
\end{table*}

%% file: tables/main_time.tex
\begin{table*}[t]
\centering
\caption{E2E latency (s, $\downarrow$), including failed tasks and timeouts. Bold and underlined values indicate the lowest and second lowest latencies, including ties. The TypeSafe Jev variant replaces finite scoring within JevSpawn.}
\label{tab:main-time}
\small\begin{adjustbox}{max width=\linewidth}
\begin{tabular}{lrrrrrrrr}
\toprule
Method &  PPNL &  Maze &  Grid &  LightsOut &  RushHour &  Sokoban &  2048 &  Nullify \\
\midrule
LATS & 131.34 & 300.00 & 300.00 & 299.01 & 298.57 & 300.00 & 300.00 & 298.18 \\
LLMCompiler & \textbf{13.61} & 235.32 & 117.16 & \underline{99.26} & 223.42 & 123.99 & \textbf{33.51} & 172.85 \\
AgentPrune & 24.23 & \underline{47.91} & 79.57 & 233.19 & 115.48 & 142.53 & 202.57 & \textbf{42.90} \\
HiAgent & \underline{16.25} & 209.06 & 206.05 & 254.54 & \textbf{20.32} & 150.76 & 300.00 & 98.18 \\
FoldAgent & 43.55 & 191.69 & \underline{61.44} & \textbf{84.17} & \underline{22.74} & \textbf{78.32} & 300.07 & \underline{48.10} \\
DyFlow & 125.47 & 285.43 & 269.57 & 294.18 & 289.03 & 280.57 & 287.35 & 287.37 \\
LatentMAS & 48.70 & 269.86 & 281.61 & 292.64 & 150.25 & 269.05 & 289.03 & 185.05 \\
TypeSafe Jev & 43.96 & 63.98 & 78.39 & 153.95 & 187.45 & 214.76 & 249.05 & 189.41 \\
JevSpawn & 21.88 & \textbf{40.91} & \textbf{40.58} & 111.05 & 89.20 & \underline{122.49} & \underline{169.94} & 104.21 \\
\bottomrule
\end{tabular}\end{adjustbox}
\end{table*}

%% file: tables/main_throughput.tex
\begin{table*}[t]
\centering
\caption{Text decoding throughput (tok/s, $\uparrow$), excluding prefill. TypeSafe Jev uses Qwen3.8-27B for text generation. / denotes not applicable. Bold and underlined values indicate the best and second best results, including ties.}
\label{tab:throughput}
\small\begin{adjustbox}{max width=\linewidth}
\begin{tabular}{lrrrrrrrr}
\toprule
Method &  PPNL &  Maze &  Grid &  LightsOut &  RushHour &  Sokoban &  2048 &  Nullify \\
\midrule
LATS & 220.54 & 212.78 & 168.61 & 226.54 & 252.04 & 271.42 & 281.84 & \underline{242.44} \\
LLMCompiler & 220.21 & 180.78 & 249.03 & 157.71 & 162.23 & 159.28 & 285.74 & 208.72 \\
AgentPrune & 236.92 & \textbf{368.71} & 160.73 & 135.88 & 157.97 & 139.71 & 139.74 & 150.21 \\
HiAgent & 278.55 & 347.87 & \textbf{471.89} & \underline{461.86} & 266.87 & \underline{369.30} & \textbf{454.35} & \textbf{399.94} \\
FoldAgent & 211.62 & 111.22 & 262.12 & 409.24 & \textbf{530.84} & \textbf{527.77} & 281.85 & 125.03 \\
DyFlow & 245.22 & 172.71 & 179.26 & 225.29 & 236.80 & 226.94 & 244.05 & 207.33 \\
LatentMAS & \underline{372.13} & \underline{351.46} & \underline{274.55} & 390.79 & 263.62 & 231.52 & 237.11 & 106.43 \\
TypeSafe Jev & 92.69 & / & / & 110.68 & \underline{280.05} & 131.76 & 153.11 & 128.11 \\
JevSpawn & \textbf{544.35} & / & / & \textbf{580.05} & 80.34 & 225.70 & \underline{420.31} & / \\
\bottomrule
\end{tabular}\end{adjustbox}
\end{table*}

%% file: derivations.tex
\section{Analysis of Structured Execution}
\label{app:derivations}

The following analysis derives the finite action distribution, bounded exploration procedure, and conditions for equivalent shared computation.

\begin{algorithm}[!ht]
\caption{JevSpawn}
\label{alg:jevspawn}
\begin{algorithmic}[1]
\REQUIRE Task $x$, model $\theta$, environment $\mathcal E$, budgets $T,R,P,K$
\STATE Initialize $(\mathcal F,M)\gets(\{v_0\},\varnothing)$
\STATE $(s_{v_0},H_{v_0},\Sigma_{v_0})\gets(s_0,\varnothing,\varnothing)$
\STATE Set $\mathcal U_T(v)=\{\mathrm{submit}\}$ for every branch $v$
\FOR{$t=0,\ldots,T$}
  \STATE Render $h_{\mathcal F}\gets\operatorname{prompt}(x,M,\{H_b\}_{b\in\mathcal F})$
  \STATE Score $\rho_t\gets\bar Q_\theta(\cdot\mid h_{\mathcal F},\mathcal F)$
  \STATE Retain $\mathcal F\gets\operatorname{Top}_{R}(\mathcal F,\rho_t)$ and select $v\gets\arg\max_{b\in\mathcal F}\rho_t(b)$
  \STATE Choose $u\gets\arg\max_{u'\in\mathcal U_t(v)}\bar Q_\theta(u'\mid h_v,\mathcal U_t(v))$
  \IF{$u=\mathrm{submit}$}
    \STATE \textbf{return} $\operatorname{Submit}_\theta(h_v,s_v)$
  \ELSIF{$u=\mathrm{revise}$}
    \STATE $(\Sigma_v,e_v)\gets\operatorname{Revise}_\theta(h_v,\Sigma_v)$
    \STATE $H_v\gets H_v\mathbin{\Vert}e_v$
  \ENDIF
  \STATE Remove $\mathcal F\gets\mathcal F\setminus\{v\mid u=\mathrm{discard}\}$
  \IF{$\mathcal F=\varnothing$}
    \STATE \textbf{return} Failure
  \ENDIF
  \STATE Select parents $\mathcal B\gets\operatorname{Top}_{P}(\{b\in\mathcal F\mid\chi(b)\chi(v)=1,\ u\in\{\mathrm{expand},\mathrm{revise}\}\},\rho_t)$
  \STATE Compose $\mathcal A_b\gets\{g_b(z)\mid z\in\operatorname{Beam}_{K}(\pi_\theta(\cdot\mid h_b,\Sigma_b))\},\ b\in\mathcal B$
  \STATE Spawn $\mathcal C\gets\{w=(b,a)\mid b\in\mathcal B,\ a\in\mathcal A_b\}$
  \FORALL{$w=(b,a)\in\mathcal C$ in parallel}
    \STATE Execute $(s_w,o_w)\gets\mathcal E(s_b,a)$
    \STATE $H_w\gets H_b\mathbin{\Vert}(a,o_w),\quad\Sigma_w\gets\Sigma_b$
  \ENDFOR
  \STATE Share $M\gets M\mathbin{\Vert}\{(w,a,o_w)\}_{w=(b,a)\in\mathcal C}$
  \STATE Update $\mathcal F\gets(\mathcal F\setminus\mathcal B)\cup\mathcal C$
\ENDFOR
\end{algorithmic}
\end{algorithm}

In Algorithm~\ref{alg:jevspawn}, $s_0$ is the initial environment state and $v_0$ is the root branch. The budgets $T$, $R$, $P$, and $K$ bound interaction rounds, retained branches, expanded parents, and actions per parent, respectively. The context $h_v=\operatorname{prompt}(x,M,H_v)$ combines task context, branch history, and shared observations. The operation set $\mathcal U_t(v)$ requires revision when a declaration is missing or rejected, permits submission or removal at a terminal branch, and otherwise permits expansion, revision, or submission. Round $T$ permits submission alone. The indicator $\chi(v)$ identifies nonterminal branches with executable declarations. Rejected revisions preserve any previous executable declaration and return feedback. Rejected submission, an empty frontier, or resource exhaustion yields failure.

Each round selects one operation for the highest ranked branch. Revision updates that branch alone. Expansion uses up to $P$ ranked executable parents when the selected branch remains executable, with existing declarations retained for additional parents. Removal produces no expansion. Unselected branches remain available until pruning.

\subsection{Constructing a compositional action space}
\label{app:representation}

A declaration defines finite choices and a shared action renderer $g_v$. With conditional field domains, the assignment and action spaces are
\begin{equation}
 \mathcal Z(\Sigma_v)=\left\{z\in\prod_{j=1}^{m_v}\mathcal X_j\,\middle|\,z_j\in D_j(z_{<j})\right\},
 \qquad \mathcal A(\Sigma_v)=\{g_v(z)\mid z\in\mathcal Z(\Sigma_v)\},
\end{equation}
where $\mathcal X_j$ is the value universe of field $j$. Independent domains yield a Cartesian product. Conditional domains exclude assignments that violate declared dependencies.

\begin{examplebox}[frametitle={Compositional movement actions},nobreak=true]
The rules permit one step north or up to two steps east. Actions use \texttt{direction} and \texttt{steps} fields,
\begin{align}
 r_1&=\text{movement direction},&D_1&=\{\mathrm{north},\mathrm{east}\},\\
 r_2&=\text{number of steps},&D_2(z_1)&=\{1,\ldots,1+\mathbb I[z_1=\mathrm{east}]\},\\
 g_v(z_1,z_2)&=\texttt{\{\char34 direction\char34: }z_1\texttt{, \char34 steps\char34: }z_2\texttt{\}}.
\end{align}
The renderer inserts quoted directions and integer step counts into shared syntax. Joint evaluation uses the block
\begin{equation}
 \mathcal D_{b_1}(\varnothing)=\{(\mathrm{north},1),(\mathrm{east},1),(\mathrm{east},2)\}.
\end{equation}
The induced policy satisfies
\begin{equation}
 \pi_\theta(z\mid h_v,\Sigma_v)
 =\bar Q_\theta(z\mid h_v,\mathcal D_{b_1}(\varnothing)),
 \qquad \sum_{z\in\mathcal D_{b_1}(\varnothing)}\pi_\theta(z\mid h_v,\Sigma_v)=1.
\end{equation}
Direction and step count are scored jointly. The two highest ranked assignments can be executed in separate environment copies, with returned observations determining the next branch.

\end{examplebox}

Optional arguments and variable length objects introduce presence or length fields before dependent values. Open strings require model generated structure when no finite domain is supplied. Syntax and type checks return declaration feedback for revision. Serialization fixes the output form, while the domains remain model predictions.

\subsection{Probability on a finite value tree}
\label{app:value-tree}

Let $\mathcal O$ be a nonempty finite option list. Entry $i$ is serialized as $\tau(i)$ with a reserved termination token, yielding the prefix free set $\mathcal T=\{\tau(i)\mid i\in\mathcal O\}$ of distinct sequences. The context $h$ contains task evidence, scoring instructions, and option descriptions. Action options specify complete arguments, while integer branch identifiers refer to states and observations.

Let $\mathcal N$ contain every proper prefix of a sequence in $\mathcal T$, including $\epsilon$, and let $C(u)$ be the outgoing tokens at $u\in\mathcal N$. The pretrained hidden state $f_\theta(h,u)$ and output row $w_y$ define
\begin{equation}
 q_\theta(y\mid h,u,C(u))=
 \frac{\exp\bigl(w_y^\top f_\theta(h,u)\bigr)}
 {\sum_{y'\in C(u)}\exp\bigl(w_{y'}^\top f_\theta(h,u)\bigr)},
 \qquad y\in C(u).
 \label{eq:token-choice}
\end{equation}
A singleton continuation has probability one. For a leaf $\tau\in\mathcal T$,
\begin{equation}
 Q_\theta(\tau\mid h,\mathcal T)=
 \prod_{(u,y)\in\operatorname{path}(\tau)\,\mid\,|C(u)|>1}
 q_\theta(y\mid h,u,C(u)).
 \label{eq:value-score}
\end{equation}

\begin{proposition}
\label{prop:normalization}
For finite, nonempty, prefix free $\mathcal T$,
\begin{equation}
 Q_\theta(\tau\mid h,\mathcal T)\geq0,\qquad
 \sum_{\tau\in\mathcal T}Q_\theta(\tau\mid h,\mathcal T)=1.
\end{equation}
\end{proposition}
\begin{proof}
\begin{align*}
 \mu(\epsilon)&=1,\qquad
 \mu(uy)=\mu(u)q_\theta(y\mid h,u,C(u))\geq0,\\
 \sum_{y\in C(u)}\mu(uy)&=\mu(u),\qquad u\in\mathcal N,\\
 \sum_{\tau\in\mathcal T}Q_\theta(\tau\mid h,\mathcal T)
 &=\sum_{\tau\in\mathcal T}\mu(\tau)=\mu(\epsilon)=1.\qedhere
\end{align*}
\end{proof}

Define $c(\tau)=|\{i\in\mathcal O\mid\tau(i)=\tau\}|$. The option distribution follows from Proposition~\ref{prop:normalization},
\begin{equation}
 \bar Q_\theta(i\mid h,\mathcal O)=\frac{Q_\theta(\tau(i)\mid h,\mathcal T)}{c(\tau(i))},
 \qquad
 \sum_{i\in\mathcal O}\bar Q_\theta(i\mid h,\mathcal O)
 =\sum_{\tau\in\mathcal T}c(\tau)\frac{Q_\theta(\tau\mid h,\mathcal T)}{c(\tau)}=1.
 \label{eq:option-probability}
\end{equation}
The same distribution scores block assignments, branches, and operations. Figure~\ref{fig:method} abbreviates the conditional distribution as $\bar Q_\theta(i)$. Duplicate serialized values share mass equally, and duplicate executable actions are collapsed within each parent state.

The locally restricted distribution differs in general from conditioning the unrestricted model on $\mathcal T$,
\begin{equation}
 p_\theta(\tau\mid h,\tau\in\mathcal T)=
 \frac{p_\theta(\tau\mid h)}{\sum_{\tau'\in\mathcal T}p_\theta(\tau'\mid h)}.
 \label{eq:sequence-conditioning}
\end{equation}
Equation~\ref{eq:sequence-conditioning} includes forced token likelihoods and the probability of remaining in the tree. Local normalization omits these factors and defines a relative action preference, not a probability of execution success.

\subsection{Parallel evaluation and value ordering}

Let $\mathcal N_b=\{u\in\mathcal N\mid |C(u)|>1\}$. Scoring requires hidden states only at these branching prefixes. Known prefix tokens permit parallel evaluation with the same ancestor visibility, positions, and recurrent initial states as independent evaluation.

Forced tokens affect subsequent hidden states but require no output score. At prefix $u$, projection uses $d_{\mathrm{model}}|C(u)|$ multiplications instead of $d_{\mathrm{model}}|\mathcal V|$, where $d_{\mathrm{model}}$ is model width and $\mathcal V$ is the vocabulary. Prefix computation is retained.

Action options are sorted by serialized field assignments. Field grouping, branch indexing, and tie resolution remain representation dependent.

\subsection{Conditional composition and bounded exploration}

Each unfinished block has a nonempty conditional domain $\mathcal D_{b_\ell}(z_{<b_\ell})$. Completed assignments have the absorbing domain $\{\epsilon\}$, with $z\mathbin{\Vert}\epsilon=z$ and $\bar Q_\theta(\epsilon\mid h_v^z,\{\epsilon\})=1$. Bound assignments enter $h_v^{z_{<b_\ell}}$. The complete log probability is
\begin{equation}
 S_\theta(z\mid h_v,\Sigma_v)=\log\pi_\theta(z\mid h_v,\Sigma_v)
 =\sum_{\ell=1}^{L}\log\bar Q_\theta\bigl(z_{b_\ell}\mid h_v^{z_{<b_\ell}},\mathcal D_{b_\ell}(z_{<b_\ell})\bigr).
 \label{eq:action-score}
\end{equation}
The partial score is defined recursively,
\begin{align}
 S_0(\varnothing)&=0,\\
 S_\ell(z_{\leq b_\ell})&=S_{\ell-1}(z_{<b_\ell})+
 \log\bar Q_\theta\bigl(z_{b_\ell}\mid h_v^{z_{<b_\ell}},\mathcal D_{b_\ell}(z_{<b_\ell})\bigr).
\end{align}
Let $\mathcal J_\ell$ denote the beam after block $\ell$. With $\mathcal J_0=\{\varnothing\}$,
\begin{align}
 \operatorname{Extend}(\mathcal J,b)&=\{z\mathbin{\Vert}z_b\mid z\in\mathcal J,\ z_b\in\mathcal D_b(z)\},\\
 \mathcal J_\ell&=\operatorname{Top}_{K}\bigl(\operatorname{Extend}(\mathcal J_{\ell-1},b_\ell),S_\ell\bigr).
 \label{eq:beam-update}
\end{align}
Completed assignments retain the accumulated scores without another model evaluation. The final beam defines $\operatorname{Beam}_K$ in Algorithm~\ref{alg:jevspawn}.

Before pruning, normalization follows by summing over the conditional domains from the final block backward,
\begin{align}
 \sum_{z\in\mathcal Z(\Sigma_v)}\pi_\theta(z\mid h_v,\Sigma_v)
 &=\sum_{z_{b_1}\in\mathcal D_{b_1}(\varnothing)}\cdots
 \sum_{z_{b_L}\in\mathcal D_{b_L}(z_{<b_L})}
 \prod_{\ell=1}^{L}\bar Q_\theta\bigl(z_{b_\ell}\mid h_v^{z_{<b_\ell}},\mathcal D_{b_\ell}(z_{<b_\ell})\bigr)\nonumber\\
 &=1.
\end{align}
Pruning removes probability mass and may discard a partial assignment with a higher scoring completion. Beam selection therefore approximates exhaustive ranking.

If $P_t\leq P$ parents are expanded in round $t<T$ and $\delta_t$ branches are discarded after retention, then
\begin{equation}
 |\mathcal C_t|\leq P_tK,\qquad
 |\mathcal F_{t+1}|\leq R-\delta_t-P_t+P_tK.
\end{equation}
Discarded branches and expanded parents are disjoint. Further pruning or deduplication can only reduce this upper bound. Since round $T$ permits submission alone,
\begin{equation}
 N_T=\sum_{t=0}^{T-1}|\mathcal C_t|\leq\sum_{t=0}^{T-1}P_tK\leq TPK.
\end{equation}
The bound counts executed child actions. Declaration, branch selection, and submission are excluded. Domain saturation, deduplication, and early termination reduce the realized count.

\subsection{Equivalence of shared attention}
\label{app:attention}

For each query row $\mathbf q$, partition the attended tokens into prefix indices $I_p$ and the visible suffix indices $I_s$. Let $d_{\mathrm{head}}$ be the attention head dimension. For $\alpha\in\{p,s\}$,
\begin{equation}
 Z_\alpha=\sum_{k\in I_\alpha}\exp\!\left(\frac{\mathbf q^\top\mathbf k_k}{\sqrt{d_{\mathrm{head}}}}\right),\quad
 \lambda_\alpha=\log Z_\alpha,\quad
 \mathbf A_\alpha=Z_\alpha^{-1}\sum_{k\in I_\alpha}\exp\!\left(\frac{\mathbf q^\top\mathbf k_k}{\sqrt{d_{\mathrm{head}}}}\right)\mathbf v_k.
\end{equation}
For nonempty partitions, shared attention is recovered by
\begin{equation}
 \mathbf A=\frac{e^{\lambda_p}\mathbf A_p+e^{\lambda_s}\mathbf A_s}{e^{\lambda_p}+e^{\lambda_s}}.
 \label{eq:attention-merge}
\end{equation}
Indeed, for each query row,
\begin{equation}
 \frac{Z_p\mathbf A_p+Z_s\mathbf A_s}{Z_p+Z_s}
 =\frac{\sum_{k\in I_p\cup I_s}\exp\!\left(\mathbf q^\top\mathbf k_k/\sqrt{d_{\mathrm{head}}}\right)\mathbf v_k}
 {\sum_{k\in I_p\cup I_s}\exp\!\left(\mathbf q^\top\mathbf k_k/\sqrt{d_{\mathrm{head}}}\right)}.
\end{equation}
An empty partition contributes zero mass, leaving attention over the nonempty partition unchanged.

A recurrent layer initialized from the shared prefix state follows the recurrence of independent sequence evaluation. Induction over layers establishes equivalence for matching positions, causal visibility, parameters, and initial states in exact arithmetic.

\subsection{Sequential latency and shared work}

For a completed trajectory, let $\Delta_{\mathrm{decl}}$ cover declaration and revision, $\Delta_{\mathrm{submit}}$ submission, and $\Delta_{\mathrm{other}}$ scheduling. Per round, $\Delta_{\mathrm{rank},t}$, $\Delta_{\mathrm{control},t}$, $\Delta_{\mathrm{spawn},t}$, and $\Delta_{\mathrm{env},t}$ denote branch selection, operation choice, action scoring, and environment execution. Partitioning elapsed time gives
\begin{align}
 \Delta_{\mathrm{task}}
 &=\Delta_{\mathrm{decl}}+\Delta_{\mathrm{submit}}+\Delta_{\mathrm{other}}\nonumber\\
 &\quad+\sum_t\bigl(\Delta_{\mathrm{rank},t}+\Delta_{\mathrm{control},t}
 +\Delta_{\mathrm{spawn},t}+\Delta_{\mathrm{env},t}\bigr).
\end{align}
Concurrent trajectories may share GPU computation.

A common prefix of $n_p$ tokens saves $(B-1)n_p$ input tokens across $B$ requests. Reusing an unchanged history of $n_{\mathrm{hist},t}$ tokens saves the corresponding prefill at the next compatible request. Suffix attention, projection, padding, and scheduling costs remain.

%% file: setup.tex
\label{sec:setup}

\subsection{Tasks and Environment Interfaces}

Task context contains the goal, initial observation, interaction rules, and answer format, shared across methods. Method prompts are supplied separately, with core JevSpawn excerpts in Appendix~\ref{app:prompts}. Benchmark interfaces execute and score actions, and copied environment states support branch exploration. JevSpawn infers missing fields and reuses finite domains provided by task rules. Maze and Grid specify complete domains, so execution requires no text generation.

PPNL contains 1,136 instances, with 250 each from the 5 by 5 and 7 by 7 test collections, 336 from the standard collection, and 300 from the additional obstacle collection. Maze contains 25 initial states. Each remaining task contains 100 instances generated from seeds 42 through 141. Invalid submissions and timeouts count as failures. TextArena rewards and KORGym scores retain the benchmark scales.

\subsection{Execution and Timing}

Up to eight ready requests are processed together, with occupancy determined by method dependencies. E2E latency includes action construction, queueing, model computation, environment interaction, and submission. Model loading is excluded. Observed elapsed times are retained for timed out tasks.

Time to first token (TTFT) spans batch preparation and prefill through the first generated token. Inter-token latency (ITL) measures successive token intervals within a sequence, including interleaved finite evaluations. Finite evaluation latency covers preparation, context processing, and value scoring for a complete batch. Queueing and environment execution are timed separately.

Text decoding throughput divides tokens emitted after the first token by the corresponding batch decoding duration measured with CUDA events. Intervals with interleaved finite evaluation are excluded. Isolated decoding uses the same token and time definitions on replayed requests. Text generation throughput in the batch scaling experiment includes prefill and decoding. Finite throughput counts requests per second, with one request evaluating an option set containing multiple assignments.

\begin{figure}[p]
\centering
\includegraphics[scale=0.615385]{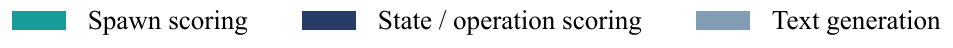}
\par\medskip
\begin{subfigure}[t]{1.0\linewidth}
\centering
\includegraphics[width=\linewidth]{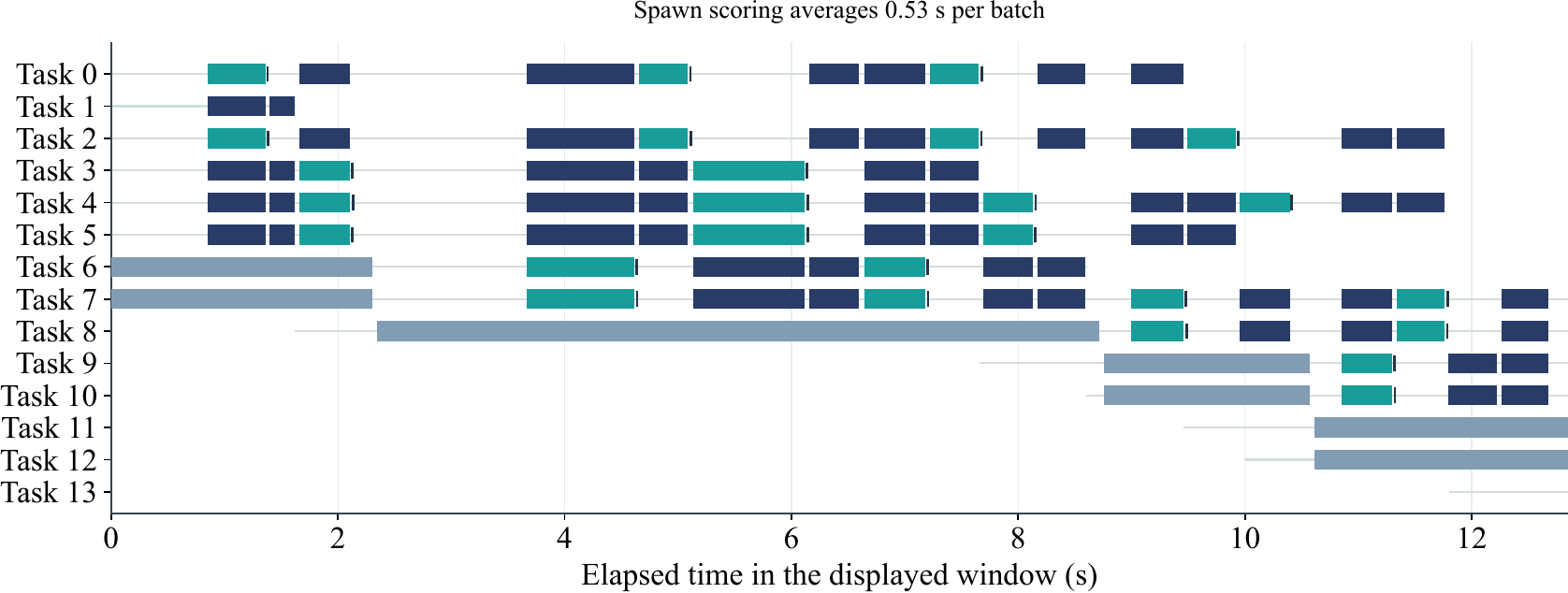}
\caption{Asynchronous execution}
\end{subfigure}
\par\medskip
\begin{subfigure}[t]{1.0\linewidth}
\centering
\includegraphics[width=\linewidth]{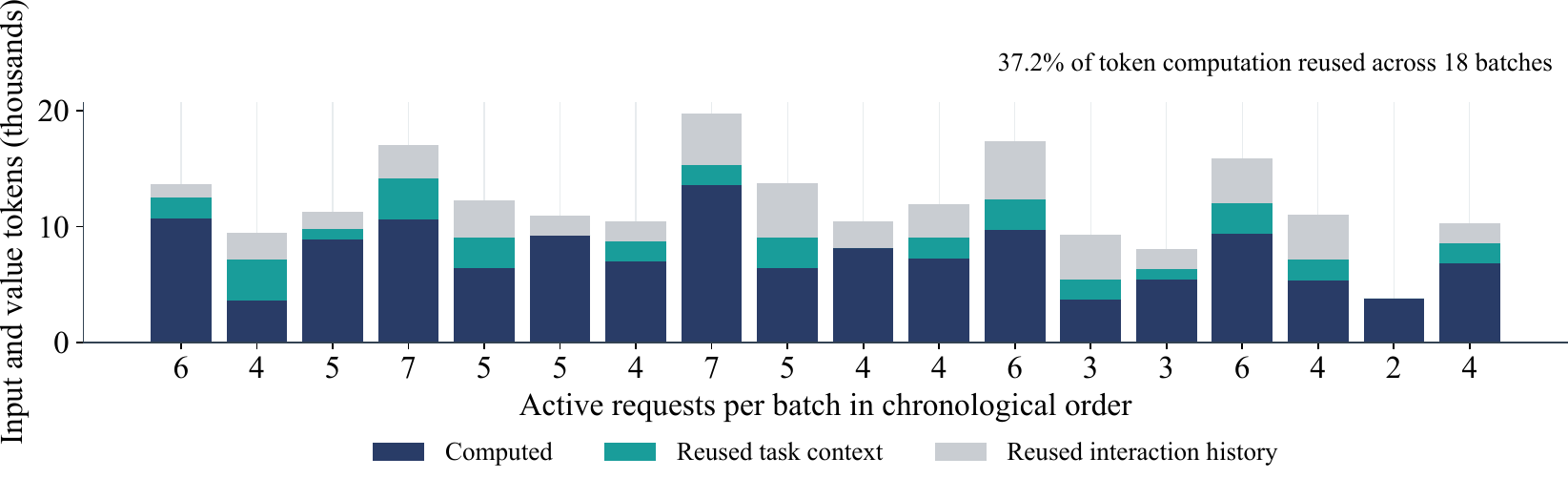}
\caption{Prefix reuse}
\end{subfigure}
\caption{Request scheduling and prefix reuse during JevSpawn execution on PPNL. (a) Request durations and returned observations during asynchronous execution. Teal spans indicate spawn scoring, with the mean batch latency annotated. Text generation can span intervening finite evaluations. (b) Computed and reused input tokens for successive finite evaluation batches.}
\label{fig:profile}
\end{figure}

In Figure~\ref{fig:profile}, batches containing spawn scoring or state selection average 0.53 seconds, and batches containing operation selection average 0.47 seconds. Different scoring operations can share a batch. During mixed text and finite evaluation, text generation averages 0.52 seconds TTFT and 137 milliseconds between emitted tokens. The isolated decoding experiment below measures token intervals without intervening finite requests.

\subsection{Action Probability Evaluation}

Profiling replays fixed requests in batches of eight. Per task, sixteen text requests and sixteen finite requests from each of the short and long history groups are used. Each batch has one warmup and three measured repetitions.

Isolated text decoding reaches 630 to 660 tokens per second with a mean ITL of 12 milliseconds. Prefill processes 29,000 to 33,000 uncached tokens per second. Finite evaluation takes 0.26 to 0.49 seconds per batch with short histories and 0.26 to 2.56 seconds with long histories, including context extension and probability computation.

Figure~\ref{fig:kernels} shows four GPU traces at a fixed number of alternatives. Longer uncached histories increase matrix computation and data movement. Collective communication includes device synchronization.

\begin{figure}[p]
\centering
\includegraphics[scale=0.593750]{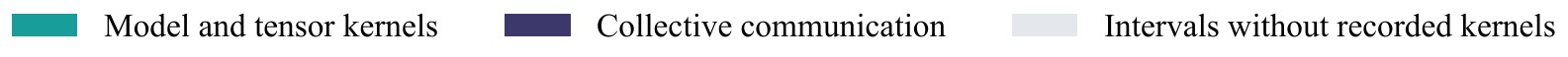}
\par\medskip
\begin{subfigure}[t]{0.48\linewidth}
\centering
\includegraphics[width=\linewidth]{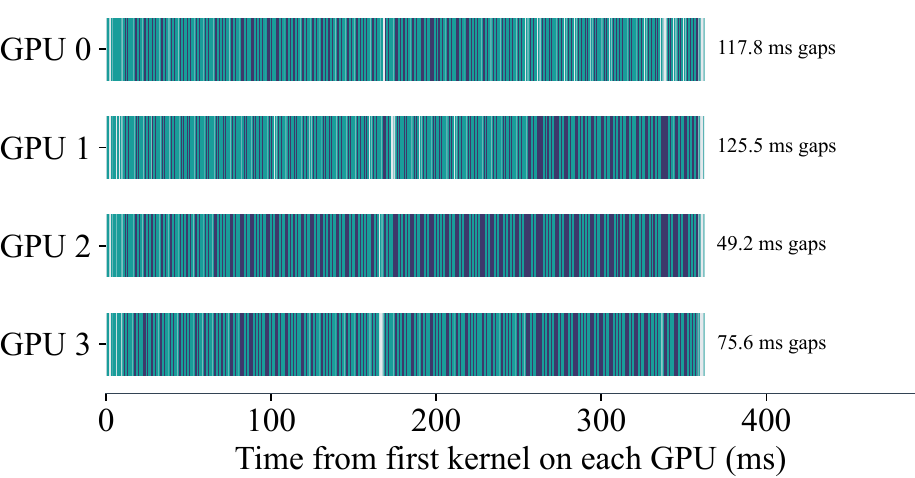}
\caption{Short history, 5,290 computed tokens}
\end{subfigure}
\hfill
\begin{subfigure}[t]{0.48\linewidth}
\centering
\includegraphics[width=\linewidth]{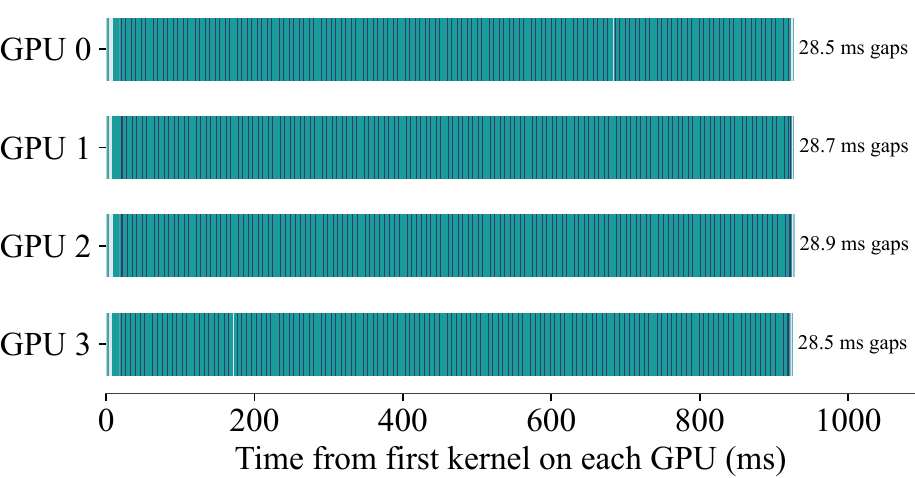}
\caption{Long history, 21,450 computed tokens}
\end{subfigure}
\caption{GPU activity during finite action evaluation of eight requests and 77 assignments. Teal indicates model computation, purple indicates collective communication, and light gray indicates gaps between kernels. Each timeline begins at the first kernel on the corresponding GPU.}
\label{fig:kernels}
\end{figure}

\subsection{TypeSafe Jev Comparison}

The TypeSafe Jev variant uses API model \texttt{jev-1.13.0} for finite scoring and Qwen for declarations and text generation. Both scorers receive the same task context and history under the common input limit. Branch exploration, action execution, and task evaluation are unchanged. E2E latency includes network and API service time.

\subsection{Further Performance Analysis}
\label{app:performance}

Among the seven agent baselines, AgentPrune achieves the highest score on four tasks. Figure~\ref{fig:completion} compares paired instance latencies and scores with JevSpawn. Timeout pairs are marked separately and excluded from the median latency curve.

\begin{figure}[p]
\centering
\includegraphics[scale=0.666667]{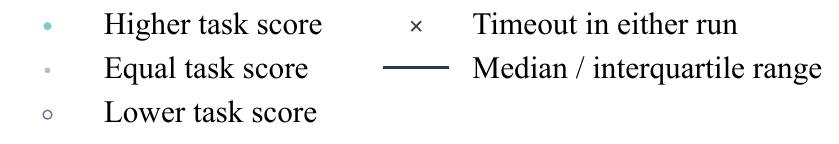}
\par\medskip
\begin{subfigure}[t]{0.48\linewidth}
\centering
\includegraphics[width=\linewidth]{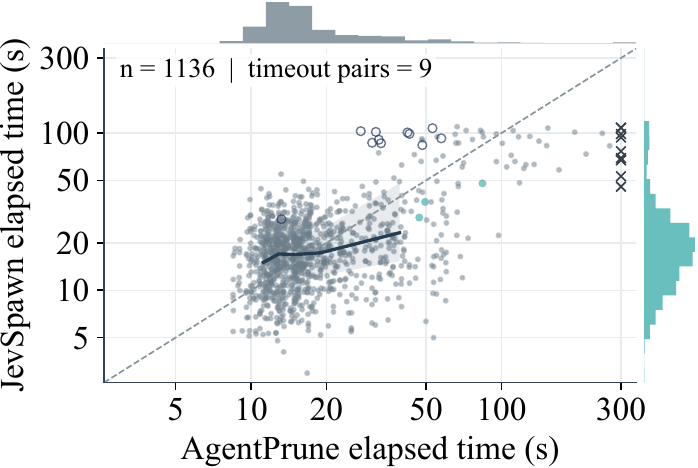}
\caption{PPNL}
\end{subfigure}
\hfill
\begin{subfigure}[t]{0.48\linewidth}
\centering
\includegraphics[width=\linewidth]{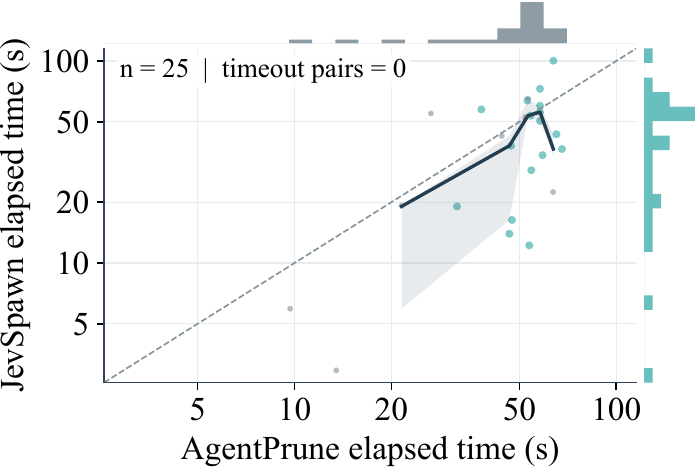}
\caption{Maze}
\end{subfigure}
\par\medskip
\begin{subfigure}[t]{0.48\linewidth}
\centering
\includegraphics[width=\linewidth]{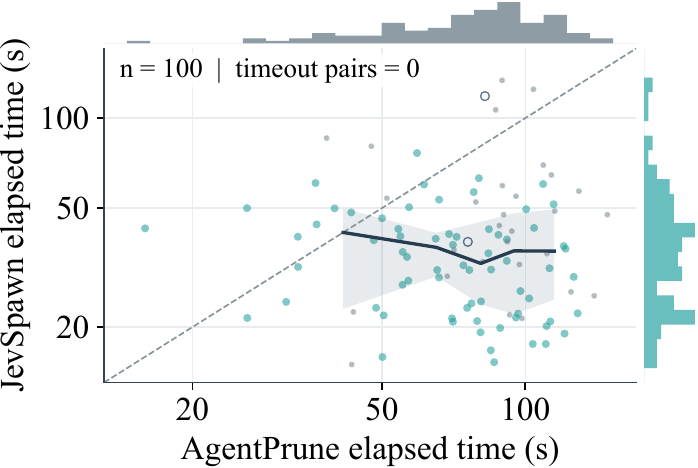}
\caption{Grid}
\end{subfigure}
\hfill
\begin{subfigure}[t]{0.48\linewidth}
\centering
\includegraphics[width=\linewidth]{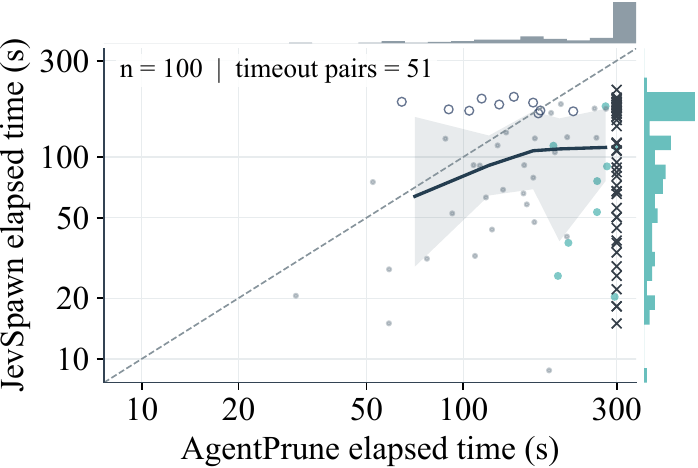}
\caption{LightsOut}
\end{subfigure}
\caption{Paired task times for JevSpawn and AgentPrune. Each point represents the same instance under both methods. Points below the diagonal indicate shorter JevSpawn execution. Filled teal, gray, and open navy marks denote higher, equal, and lower JevSpawn scores. Crosses identify pairs containing a timeout. Marginal histograms show elapsed time distributions. The dark curve and band summarize median and interquartile JevSpawn times within equal-count groups of AgentPrune times, excluding timeout pairs.}
\label{fig:completion}
\end{figure}

\begin{figure}[p]
\ContinuedFloat
\centering
\includegraphics[scale=0.666667]{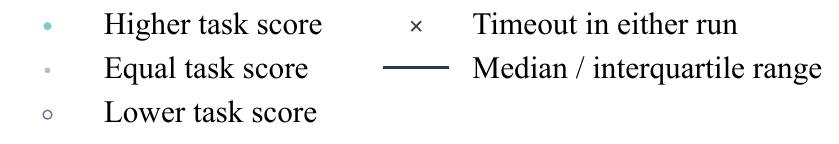}
\par\medskip
\begin{subfigure}[t]{0.48\linewidth}
\centering
\includegraphics[width=\linewidth]{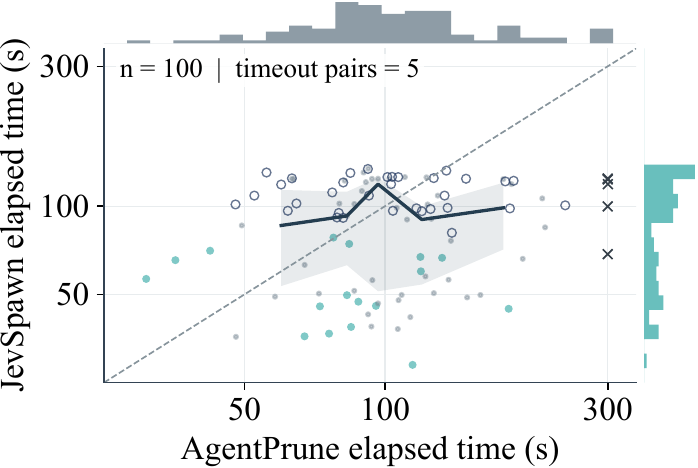}
\caption{RushHour}
\end{subfigure}
\hfill
\begin{subfigure}[t]{0.48\linewidth}
\centering
\includegraphics[width=\linewidth]{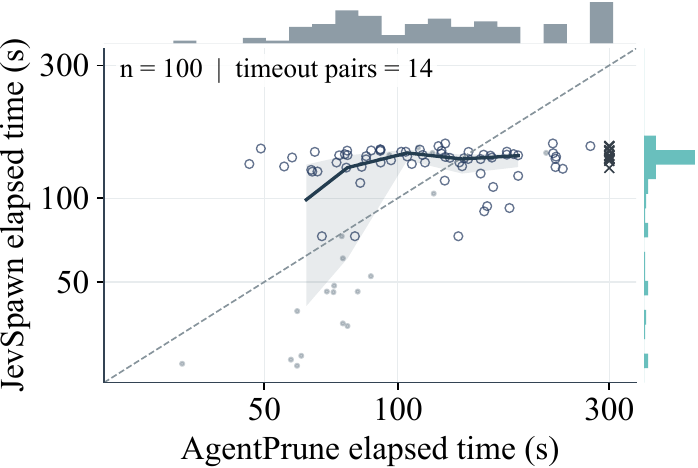}
\caption{Sokoban}
\end{subfigure}
\par\medskip
\begin{subfigure}[t]{0.48\linewidth}
\centering
\includegraphics[width=\linewidth]{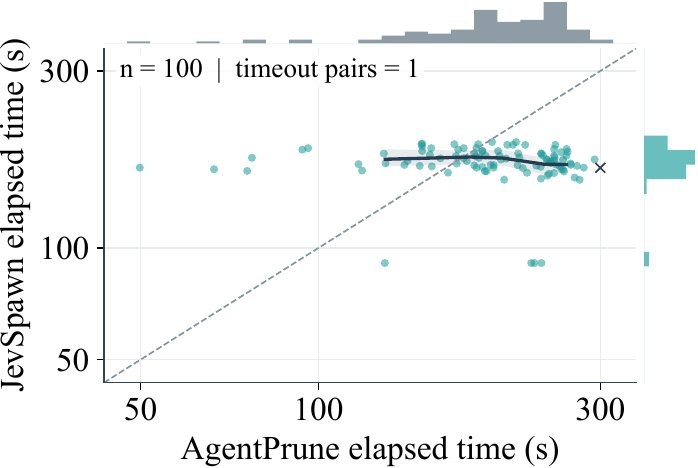}
\caption{2048}
\end{subfigure}
\hfill
\begin{subfigure}[t]{0.48\linewidth}
\centering
\includegraphics[width=\linewidth]{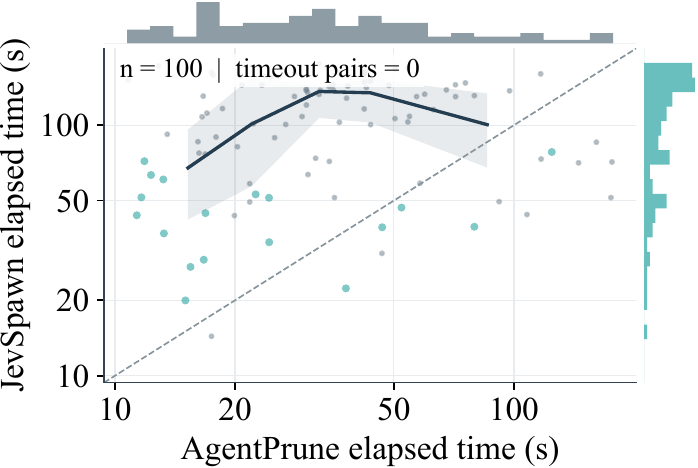}
\caption{Nullify}
\end{subfigure}
\caption[]{Continued. (e) RushHour, (f) Sokoban, (g) 2048, and (h) Nullify.}

\end{figure}

\input{prompts}

%% file: prompts.tex
\subsection{Core Prompts}
\label{app:prompts}

The following excerpts specify action construction, joint scoring, and feedback guided continuation. The task context and observed history accompany these instructions.

\begin{examplebox}[frametitle={Action construction}]

Declare the reusable commands for one environment transition. Use the smallest executable transaction accepted by the public interface. When a command permits a sequence of moves, declare one move and let subsequent turns execute subsequent moves after observing their results. When the interface requires a joint tuple in one call, declare that complete tuple. Output one signature per primitive command form. The same signature is reused at every applicable state. Fixed native prefixes, separators, spacing and punctuation remain literally in the signature. Replace each variable component with a named slot whose complete finite domain is declared inline as \{name:enum("value\_a", "value\_b")\} or \{name:range(start, stop)\}. Here range uses integers with an inclusive start and exclusive stop.

\end{examplebox}

\begin{examplebox}[frametitle={Joint action scoring}]

Select the joint field assignment whose resulting action makes the best progress toward the task's stated goal from the current state. Each option binds the pending fields together. Evaluate their compatibility, the action's applicability, and the latest feedback as one decision. Previously bound fields remain fixed.

\end{examplebox}

\begin{examplebox}[frametitle={Feedback guided continuation}]

Use the task rules, executed actions and returned observations. A stopped environment may represent failure. Determine the outcome from its observation. The following operation can continue execution, revise the action declaration, submit, or discard a stopped branch.

Choose the next operation for the selected branch using its active declaration and the latest actual action-observation pairs below. A declaration with a reported native syntax or argument-format error needs revise so that subsequent values produce executable commands. A well-formed command rejected by the current environment state needs a different value combination through expand, unless the required command or value is absent from the declaration. A completed task supported by actual observations needs submit.

\end{examplebox}

%% file: ablations.tex
\section{Ablation Studies}
\label{app:ablations}

\subsection{Expansion Width}

Expansion width controls the maximum number of child actions from one parent. Table~\ref{tab:ablation-width} compares widths of one, two, four, eight, and ten against the default width of four. Width one reduces latency but substantially lowers scores on Maze, Grid, and LightsOut. Wider expansion yields task dependent gains.

\input{tables/ablation_width}

\begin{figure}[p]
\centering

\begin{subfigure}[t]{0.24\linewidth}
\centering
\includegraphics[width=\linewidth]{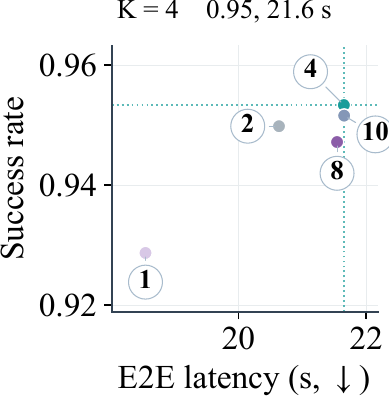}
\caption{PPNL}
\end{subfigure}
\hfill
\begin{subfigure}[t]{0.24\linewidth}
\centering
\includegraphics[width=\linewidth]{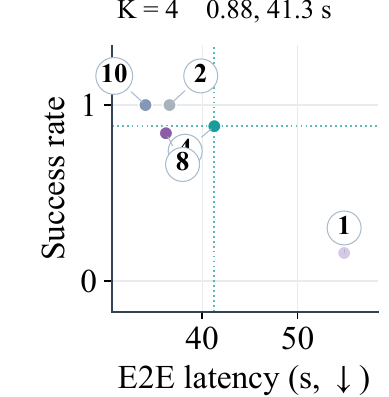}
\caption{Maze}
\end{subfigure}
\hfill
\begin{subfigure}[t]{0.24\linewidth}
\centering
\includegraphics[width=\linewidth]{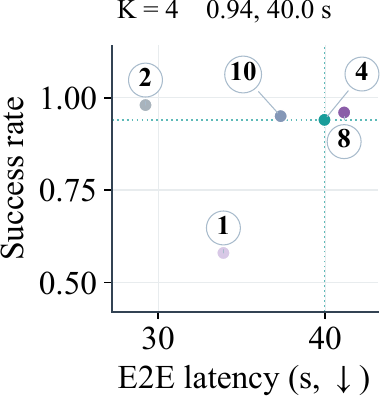}
\caption{Grid}
\end{subfigure}
\hfill
\begin{subfigure}[t]{0.24\linewidth}
\centering
\includegraphics[width=\linewidth]{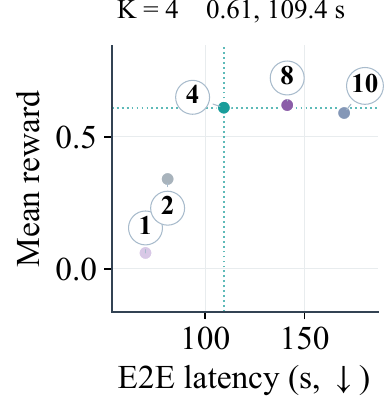}
\caption{LightsOut}
\end{subfigure}
\par\bigskip
\begin{subfigure}[t]{0.24\linewidth}
\centering
\includegraphics[width=\linewidth]{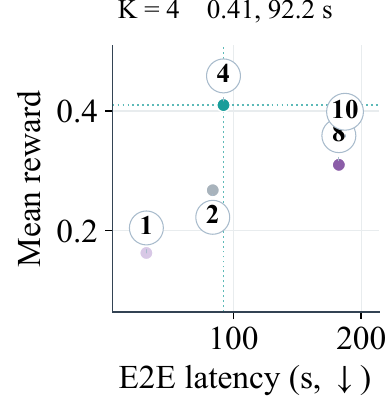}
\caption{RushHour}
\end{subfigure}
\hfill
\begin{subfigure}[t]{0.24\linewidth}
\centering
\includegraphics[width=\linewidth]{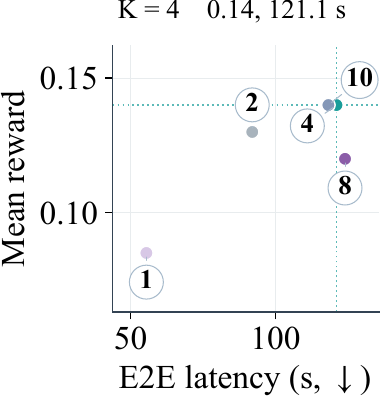}
\caption{Sokoban}
\end{subfigure}
\hfill
\begin{subfigure}[t]{0.24\linewidth}
\centering
\includegraphics[width=\linewidth]{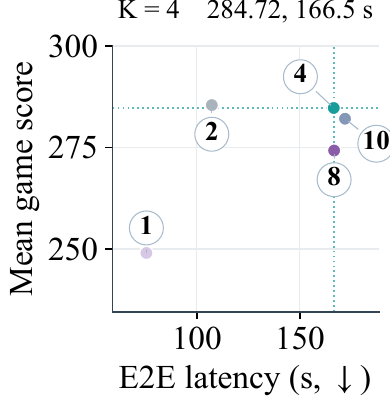}
\caption{2048}
\end{subfigure}
\hfill
\begin{subfigure}[t]{0.24\linewidth}
\centering
\includegraphics[width=\linewidth]{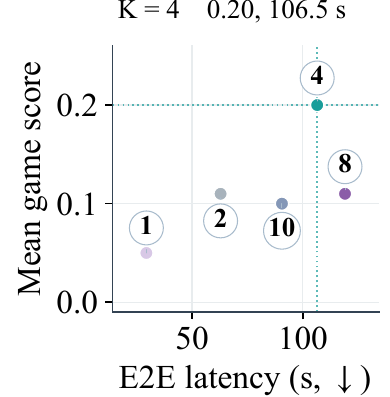}
\caption{Nullify}
\end{subfigure}
\caption{Task performance and E2E latency across expansion widths. Circled labels indicate the maximum number of spawned actions, $K$. Dotted lines mark the reference setting, $K=4$.}
\label{fig:width}
\end{figure}

\subsection{Expansion Topology}

Table~\ref{tab:ablation-topology} allocates four actions as $1\times4$, $2\times2$, or $4\times1$ parent--child configurations. Two parents achieve the highest Maze and Grid success, while four parents achieve the highest 2048 score. The action budget is fixed across configurations.

\input{tables/ablation_topology}

\subsection{Declaration Revision}

Declaration revision is disabled after the first executable declaration in Table~\ref{tab:ablation-revision}. Feedback and branch selection remain active. Fixed declarations improve scores on four tasks and reduce latency on six, while adaptive declarations score higher on the other four tasks.

\input{tables/ablation_components}

\subsection{Branch Information Sharing}

Table~\ref{tab:ablation-sharing} restricts observations to the current branch history. Shared observations reduce latency on all eight tasks and improve scores on PPNL, Sokoban, 2048, and Nullify. Restricting observations improves scores on Maze, Grid, LightsOut, and RushHour.

\subsection{Interaction Length and Time Budget}

Round limits of 36, 54, 72, and 108 are evaluated at a context limit of 131,072 tokens, with all other settings fixed. Each condition is run with a 300 second deadline and without a deadline. Table~\ref{tab:turns} reports scores, E2E latency, and timeout rates for the timed series. Figure~\ref{fig:turns} shows the score trends.

\input{tables/turns}

\begin{figure}[p]
\centering
\includegraphics[scale=0.769231]{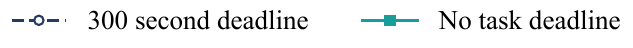}\par\medskip
\begin{subfigure}[t]{0.24\linewidth}
\centering
\includegraphics[width=\linewidth]{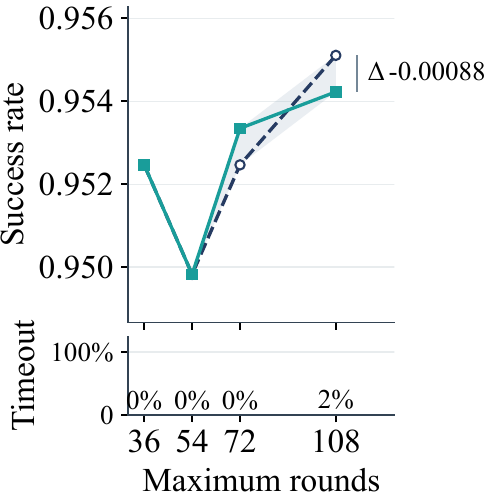}
\caption{PPNL}
\end{subfigure}
\hfill
\begin{subfigure}[t]{0.24\linewidth}
\centering
\includegraphics[width=\linewidth]{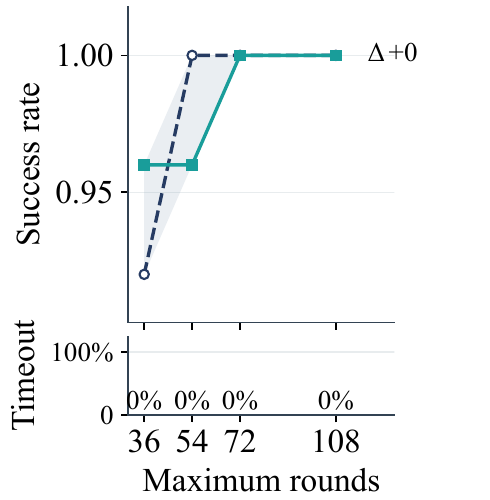}
\caption{Maze}
\end{subfigure}
\hfill
\begin{subfigure}[t]{0.24\linewidth}
\centering
\includegraphics[width=\linewidth]{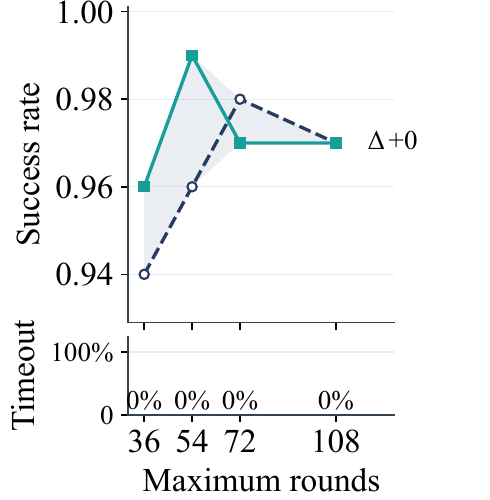}
\caption{Grid}
\end{subfigure}
\hfill
\begin{subfigure}[t]{0.24\linewidth}
\centering
\includegraphics[width=\linewidth]{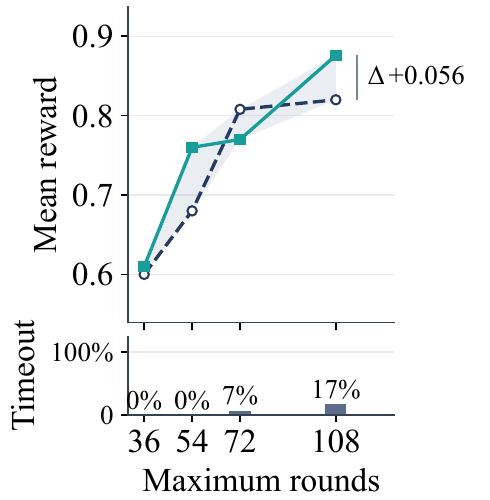}
\caption{LightsOut}
\end{subfigure}
\par\bigskip
\begin{subfigure}[t]{0.24\linewidth}
\centering
\includegraphics[width=\linewidth]{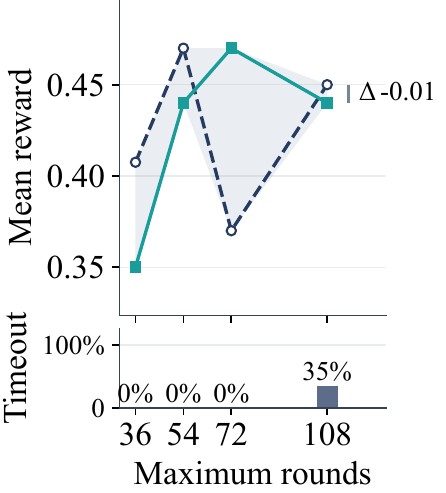}
\caption{RushHour}
\end{subfigure}
\hfill
\begin{subfigure}[t]{0.24\linewidth}
\centering
\includegraphics[width=\linewidth]{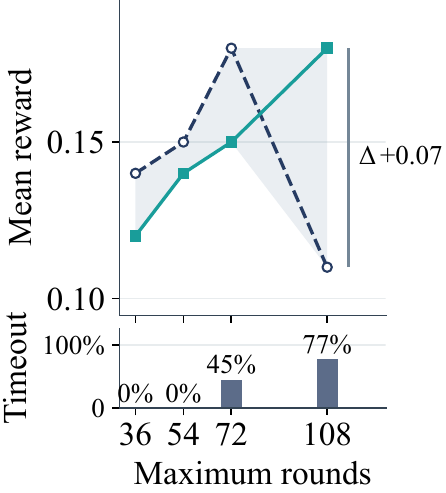}
\caption{Sokoban}
\end{subfigure}
\hfill
\begin{subfigure}[t]{0.24\linewidth}
\centering
\includegraphics[width=\linewidth]{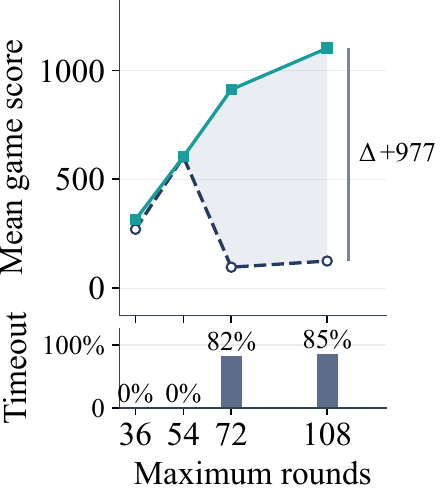}
\caption{2048}
\end{subfigure}
\hfill
\begin{subfigure}[t]{0.24\linewidth}
\centering
\includegraphics[width=\linewidth]{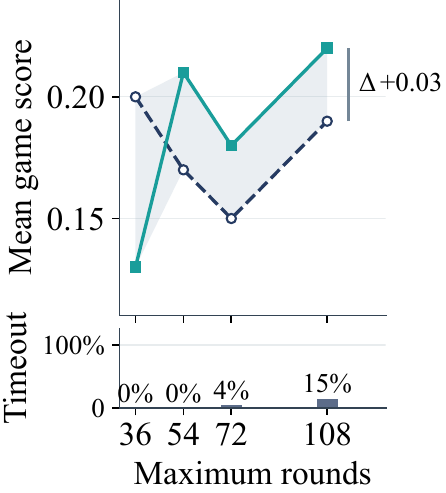}
\caption{Nullify}
\end{subfigure}
\caption{Task performance and timeout incidence across interaction horizons. Shading separates scores with a fixed deadline and unrestricted time. Brackets mark the gap at 108 rounds, with timeout rates below.}
\label{fig:turns}
\end{figure}

Without a deadline, the 2048 score rises from 315.16 at 36 rounds to 603.72, 912.52, and 1102.24 at 54, 72, and 108 rounds. LightsOut reward rises from 0.61 to 0.876. Navigation success changes little, and RushHour reward peaks at 0.47 at 72 rounds before declining to 0.44.

With the deadline enforced, 82 and 85 percent of 2048 tasks time out at 72 and 108 rounds, reducing scores to 96.60 and 124.84. Sokoban reaches a timeout rate of 77 percent at 108 rounds. Timed out tasks receive zero reward.

At 108 rounds without a deadline, mean E2E latency reaches 464.67 seconds on 2048 and 415.77 seconds on Sokoban, with 99th percentiles of 612 and 559 seconds. Sokoban reward increases from 0.12 at 36 rounds to 0.18. In 2048, 92 percent of tasks exceed 300 seconds and nine percent exceed 600 seconds. The longest observed run takes 623 seconds.

Removing only the deadline at the main context limit of 16,384 tokens and 36 rounds yields success rates of 0.951, 0.96, and 0.94 on PPNL, Maze, and Grid.

\subsection{Batch Scaling}

Batch scaling replays 128 fixed requests for each of text generation and finite evaluation, with one warmup and three measured repetitions per batch. Throughput includes input processing. Finite evaluation is separated by history length and measured in requests per second, with each request scoring one option set.

\begin{figure}[p]
\centering
\begin{subfigure}[t]{0.48\linewidth}
\centering
\includegraphics[width=\linewidth]{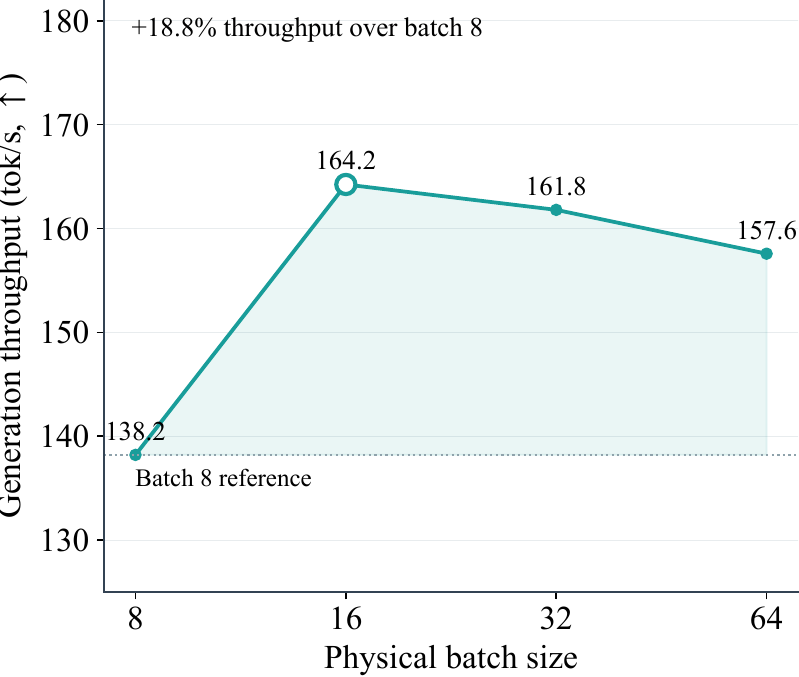}
\caption{Throughput scaling}
\end{subfigure}
\hfill
\begin{subfigure}[t]{0.48\linewidth}
\centering
\includegraphics[width=\linewidth]{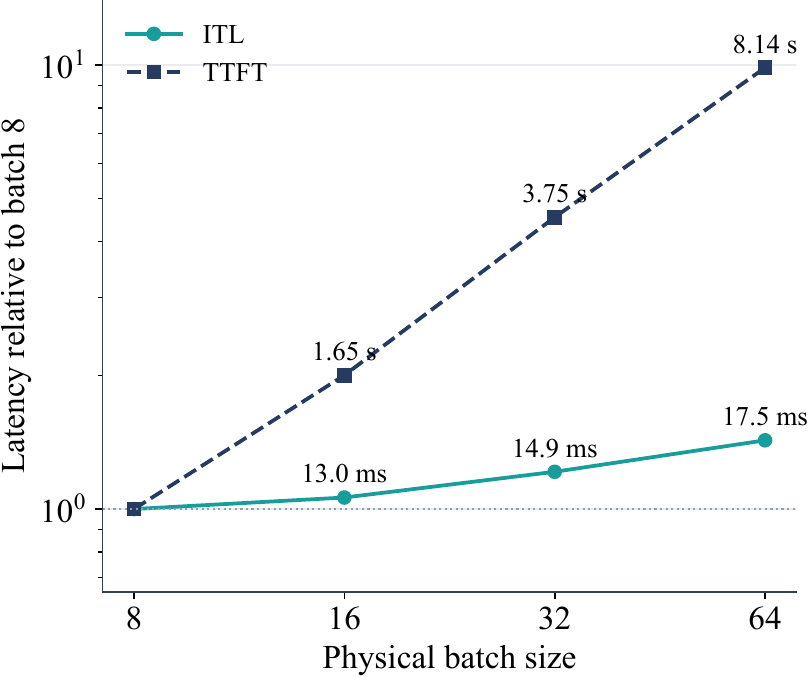}
\caption{Latency scaling}
\end{subfigure}
\par\medskip
\begin{subfigure}[t]{0.48\linewidth}
\centering
\includegraphics[width=\linewidth]{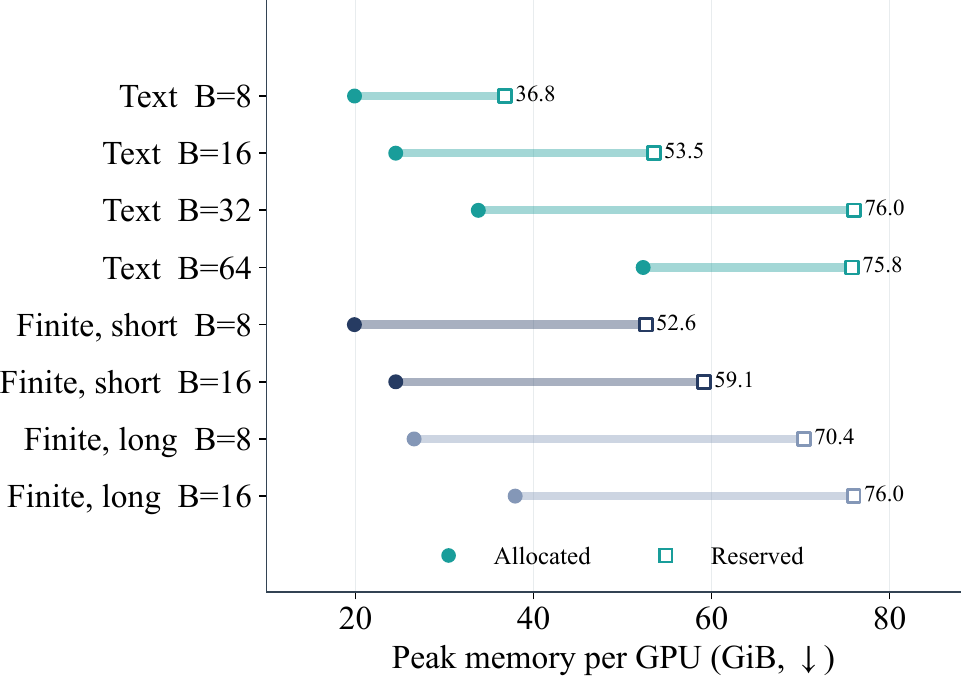}
\caption{GPU memory}
\end{subfigure}
\hfill
\begin{subfigure}[t]{0.48\linewidth}
\centering
\includegraphics[width=\linewidth]{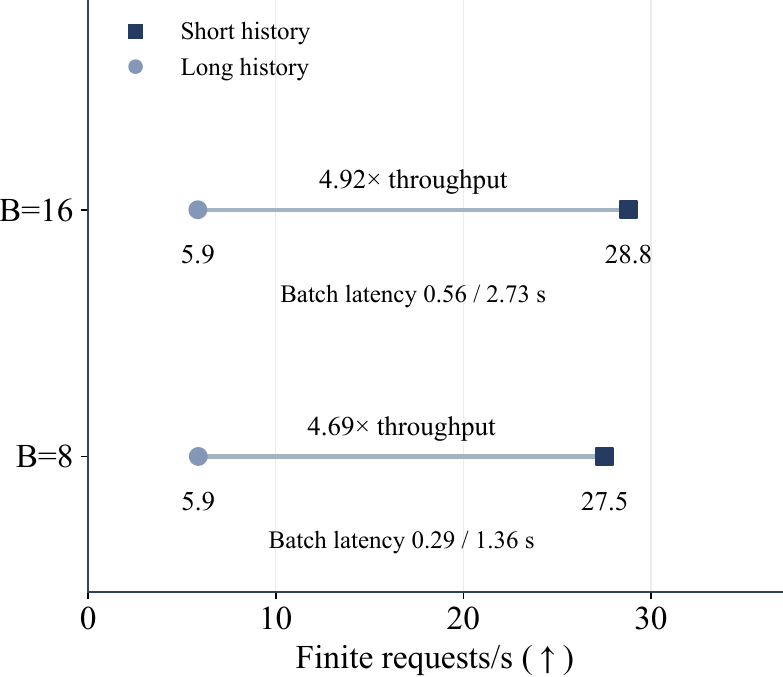}
\caption{Effect of history length}
\end{subfigure}
\caption{Throughput, latency, and memory under increasing physical batch size. (a) The throughput curve highlights the best measured batch and the gain over batch 8. (b) ITL and TTFT are normalized to batch 8, with absolute values annotated. (c) Memory intervals connect allocated and reserved maxima across four GPUs. (d) Paired finite throughput measurements contrast short and long histories, with ratios and batch latencies shown. Finite throughput includes context processing. The next tested text and finite batches, 128 and 32, exceed memory.}
\label{fig:batch}
\end{figure}

Text generation throughput, including prefill, peaks at batch size 16, increasing from 138.2 to 164.2 tokens per second relative to batch size 8. Mean ITL rises from 12.28 to 17.54 milliseconds between batch sizes 8 and 64. Finite throughput remains near 28 requests per second for short histories and six for long histories at both batch sizes 8 and 16.

The largest completed batch sizes are 64 for text generation and 16 for finite evaluation. The next tested sizes, 128 and 32, exceed GPU memory.

%% file: tables/ablation_width.tex
\begin{table}[p]
\centering
\caption{Effect of exploration width on score ($\uparrow$) and E2E latency (s, $\downarrow$). Bold values mark the best result for each metric and task, including ties.}
\label{tab:ablation-width}
\small\begin{adjustbox}{max width=\linewidth}
\begin{tabular}{llrrrrrrrr}
\toprule
Width & Metric &  PPNL &  Maze &  Grid &  LightsOut &  RushHour &  Sokoban &  2048 &  Nullify \\
\midrule
\multirow{2}{*}{1} & Score & 0.929 & 0.160 & 0.580 & 0.060 & 0.163 & 0.085 & 249.000 & 0.050 \\
 & E2E & \textbf{18.54} & 54.87 & 33.91 & \textbf{69.93} & \textbf{31.90} & \textbf{55.45} & \textbf{75.51} & \textbf{29.42} \\
\midrule
\multirow{2}{*}{2} & Score & 0.950 & \textbf{1.000} & \textbf{0.980} & 0.340 & 0.268 & 0.130 & \textbf{285.440} & 0.110 \\
 & E2E & 20.63 & 36.58 & \textbf{29.24} & 81.07 & 83.82 & 91.99 & 107.26 & 62.89 \\
\midrule
\multirow{2}{*}{4} & Score & \textbf{0.953} & 0.880 & 0.940 & 0.610 & \textbf{0.410} & \textbf{0.140} & 284.720 & \textbf{0.200} \\
 & E2E & 21.65 & 41.28 & 39.96 & 109.37 & 92.22 & 121.09 & 166.46 & 106.47 \\
\midrule
\multirow{2}{*}{8} & Score & 0.947 & 0.840 & 0.960 & \textbf{0.620} & 0.310 & 0.120 & 274.240 & 0.110 \\
 & E2E & 21.54 & 36.21 & 41.13 & 141.26 & 182.45 & 124.02 & 166.64 & 119.03 \\
\midrule
\multirow{2}{*}{10} & Score & 0.952 & \textbf{1.000} & 0.950 & 0.590 & 0.350 & \textbf{0.140} & 282.080 & 0.100 \\
 & E2E & 21.65 & \textbf{34.07} & 37.33 & 169.90 & 187.27 & 118.28 & 171.93 & 90.57 \\
\bottomrule
\end{tabular}\end{adjustbox}
\end{table}

%% file: tables/ablation_topology.tex
\begin{table}[p]
\centering
\caption{Effect of exploration topology under a fixed expansion budget. Score ($\uparrow$) and E2E latency (s, $\downarrow$) are reported. Bold values mark the best result for each metric and task, including ties.}
\label{tab:ablation-topology}
\small\begin{adjustbox}{max width=\linewidth}
\begin{tabular}{llrrrrrrrr}
\toprule
Parents $\times$ children & Metric &  PPNL &  Maze &  Grid &  LightsOut &  RushHour &  Sokoban &  2048 &  Nullify \\
\midrule
\multirow{2}{*}{$1\times4$} & Score & \textbf{0.953} & 0.880 & 0.940 & \textbf{0.610} & \textbf{0.410} & 0.140 & 284.720 & \textbf{0.200} \\
 & E2E & \textbf{21.65} & 41.28 & 39.96 & \textbf{109.37} & 92.22 & 121.09 & \textbf{166.46} & 106.47 \\
\midrule
\multirow{2}{*}{$2\times2$} & Score & 0.949 & \textbf{1.000} & \textbf{1.000} & 0.410 & 0.290 & \textbf{0.190} & 307.040 & 0.150 \\
 & E2E & 22.24 & \textbf{27.79} & \textbf{25.89} & 166.35 & 104.24 & 134.25 & 181.22 & 78.57 \\
\midrule
\multirow{2}{*}{$4\times1$} & Score & 0.933 & 0.120 & 0.750 & 0.090 & 0.060 & 0.120 & \textbf{350.880} & 0.070 \\
 & E2E & 27.88 & 146.46 & 63.11 & 263.20 & \textbf{51.30} & \textbf{119.74} & 217.78 & \textbf{61.29} \\
\bottomrule
\end{tabular}\end{adjustbox}
\end{table}

%% file: tables/ablation_components.tex
\begin{table}[p]
\centering
\caption{Ablations of action declaration revision and cross-branch observation sharing. Score ($\uparrow$) and E2E latency in seconds ($\downarrow$) are reported for each setting. Bold values mark the best result for each metric and task, including ties.}
\label{tab:ablation-revision}
\label{tab:ablation-sharing}
\small\begin{adjustbox}{max width=\linewidth}
\begin{tabular}{lllrrrrrrrr}
\toprule
Revision & Sharing & Metric &  PPNL &  Maze &  Grid &  LightsOut &  RushHour &  Sokoban &  2048 &  Nullify \\
\midrule
\multirow{2}{*}{w/} & \multirow{2}{*}{w/} & Score & \textbf{0.953} & 0.880 & 0.940 & 0.610 & 0.410 & 0.140 & 284.720 & \textbf{0.200} \\
 &  & E2E & 21.65 & 41.28 & 39.96 & 109.37 & 92.22 & 121.09 & \textbf{166.46} & \textbf{106.47} \\
\midrule
\multirow{2}{*}{w/o} & \multirow{2}{*}{w/} & Score & 0.945 & \textbf{0.920} & 0.930 & \textbf{0.700} & 0.360 & \textbf{0.155} & \textbf{293.920} & 0.170 \\
 &  & E2E & \textbf{21.40} & \textbf{36.94} & \textbf{38.61} & \textbf{97.45} & \textbf{91.91} & \textbf{118.41} & 166.53 & 106.67 \\
\midrule
\multirow{2}{*}{w/} & \multirow{2}{*}{w/o} & Score & 0.952 & \textbf{0.920} & \textbf{0.950} & 0.690 & \textbf{0.480} & 0.120 & 274.160 & 0.180 \\
 &  & E2E & 24.05 & 48.59 & 49.70 & 116.54 & 103.22 & 169.54 & 206.02 & 147.35 \\
\bottomrule
\end{tabular}\end{adjustbox}
\end{table}

%% file: tables/turns.tex
\begin{table}[p]
\centering
\caption{Interaction length with and without a 300 s time budget. Score ($\uparrow$), E2E latency in seconds ($\downarrow$), and timeout rate TO in percent ($\downarrow$) are reported. Bold values mark the best result within each budget condition for each metric and task, including ties.}
\label{tab:turns}
\small\begin{adjustbox}{max width=\linewidth}
\begin{tabular}{lllrrrrrrrr}
\toprule
Setting & Rounds & Metric &  PPNL &  Maze &  Grid &  LightsOut &  RushHour &  Sokoban &  2048 &  Nullify \\
\midrule
\multirow{12}{*}{w/ time budget} & \multirow{3}{*}{36} & Score & 0.952 & 0.920 & 0.940 & 0.600 & 0.407 & 0.140 & 270.760 & \textbf{0.200} \\
 &  & E2E & \textbf{21.7} & 35.1 & \textbf{37.5} & \textbf{79.3} & \textbf{89.7} & \textbf{106.1} & \textbf{131.8} & \textbf{101.6} \\
 &  & TO & \textbf{0.0} & \textbf{0.0} & \textbf{0.0} & \textbf{0.0} & \textbf{0.0} & \textbf{0.0} & \textbf{0.0} & \textbf{0.0} \\
\addlinespace
 & \multirow{3}{*}{54} & Score & 0.950 & \textbf{1.000} & 0.960 & 0.680 & \textbf{0.470} & 0.150 & \textbf{602.960} & 0.170 \\
 &  & E2E & 23.7 & 43.4 & 42.0 & 110.9 & 112.9 & 171.3 & 216.9 & 123.9 \\
 &  & TO & \textbf{0.0} & \textbf{0.0} & \textbf{0.0} & \textbf{0.0} & \textbf{0.0} & \textbf{0.0} & \textbf{0.0} & \textbf{0.0} \\
\addlinespace
 & \multirow{3}{*}{72} & Score & 0.952 & \textbf{1.000} & \textbf{0.980} & 0.808 & 0.370 & \textbf{0.180} & 96.600 & 0.150 \\
 &  & E2E & 26.2 & \textbf{34.5} & 41.0 & 123.1 & 177.3 & 236.6 & 290.6 & 140.0 \\
 &  & TO & \textbf{0.0} & \textbf{0.0} & \textbf{0.0} & 7.0 & \textbf{0.0} & 45.0 & 82.0 & 4.0 \\
\addlinespace
 & \multirow{3}{*}{108} & Score & \textbf{0.955} & \textbf{1.000} & 0.970 & \textbf{0.820} & 0.450 & 0.110 & 124.840 & 0.190 \\
 &  & E2E & 29.0 & 38.7 & 38.1 & 130.6 & 193.1 & 266.2 & 293.6 & 160.1 \\
 &  & TO & 1.9 & \textbf{0.0} & \textbf{0.0} & 17.0 & 35.0 & 77.0 & 85.0 & 15.0 \\
\midrule
\multirow{8}{*}{w/o time budget} & \multirow{2}{*}{36} & Score & 0.952 & 0.960 & 0.960 & 0.610 & 0.350 & 0.120 & 315.160 & 0.130 \\
 &  & E2E & \textbf{21.8} & \textbf{34.7} & \textbf{34.6} & \textbf{81.8} & \textbf{93.7} & \textbf{111.0} & \textbf{127.8} & \textbf{99.5} \\
\addlinespace
 & \multirow{2}{*}{54} & Score & 0.950 & 0.960 & \textbf{0.990} & 0.760 & 0.440 & 0.140 & 603.720 & 0.210 \\
 &  & E2E & 24.1 & 43.2 & 37.1 & 103.6 & 119.4 & 173.9 & 220.7 & 121.0 \\
\addlinespace
 & \multirow{2}{*}{72} & Score & 0.953 & \textbf{1.000} & 0.970 & 0.770 & \textbf{0.470} & 0.150 & 912.520 & 0.180 \\
 &  & E2E & 25.6 & 36.2 & 35.5 & 127.9 & 158.2 & 246.8 & 321.6 & 147.7 \\
\addlinespace
 & \multirow{2}{*}{108} & Score & \textbf{0.954} & \textbf{1.000} & 0.970 & \textbf{0.876} & 0.440 & \textbf{0.180} & \textbf{1102.240} & \textbf{0.220} \\
 &  & E2E & 30.4 & 49.1 & 39.3 & 154.9 & 246.8 & 415.8 & 464.7 & 203.8 \\
\bottomrule
\end{tabular}\end{adjustbox}
\end{table}

%% file: cases.tex
\begingroup
\mdfsetup{skipabove=8pt,skipbelow=6pt,innertopmargin=6pt,innerbottommargin=6pt}
\section{Case Studies}
\label{app:cases}
One instance from each benchmark task illustrates how finite action fields support exploration. JevSpawn attains the maximum reward on seven selected instances, whereas most baselines do not complete the corresponding task. The 2048 example is compared by accumulated score. Consecutive rounds are shown in alternating colors, with observations condensed to the returned state changes and terminal feedback. Branch identifiers link states across rounds, and bold identifiers mark the submitted path. Recovery in a round title denotes a return to an earlier branch. Field identifiers label the slots in each action template.
\subsection{PPNL}
The goal lies beyond two obstacles in the top row. The selected path descends to the second row, crosses below both obstacles, and returns to the goal in seven moves.
\begin{examplebox}[frametitle={Task and outcome},linecolor=black]
Additional obstacle test, instance 164. JevSpawn score 1.
\par\smallskip\centering
\small\begin{adjustbox}{max width=\linewidth}\begin{tabular}{rrrrrrr}
LATS & LLMCompiler & AgentPrune & HiAgent & FoldAgent & DyFlow & LatentMAS \\
0 & 0 & 1 & 0 & 0 & 1 & 0 \\
\end{tabular}\end{adjustbox}
\end{examplebox}
\begin{center}\captionsetup{hypcap=false}
\includegraphics[width=\linewidth]{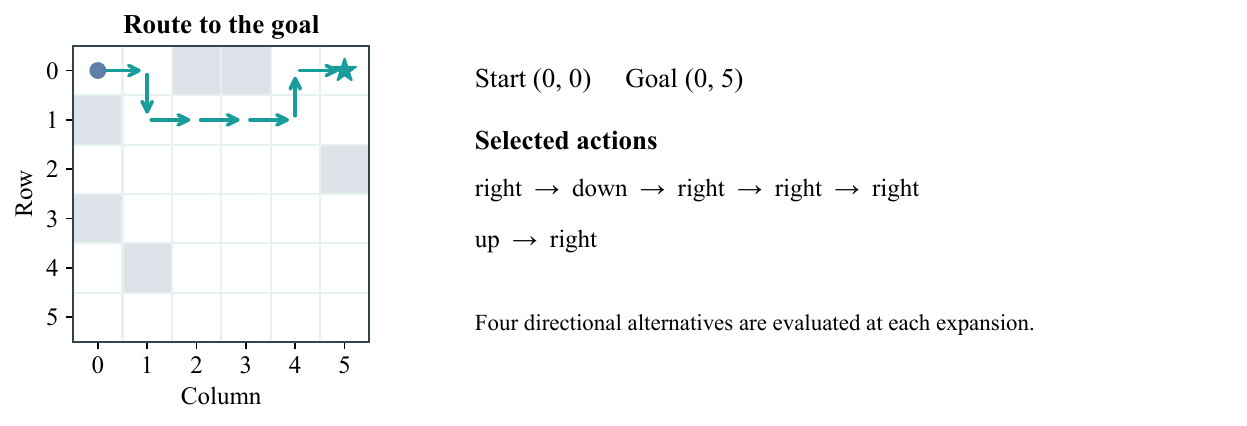}
\captionof{figure}{PPNL path around the upper row obstacles. The circle and star mark the start and goal. Arrows show the submitted moves.}
\end{center}
\begin{examplebox}[frametitle={Round 1},linecolor=black,backgroundcolor=caseblue!9,frametitlebackgroundcolor=caseblue!20,nobreak=true]
\small Initial state.
\par\smallskip Fields \texttt{f0} (movement direction) $\in$ \{\texttt{up}, \texttt{down}, \texttt{left}, \texttt{right}\}.
\par Action template \texttt{execute(actions = \$\{f0\})}.
\par\smallskip\noindent
\begin{tabularx}{\linewidth}{@{}l l r X@{}}
Branch & Action & $\pi_{\theta}$ & Returned observation \\ \hline
n0 & \texttt{down} & 0.480 & After 0 steps, I am at (0,0). If I perform the next step I will run into the obstacle at (1,0). \\
n1 & \texttt{left} & 0.002 & After 0 steps, I am at (0,0). Performing the next action would lead me outside the grid. \\
\textbf{n2} & \texttt{right} & 0.510 & Position (0,1). Goal not reached. \\
n3 & \texttt{up} & 0.008 & After 0 steps, I am at (0,0). Performing the next action would lead me outside the grid. \\
\end{tabularx}
\end{examplebox}
\begin{examplebox}[frametitle={Round 2},linecolor=black,backgroundcolor=casegreen!9,frametitlebackgroundcolor=casegreen!20,nobreak=true]
\small Selected branch \texttt{n2}, selection probability 0.996.
\par\smallskip\noindent
\begin{tabularx}{\linewidth}{@{}l l r X@{}}
Branch & Action & $\pi_{\theta}$ & Returned observation \\ \hline
\textbf{n4} & \texttt{down} & 0.049 & Position (1,1). Goal not reached. \\
n5 & \texttt{left} & 0.008 & Position (0,0). Goal not reached. \\
n6 & \texttt{right} & 0.933 & After 0 steps, I am at (0,1). If I perform the next step I will run into the obstacle at (0,2). \\
n7 & \texttt{up} & 0.010 & After 0 steps, I am at (0,1). Performing the next action would lead me outside the grid. \\
\end{tabularx}
\end{examplebox}
\begin{examplebox}[frametitle={Round 3},linecolor=black,backgroundcolor=caseblue!9,frametitlebackgroundcolor=caseblue!20,nobreak=true]
\small Selected branch \texttt{n4}, selection probability 0.970.
\par\smallskip\noindent
\begin{tabularx}{\linewidth}{@{}l l r X@{}}
Branch & Action & $\pi_{\theta}$ & Returned observation \\ \hline
n8 & \texttt{down} & 0.523 & Position (2,1). Goal not reached. \\
n9 & \texttt{left} & 0.029 & After 0 steps, I am at (1,1). If I perform the next step I will run into the obstacle at (1,0). \\
\textbf{n10} & \texttt{right} & 0.440 & Position (1,2). Goal not reached. \\
n11 & \texttt{up} & 0.008 & Position (0,1). Goal not reached. \\
\end{tabularx}
\end{examplebox}
\begin{examplebox}[frametitle={Round 4},linecolor=black,backgroundcolor=casegreen!9,frametitlebackgroundcolor=casegreen!20,nobreak=true]
\small Selected branch \texttt{n10}, selection probability 0.510.
\par\smallskip\noindent
\begin{tabularx}{\linewidth}{@{}l l r X@{}}
Branch & Action & $\pi_{\theta}$ & Returned observation \\ \hline
n12 & \texttt{down} & 0.517 & Position (2,2). Goal not reached. \\
n13 & \texttt{left} & 0.004 & Position (1,1). Goal not reached. \\
\textbf{n14} & \texttt{right} & 0.455 & Position (1,3). Goal not reached. \\
n15 & \texttt{up} & 0.024 & After 0 steps, I am at (1,2). If I perform the next step I will run into the obstacle at (0,2). \\
\end{tabularx}
\end{examplebox}
\begin{examplebox}[frametitle={Round 5},linecolor=black,backgroundcolor=caseblue!9,frametitlebackgroundcolor=caseblue!20,nobreak=true]
\small Selected branch \texttt{n14}, selection probability 0.562.
\par\smallskip\noindent
\begin{tabularx}{\linewidth}{@{}l l r X@{}}
Branch & Action & $\pi_{\theta}$ & Returned observation \\ \hline
n16 & \texttt{down} & 0.406 & Position (2,3). Goal not reached. \\
n17 & \texttt{left} & 0.007 & Position (1,2). Goal not reached. \\
\textbf{n18} & \texttt{right} & 0.544 & Position (1,4). Goal not reached. \\
n19 & \texttt{up} & 0.043 & After 0 steps, I am at (1,3). If I perform the next step I will run into the obstacle at (0,3). \\
\end{tabularx}
\end{examplebox}
\begin{examplebox}[frametitle={Round 6},linecolor=black,backgroundcolor=casegreen!9,frametitlebackgroundcolor=casegreen!20,nobreak=true]
\small Selected branch \texttt{n18}, selection probability 0.824.
\par\smallskip\noindent
\begin{tabularx}{\linewidth}{@{}l l r X@{}}
Branch & Action & $\pi_{\theta}$ & Returned observation \\ \hline
n20 & \texttt{down} & 0.155 & Position (2,4). Goal not reached. \\
n21 & \texttt{left} & 0.006 & Position (1,3). Goal not reached. \\
n22 & \texttt{right} & 0.758 & Position (1,5). Goal not reached. \\
\textbf{n23} & \texttt{up} & 0.081 & Position (0,4). Goal not reached. \\
\end{tabularx}
\end{examplebox}
\begin{examplebox}[frametitle={Round 7},linecolor=black,backgroundcolor=caseblue!9,frametitlebackgroundcolor=caseblue!20,nobreak=true]
\small Selected branch \texttt{n23}, selection probability 0.992.
\par\smallskip\noindent
\begin{tabularx}{\linewidth}{@{}l l r X@{}}
Branch & Action & $\pi_{\theta}$ & Returned observation \\ \hline
n24 & \texttt{down} & 0.009 & Position (1,4). Goal not reached. \\
n25 & \texttt{left} & 0.004 & After 0 steps, I am at (0,4). If I perform the next step I will run into the obstacle at (0,3). \\
\textbf{n26} & \texttt{right} & 0.980 & Position (0,5). Goal reached. \\
n27 & \texttt{up} & 0.008 & After 0 steps, I am at (0,4). Performing the next action would lead me outside the grid. \\
\end{tabularx}
\end{examplebox}
\begin{examplebox}[frametitle={Round 8},linecolor=black,backgroundcolor=casegreen!9,frametitlebackgroundcolor=casegreen!20,nobreak=true]
\small Selected branch \texttt{n26}, selection probability 0.994.
\par Submitted actions \texttt{right}, \texttt{down}, \texttt{right}, \texttt{right}, \texttt{right}, \texttt{up}, \texttt{right}.
\end{examplebox}
\clearpage
\subsection{Maze}
Only local wall observations are available. An initial right move is blocked. Subsequent selections follow an observed corridor to the goal without requiring a complete map.
\begin{examplebox}[frametitle={Task and outcome},linecolor=black]
Initial state 10. JevSpawn score 1.
\par\smallskip\centering
\small\begin{adjustbox}{max width=\linewidth}\begin{tabular}{rrrrrrr}
LATS & LLMCompiler & AgentPrune & HiAgent & FoldAgent & DyFlow & LatentMAS \\
0 & 0 & 0 & 0 & 0 & 0 & 0 \\
\end{tabular}\end{adjustbox}
\end{examplebox}
\begin{center}\captionsetup{hypcap=false}
\includegraphics[width=\linewidth]{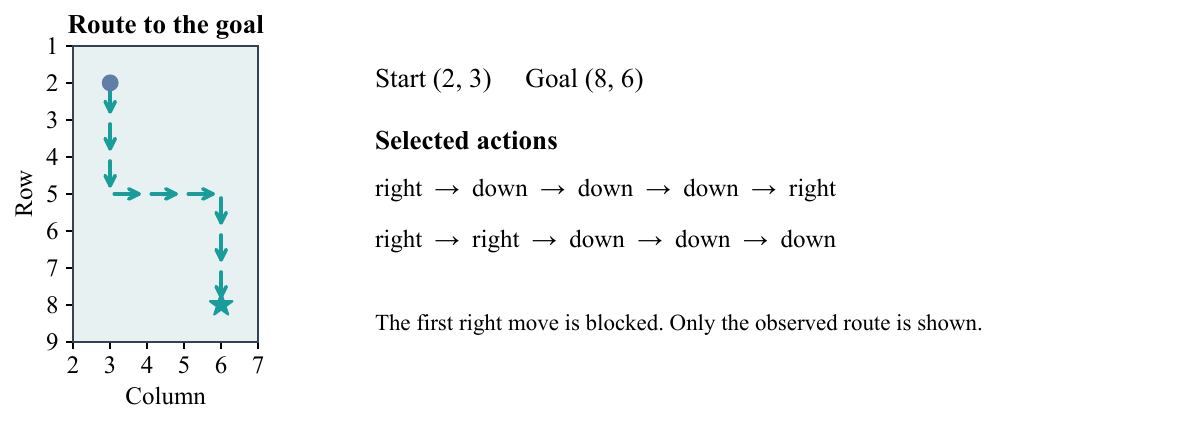}
\captionof{figure}{Maze navigation from local observations. The route contains one blocked right move followed by nine moves to the goal.}
\end{center}
\begin{examplebox}[frametitle={Round 1},linecolor=black,backgroundcolor=caseblue!9,frametitlebackgroundcolor=caseblue!20,nobreak=true]
\small Initial state.
\par\smallskip Fields \texttt{f0} (movement command) $\in$ \{\texttt{move left}, \texttt{move right}, \texttt{move up}, \texttt{move down}\}.
\par Action template \texttt{move(action = \$\{f0\})}.
\par\smallskip\noindent
\begin{tabularx}{\linewidth}{@{}l l r X@{}}
Branch & Action & $\pi_{\theta}$ & Returned observation \\ \hline
n0 & \texttt{move down} & 0.455 & Position 3, 3. Walls to your right, to your left. \\
n1 & \texttt{move left} & 0.031 & Position 2, 3. Walls to your right, to your left. \\
\textbf{n2} & \texttt{move right} & 0.234 & Position 2, 3. Walls to your right, to your left. \\
n3 & \texttt{move up} & 0.279 & Position 1, 3. There is a wall above you. \\
\end{tabularx}
\end{examplebox}
\begin{examplebox}[frametitle={Round 2},linecolor=black,backgroundcolor=casegreen!9,frametitlebackgroundcolor=casegreen!20,nobreak=true]
\small Selected branch \texttt{n2}, selection probability 0.349.
\par\smallskip\noindent
\begin{tabularx}{\linewidth}{@{}l l r X@{}}
Branch & Action & $\pi_{\theta}$ & Returned observation \\ \hline
\textbf{n4} & \texttt{move down} & 0.531 & Position 3, 3. Walls to your right, to your left. \\
n5 & \texttt{move left} & 0.008 & Position 2, 3. Walls to your right, to your left. \\
n6 & \texttt{move right} & 0.015 & Position 2, 3. Walls to your right, to your left. \\
n7 & \texttt{move up} & 0.445 & Position 1, 3. There is a wall above you. \\
\end{tabularx}
\end{examplebox}
\begin{examplebox}[frametitle={Round 3},linecolor=black,backgroundcolor=caseblue!9,frametitlebackgroundcolor=caseblue!20,nobreak=true]
\small Selected branch \texttt{n4}, selection probability 0.411.
\par\smallskip\noindent
\begin{tabularx}{\linewidth}{@{}l l r X@{}}
Branch & Action & $\pi_{\theta}$ & Returned observation \\ \hline
\textbf{n8} & \texttt{move down} & 0.385 & Position 4, 3. Walls to your right, to your left. \\
n9 & \texttt{move left} & 0.064 & Position 3, 3. Walls to your right, to your left. \\
n10 & \texttt{move right} & 0.318 & Position 3, 3. Walls to your right, to your left. \\
n11 & \texttt{move up} & 0.234 & Position 2, 3. Walls to your right, to your left. \\
\end{tabularx}
\end{examplebox}
\begin{examplebox}[frametitle={Round 4},linecolor=black,backgroundcolor=casegreen!9,frametitlebackgroundcolor=casegreen!20,nobreak=true]
\small Selected branch \texttt{n8}, selection probability 0.486.
\par\smallskip\noindent
\begin{tabularx}{\linewidth}{@{}l l r X@{}}
Branch & Action & $\pi_{\theta}$ & Returned observation \\ \hline
\textbf{n12} & \texttt{move down} & 0.848 & Position 5, 3. Walls to your left, below you. \\
n13 & \texttt{move left} & 0.016 & Position 4, 3. Walls to your right, to your left. \\
n14 & \texttt{move right} & 0.092 & Position 4, 3. Walls to your right, to your left. \\
n15 & \texttt{move up} & 0.044 & Position 3, 3. Walls to your right, to your left. \\
\end{tabularx}
\end{examplebox}
\begin{examplebox}[frametitle={Round 5},linecolor=black,backgroundcolor=caseblue!9,frametitlebackgroundcolor=caseblue!20,nobreak=true]
\small Selected branch \texttt{n12}, selection probability 0.310.
\par\smallskip\noindent
\begin{tabularx}{\linewidth}{@{}l l r X@{}}
Branch & Action & $\pi_{\theta}$ & Returned observation \\ \hline
n16 & \texttt{move down} & 0.109 & Position 5, 3. Walls to your left, below you. \\
n17 & \texttt{move left} & 0.011 & Position 5, 3. Walls to your left, below you. \\
\textbf{n18} & \texttt{move right} & 0.788 & Position 5, 4. Walls above you, below you. \\
n19 & \texttt{move up} & 0.092 & Position 4, 3. Walls to your right, to your left. \\
\end{tabularx}
\end{examplebox}
\begin{examplebox}[frametitle={Round 6},linecolor=black,backgroundcolor=casegreen!9,frametitlebackgroundcolor=casegreen!20,nobreak=true]
\small Selected branch \texttt{n18}, selection probability 0.422.
\par\smallskip\noindent
\begin{tabularx}{\linewidth}{@{}l l r X@{}}
Branch & Action & $\pi_{\theta}$ & Returned observation \\ \hline
n20 & \texttt{move down} & 0.026 & Position 5, 4. Walls above you, below you. \\
n21 & \texttt{move left} & 0.053 & Position 5, 3. Walls to your left, below you. \\
\textbf{n22} & \texttt{move right} & 0.915 & Position 5, 5. Walls above you, below you. \\
n23 & \texttt{move up} & 0.006 & Position 5, 4. Walls above you, below you. \\
\end{tabularx}
\end{examplebox}
\begin{examplebox}[frametitle={Round 7},linecolor=black,backgroundcolor=caseblue!9,frametitlebackgroundcolor=caseblue!20,nobreak=true]
\small Selected branch \texttt{n22}, selection probability 0.570.
\par\smallskip\noindent
\begin{tabularx}{\linewidth}{@{}l l r X@{}}
Branch & Action & $\pi_{\theta}$ & Returned observation \\ \hline
n24 & \texttt{move down} & 0.024 & Position 5, 5. Walls above you, below you. \\
n25 & \texttt{move left} & 0.030 & Position 5, 4. Walls above you, below you. \\
\textbf{n26} & \texttt{move right} & 0.938 & Position 5, 6. There is a wall above you. \\
n27 & \texttt{move up} & 0.007 & Position 5, 5. Walls above you, below you. \\
\end{tabularx}
\end{examplebox}
\begin{examplebox}[frametitle={Round 8},linecolor=black,backgroundcolor=casegreen!9,frametitlebackgroundcolor=casegreen!20,nobreak=true]
\small Selected branch \texttt{n26}, selection probability 0.767.
\par\smallskip\noindent
\begin{tabularx}{\linewidth}{@{}l l r X@{}}
Branch & Action & $\pi_{\theta}$ & Returned observation \\ \hline
\textbf{n28} & \texttt{move down} & 0.108 & Position 6, 6. Walls to your right, to your left. \\
n29 & \texttt{move left} & 0.020 & Position 5, 5. Walls above you, below you. \\
n30 & \texttt{move right} & 0.870 & Position 5, 7. Walls above you, below you. \\
n31 & \texttt{move up} & 0.003 & Position 5, 6. There is a wall above you. \\
\end{tabularx}
\end{examplebox}
\begin{examplebox}[frametitle={Round 9},linecolor=black,backgroundcolor=caseblue!9,frametitlebackgroundcolor=caseblue!20,nobreak=true]
\small Selected branch \texttt{n28}, selection probability 0.618.
\par\smallskip\noindent
\begin{tabularx}{\linewidth}{@{}l l r X@{}}
Branch & Action & $\pi_{\theta}$ & Returned observation \\ \hline
\textbf{n32} & \texttt{move down} & 0.454 & Position 7, 6. Walls to your right, to your left. \\
n33 & \texttt{move left} & 0.059 & Position 6, 6. Walls to your right, to your left. \\
n34 & \texttt{move right} & 0.185 & Position 6, 6. Walls to your right, to your left. \\
n35 & \texttt{move up} & 0.301 & Position 5, 6. There is a wall above you. \\
\end{tabularx}
\end{examplebox}
\begin{examplebox}[frametitle={Round 10},linecolor=black,backgroundcolor=casegreen!9,frametitlebackgroundcolor=casegreen!20,nobreak=true]
\small Selected branch \texttt{n32}, selection probability 0.405.
\par\smallskip\noindent
\begin{tabularx}{\linewidth}{@{}l l r X@{}}
Branch & Action & $\pi_{\theta}$ & Returned observation \\ \hline
\textbf{n36} & \texttt{move down} & 0.697 & Success
 \\
n37 & \texttt{move left} & 0.060 & Position 7, 6. Walls to your right, to your left. \\
n38 & \texttt{move right} & 0.176 & Position 7, 6. Walls to your right, to your left. \\
n39 & \texttt{move up} & 0.067 & Position 6, 6. Walls to your right, to your left. \\
\end{tabularx}
\end{examplebox}
\begin{examplebox}[frametitle={Round 11},linecolor=black,backgroundcolor=caseblue!9,frametitlebackgroundcolor=caseblue!20,nobreak=true]
\small Selected branch \texttt{n36}, selection probability 0.983.
\par Submitted actions \texttt{move right}, \texttt{move down}, \texttt{move down}, \texttt{move down}, \texttt{move right}, \texttt{move right}, \texttt{move right}, \texttt{move down}, \texttt{move down}, \texttt{move down}.
\end{examplebox}
\clearpage
\subsection{Grid}
Room descriptions reveal adjacent doors and visible objects. The six submitted actions include one blocked move before the treasure is reached. The action declaration remains unchanged throughout exploration. Direction indices 0, 1, 2, and 3 denote north, east, west, and south, respectively.
\begin{examplebox}[frametitle={Task and outcome},linecolor=black]
Environment seed 104. JevSpawn score 1.
\par\smallskip\centering
\small\begin{adjustbox}{max width=\linewidth}\begin{tabular}{rrrrrrr}
LATS & LLMCompiler & AgentPrune & HiAgent & FoldAgent & DyFlow & LatentMAS \\
0 & 0 & 0 & 0 & 0 & 0 & 0 \\
\end{tabular}\end{adjustbox}
\end{examplebox}
\begin{center}\captionsetup{hypcap=false}
\includegraphics[width=\linewidth]{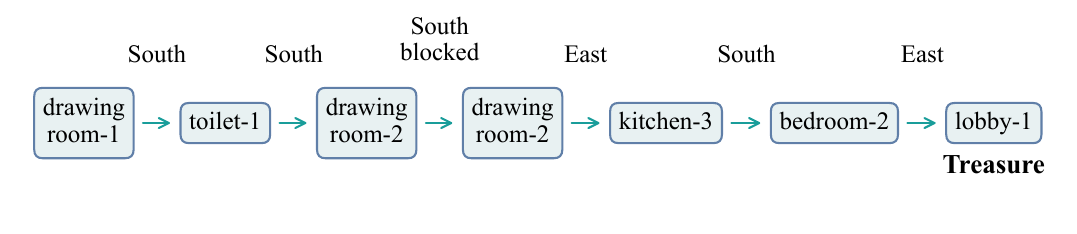}
\captionof{figure}{Grid navigation through observed rooms. Direction labels follow the selected actions, with the treasure found in the final room.}
\end{center}
\begin{examplebox}[frametitle={Round 1},linecolor=black,backgroundcolor=caseblue!9,frametitlebackgroundcolor=caseblue!20,nobreak=true]
\small Initial state.
\par\smallskip Fields \texttt{f0} (direction index) $\in$ \{\texttt{0}, \texttt{1}, \texttt{2}, \texttt{3}\}.
\par Action template \texttt{execute(action = \$\{f0\})}.
\par\smallskip\noindent
\begin{tabularx}{\linewidth}{@{}l l r X@{}}
Branch & Action & $\pi_{\theta}$ & Returned observation \\ \hline
n0 & \texttt{0} & 0.843 & kitchen-2. No treasure observed. \\
n1 & \texttt{1} & 0.023 & You remained in room drawing room-1 as there is no door in the direction east. \\
n2 & \texttt{2} & 0.000 & You remained in room drawing room-1 as there is no door in the direction west. \\
\textbf{n3} & \texttt{3} & 0.134 & toilet-1. No treasure observed. \\
\end{tabularx}
\end{examplebox}
\begin{examplebox}[frametitle={Round 2},linecolor=black,backgroundcolor=casegreen!9,frametitlebackgroundcolor=casegreen!20,nobreak=true]
\small Selected branch \texttt{n3}, selection probability 0.500.
\par\smallskip\noindent
\begin{tabularx}{\linewidth}{@{}l l r X@{}}
Branch & Action & $\pi_{\theta}$ & Returned observation \\ \hline
n4 & \texttt{0} & 0.147 & drawing room-1. No treasure observed. \\
n5 & \texttt{1} & 0.002 & You remained in room toilet-1 as there is no door in the direction east. \\
n6 & \texttt{2} & 0.623 & kitchen-1. No treasure observed. \\
\textbf{n7} & \texttt{3} & 0.227 & drawing room-2. No treasure observed. \\
\end{tabularx}
\end{examplebox}
\begin{examplebox}[frametitle={Round 3},linecolor=black,backgroundcolor=caseblue!9,frametitlebackgroundcolor=caseblue!20,nobreak=true]
\small Selected branch \texttt{n7}, selection probability 0.587.
\par\smallskip\noindent
\begin{tabularx}{\linewidth}{@{}l l r X@{}}
Branch & Action & $\pi_{\theta}$ & Returned observation \\ \hline
n8 & \texttt{0} & 0.033 & toilet-1. No treasure observed. \\
n9 & \texttt{1} & 0.956 & kitchen-3. No treasure observed. \\
n10 & \texttt{2} & 0.002 & You remained in room drawing room-2 as there is no door in the direction west. \\
\textbf{n11} & \texttt{3} & 0.009 & You remained in room drawing room-2 as there is no door in the direction south. \\
\end{tabularx}
\end{examplebox}
\begin{examplebox}[frametitle={Round 4},linecolor=black,backgroundcolor=casegreen!9,frametitlebackgroundcolor=casegreen!20,nobreak=true]
\small Selected branch \texttt{n11}, selection probability 0.529.
\par\smallskip\noindent
\begin{tabularx}{\linewidth}{@{}l l r X@{}}
Branch & Action & $\pi_{\theta}$ & Returned observation \\ \hline
n12 & \texttt{0} & 0.619 & toilet-1. No treasure observed. \\
\textbf{n13} & \texttt{1} & 0.322 & kitchen-3. No treasure observed. \\
n14 & \texttt{2} & 0.034 & You remained in room drawing room-2 as there is no door in the direction west. \\
n15 & \texttt{3} & 0.024 & You remained in room drawing room-2 as there is no door in the direction south. \\
\end{tabularx}
\end{examplebox}
\begin{examplebox}[frametitle={Round 5},linecolor=black,backgroundcolor=caseblue!9,frametitlebackgroundcolor=caseblue!20,nobreak=true]
\small Selected branch \texttt{n13}, selection probability 0.324.
\par\smallskip\noindent
\begin{tabularx}{\linewidth}{@{}l l r X@{}}
Branch & Action & $\pi_{\theta}$ & Returned observation \\ \hline
n16 & \texttt{0} & 0.009 & You remained in room kitchen-3 as there is no door in the direction north. \\
n17 & \texttt{1} & 0.650 & balcony-1. No treasure observed. \\
n18 & \texttt{2} & 0.013 & drawing room-2. No treasure observed. \\
\textbf{n19} & \texttt{3} & 0.328 & bedroom-2. No treasure observed. \\
\end{tabularx}
\end{examplebox}
\begin{examplebox}[frametitle={Round 6},linecolor=black,backgroundcolor=casegreen!9,frametitlebackgroundcolor=casegreen!20,nobreak=true]
\small Selected branch \texttt{n19}, selection probability 0.914.
\par\smallskip\noindent
\begin{tabularx}{\linewidth}{@{}l l r X@{}}
Branch & Action & $\pi_{\theta}$ & Returned observation \\ \hline
n20 & \texttt{0} & 0.022 & kitchen-3. No treasure observed. \\
\textbf{n21} & \texttt{1} & 0.955 & lobby-1. Treasure found. \\
n22 & \texttt{2} & 0.004 & You remained in room bedroom-2 as there is no door in the direction west. \\
n23 & \texttt{3} & 0.019 & You remained in room bedroom-2 as there is no door in the direction south. \\
\end{tabularx}
\end{examplebox}
\begin{examplebox}[frametitle={Round 7},linecolor=black,backgroundcolor=caseblue!9,frametitlebackgroundcolor=caseblue!20,nobreak=true]
\small Selected branch \texttt{n21}, selection probability 0.854.
\par Submitted actions \texttt{3}, \texttt{3}, \texttt{3}, \texttt{1}, \texttt{3}, \texttt{1}.
\end{examplebox}
\clearpage
\subsection{LightsOut}
Row and column fields define sixteen possible button presses. Three presses extinguish the board. Four actions are explored at each expansion, and subsequent states are selected from the returned boards.
\begin{examplebox}[frametitle={Task and outcome},linecolor=black]
Environment seed 124. JevSpawn score 1.
\par\smallskip\centering
\small\begin{adjustbox}{max width=\linewidth}\begin{tabular}{rrrrrrr}
LATS & LLMCompiler & AgentPrune & HiAgent & FoldAgent & DyFlow & LatentMAS \\
0 & 0 & 0 & 0 & 0 & 0 & 0 \\
\end{tabular}\end{adjustbox}
\end{examplebox}
\begin{center}\captionsetup{hypcap=false}
\includegraphics[width=\linewidth]{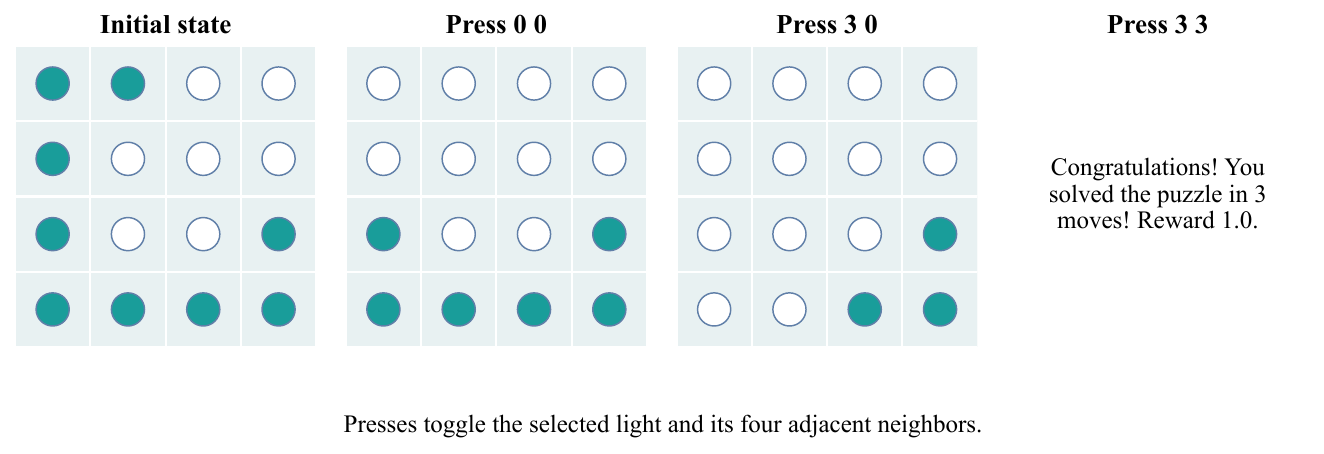}
\captionof{figure}{LightsOut solution in three presses. Filled circles denote lights remaining on the observed boards. The final panel shows the environment's success message.}
\end{center}
\begin{examplebox}[frametitle={Round 1},linecolor=black,backgroundcolor=caseblue!9,frametitlebackgroundcolor=caseblue!20,nobreak=true]
\small Initial state.
\par\smallskip Fields \texttt{f0} (row) $\in$ \{\texttt{0}, \texttt{1}, \texttt{2}, \texttt{3}\}, \texttt{f1} (column) $\in$ \{\texttt{0}, \texttt{1}, \texttt{2}, \texttt{3}\}.
\par Action template \texttt{execute(action = \$\{f0\} \$\{f1\})}.
\par\smallskip\noindent
\begin{tabularx}{\linewidth}{@{}l l r X@{}}
Branch & Action & $\pi_{\theta}$ & Returned observation \\ \hline
\textbf{n0} & \texttt{0 0} & 0.852 & 6 lights on. \\
n1 & \texttt{1 0} & 0.018 & 7 lights on. \\
n2 & \texttt{3 0} & 0.059 & 6 lights on. \\
n3 & \texttt{3 3} & 0.013 & 6 lights on. \\
\end{tabularx}
\end{examplebox}
\begin{examplebox}[frametitle={Round 2},linecolor=black,backgroundcolor=casegreen!9,frametitlebackgroundcolor=casegreen!20,nobreak=true]
\small Selected branch \texttt{n0}, selection probability 0.803.
\par\smallskip\noindent
\begin{tabularx}{\linewidth}{@{}l l r X@{}}
Branch & Action & $\pi_{\theta}$ & Returned observation \\ \hline
n4 & \texttt{2 0} & 0.672 & 6 lights on. \\
n5 & \texttt{2 3} & 0.046 & 6 lights on. \\
\textbf{n6} & \texttt{3 0} & 0.138 & 3 lights on. \\
n7 & \texttt{3 1} & 0.058 & 4 lights on. \\
\end{tabularx}
\end{examplebox}
\begin{examplebox}[frametitle={Round 3},linecolor=black,backgroundcolor=caseblue!9,frametitlebackgroundcolor=caseblue!20,nobreak=true]
\small Selected branch \texttt{n6}, selection probability 0.891.
\par\smallskip\noindent
\begin{tabularx}{\linewidth}{@{}l l r X@{}}
Branch & Action & $\pi_{\theta}$ & Returned observation \\ \hline
n8 & \texttt{2 3} & 0.546 & 3 lights on. \\
n9 & \texttt{3 1} & 0.033 & 5 lights on. \\
n10 & \texttt{3 2} & 0.332 & 3 lights on. \\
\textbf{n11} & \texttt{3 3} & 0.063 & Congratulations! You solved the puzzle in 3 moves! Reward 1.0. \\
\end{tabularx}
\end{examplebox}
\begin{examplebox}[frametitle={Round 4},linecolor=black,backgroundcolor=casegreen!9,frametitlebackgroundcolor=casegreen!20,nobreak=true]
\small Selected branch \texttt{n11}, selection probability 0.995.
\par Submitted actions \texttt{0 0}, \texttt{3 0}, \texttt{3 3}.
\end{examplebox}
\clearpage
\subsection{RushHour}
Vehicle and direction fields express the available sliding actions. Exploration returns to retained alternatives before selecting a path of nine actions that releases the red car.
\begin{examplebox}[frametitle={Task and outcome},linecolor=black]
Environment seed 138. JevSpawn score 1.
\par\smallskip\centering
\small\begin{adjustbox}{max width=\linewidth}\begin{tabular}{rrrrrrr}
LATS & LLMCompiler & AgentPrune & HiAgent & FoldAgent & DyFlow & LatentMAS \\
0 & 0 & 0.5 & 0.5 & 0.5 & 0 & 0.5 \\
\end{tabular}\end{adjustbox}
\end{examplebox}
\begin{center}\captionsetup{hypcap=false}
\includegraphics[width=\linewidth]{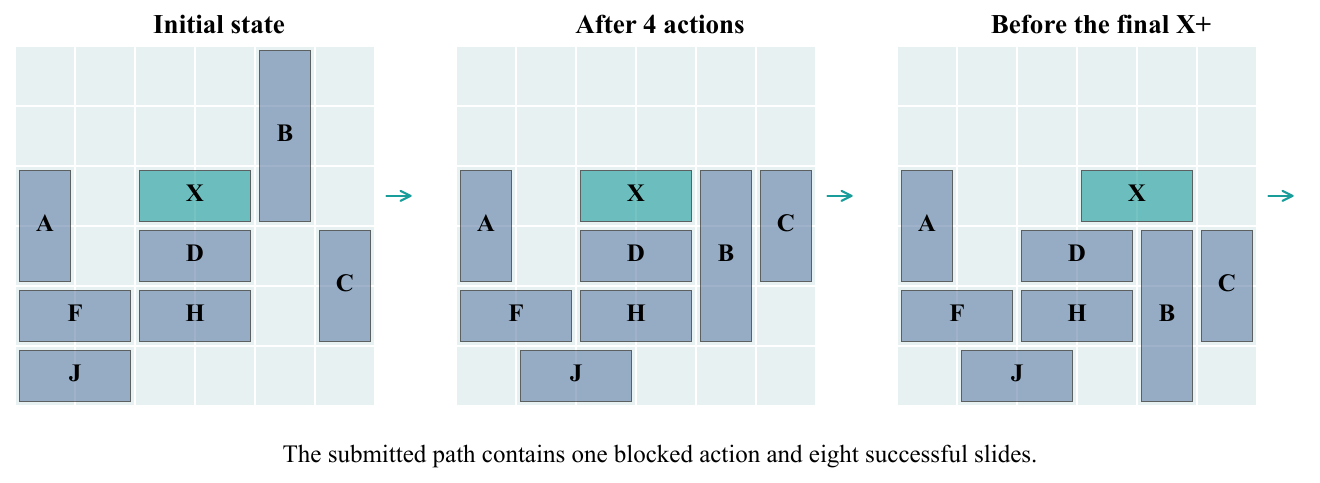}
\captionof{figure}{RushHour states along the submitted branch. The middle panel follows J+, B+, C-, B+. The car marked X exits after the final forward move.}
\end{center}
\begin{examplebox}[frametitle={Round 1},linecolor=black,backgroundcolor=caseblue!9,frametitlebackgroundcolor=caseblue!20,nobreak=true]
\small Initial state.
\par\smallskip Fields \texttt{f0} (vehicle) $\in$ \{\texttt{A}, \texttt{B}, \texttt{C}, \texttt{D}, \texttt{F}, \texttt{H}, \texttt{J}, \texttt{X}\}, \texttt{f1} (direction) $\in$ \{\texttt{+}, \texttt{-}\}.
\par Action template \texttt{execute(action = \$\{f0\}\$\{f1\})}.
\par\smallskip\noindent
\begin{tabularx}{\linewidth}{@{}l l r X@{}}
Branch & Action & $\pi_{\theta}$ & Returned observation \\ \hline
n0 & \texttt{B+} & 0.302 & Red car [(2, 2), (2, 3)]. Blocking car B [(1, 4), (2, 4), (3, 4)]. \\
n1 & \texttt{B-} & 0.239 & Invalid move. Move blocked. \\
n2 & \texttt{D+} & 0.088 & Red car [(2, 2), (2, 3)]. Blocking car B [(0, 4), (1, 4), (2, 4)]. \\
\textbf{n3} & \texttt{J+} & 0.134 & Red car [(2, 2), (2, 3)]. Blocking car B [(0, 4), (1, 4), (2, 4)]. \\
\end{tabularx}
\end{examplebox}
\begin{examplebox}[frametitle={Round 2},linecolor=black,backgroundcolor=casegreen!9,frametitlebackgroundcolor=casegreen!20,nobreak=true]
\small Selected branch \texttt{n3}, selection probability 0.666.
\par\smallskip\noindent
\begin{tabularx}{\linewidth}{@{}l l r X@{}}
Branch & Action & $\pi_{\theta}$ & Returned observation \\ \hline
\textbf{n4} & \texttt{B+} & 0.138 & Red car [(2, 2), (2, 3)]. Blocking car B [(1, 4), (2, 4), (3, 4)]. \\
n5 & \texttt{C-} & 0.122 & Red car [(2, 2), (2, 3)]. Blocking car B [(0, 4), (1, 4), (2, 4)]. \\
n6 & \texttt{D-} & 0.122 & Red car [(2, 2), (2, 3)]. Blocking car B [(0, 4), (1, 4), (2, 4)]. \\
n7 & \texttt{H+} & 0.106 & Red car [(2, 2), (2, 3)]. Blocking car B [(0, 4), (1, 4), (2, 4)]. \\
\end{tabularx}
\end{examplebox}
\begin{examplebox}[frametitle={Round 3},linecolor=black,backgroundcolor=caseblue!9,frametitlebackgroundcolor=caseblue!20,nobreak=true]
\small Selected branch \texttt{n4}, selection probability 0.488.
\par\smallskip\noindent
\begin{tabularx}{\linewidth}{@{}l l r X@{}}
Branch & Action & $\pi_{\theta}$ & Returned observation \\ \hline
n8 & \texttt{B-} & 0.116 & Red car [(2, 2), (2, 3)]. Blocking car B [(0, 4), (1, 4), (2, 4)]. \\
\textbf{n9} & \texttt{C-} & 0.226 & Red car [(2, 2), (2, 3)]. Blocking car B [(1, 4), (2, 4), (3, 4)]. \\
n10 & \texttt{D-} & 0.120 & Red car [(2, 2), (2, 3)]. Blocking car B [(1, 4), (2, 4), (3, 4)]. \\
n11 & \texttt{X+} & 0.150 & Invalid move. Move blocked. \\
\end{tabularx}
\end{examplebox}
\begin{examplebox}[frametitle={Round 4 (recovery)},linecolor=black,backgroundcolor=casegreen!9,frametitlebackgroundcolor=casegreen!20,nobreak=true]
\small Selected branch \texttt{n1}, selection probability 0.693.
\par\smallskip\noindent
\begin{tabularx}{\linewidth}{@{}l l r X@{}}
Branch & Action & $\pi_{\theta}$ & Returned observation \\ \hline
n12 & \texttt{A-} & 0.060 & Red car [(2, 2), (2, 3)]. Blocking car B [(0, 4), (1, 4), (2, 4)]. \\
n13 & \texttt{B+} & 0.101 & Red car [(2, 2), (2, 3)]. Blocking car B [(1, 4), (2, 4), (3, 4)]. \\
n14 & \texttt{D+} & 0.533 & Red car [(2, 2), (2, 3)]. Blocking car B [(0, 4), (1, 4), (2, 4)]. \\
n15 & \texttt{D-} & 0.076 & Red car [(2, 2), (2, 3)]. Blocking car B [(0, 4), (1, 4), (2, 4)]. \\
\end{tabularx}
\end{examplebox}
\begin{examplebox}[frametitle={Round 5 (recovery)},linecolor=black,backgroundcolor=caseblue!9,frametitlebackgroundcolor=caseblue!20,nobreak=true]
\small Selected branch \texttt{n5}, selection probability 0.435.
\par\smallskip\noindent
\begin{tabularx}{\linewidth}{@{}l l r X@{}}
Branch & Action & $\pi_{\theta}$ & Returned observation \\ \hline
n16 & \texttt{B+} & 0.142 & Red car [(2, 2), (2, 3)]. Blocking car B [(1, 4), (2, 4), (3, 4)]. \\
n17 & \texttt{B-} & 0.264 & Invalid move. Move blocked. \\
n18 & \texttt{C-} & 0.083 & Red car [(2, 2), (2, 3)]. Blocking car B [(0, 4), (1, 4), (2, 4)]. \\
n19 & \texttt{D-} & 0.197 & Red car [(2, 2), (2, 3)]. Blocking car B [(0, 4), (1, 4), (2, 4)]. \\
\end{tabularx}
\end{examplebox}
\begin{examplebox}[frametitle={Round 6 (recovery)},linecolor=black,backgroundcolor=casegreen!9,frametitlebackgroundcolor=casegreen!20,nobreak=true]
\small Selected branch \texttt{n9}, selection probability 0.395.
\par\smallskip\noindent
\begin{tabularx}{\linewidth}{@{}l l r X@{}}
Branch & Action & $\pi_{\theta}$ & Returned observation \\ \hline
\textbf{n20} & \texttt{B+} & 0.205 & Red car [(2, 2), (2, 3)]. Blocking car B [(2, 4), (3, 4), (4, 4)]. \\
n21 & \texttt{B-} & 0.328 & Red car [(2, 2), (2, 3)]. Blocking car B [(0, 4), (1, 4), (2, 4)]. \\
n22 & \texttt{D-} & 0.214 & Red car [(2, 2), (2, 3)]. Blocking car B [(1, 4), (2, 4), (3, 4)]. \\
n23 & \texttt{X+} & 0.095 & Invalid move. Move blocked. \\
\end{tabularx}
\end{examplebox}
\begin{examplebox}[frametitle={Round 7 (recovery)},linecolor=black,backgroundcolor=caseblue!9,frametitlebackgroundcolor=caseblue!20,nobreak=true]
\small Selected branch \texttt{n18}, selection probability 0.403.
\par\smallskip\noindent
\begin{tabularx}{\linewidth}{@{}l l r X@{}}
Branch & Action & $\pi_{\theta}$ & Returned observation \\ \hline
n24 & \texttt{B+} & 0.436 & Red car [(2, 2), (2, 3)]. Blocking car B [(1, 4), (2, 4), (3, 4)]. \\
n25 & \texttt{B-} & 0.214 & Invalid move. Move blocked. \\
n26 & \texttt{D+} & 0.058 & Red car [(2, 2), (2, 3)]. Blocking car B [(0, 4), (1, 4), (2, 4)]. \\
n27 & \texttt{D-} & 0.197 & Red car [(2, 2), (2, 3)]. Blocking car B [(0, 4), (1, 4), (2, 4)]. \\
\end{tabularx}
\end{examplebox}
\begin{examplebox}[frametitle={Round 8 (recovery)},linecolor=black,backgroundcolor=casegreen!9,frametitlebackgroundcolor=casegreen!20,nobreak=true]
\small Selected branch \texttt{n20}, selection probability 0.445.
\par\smallskip\noindent
\begin{tabularx}{\linewidth}{@{}l l r X@{}}
Branch & Action & $\pi_{\theta}$ & Returned observation \\ \hline
\textbf{n28} & \texttt{B+} & 0.137 & Red car [(2, 2), (2, 3)]. Blocking car B [(3, 4), (4, 4), (5, 4)]. \\
n29 & \texttt{B-} & 0.151 & Red car [(2, 2), (2, 3)]. Blocking car B [(1, 4), (2, 4), (3, 4)]. \\
n30 & \texttt{C-} & 0.280 & Red car [(2, 2), (2, 3)]. Blocking car B [(2, 4), (3, 4), (4, 4)]. \\
n31 & \texttt{D-} & 0.141 & Red car [(2, 2), (2, 3)]. Blocking car B [(2, 4), (3, 4), (4, 4)]. \\
\end{tabularx}
\end{examplebox}
\begin{examplebox}[frametitle={Round 9 (recovery)},linecolor=black,backgroundcolor=caseblue!9,frametitlebackgroundcolor=caseblue!20,nobreak=true]
\small Selected branch \texttt{n24}, selection probability 0.550.
\par\smallskip\noindent
\begin{tabularx}{\linewidth}{@{}l l r X@{}}
Branch & Action & $\pi_{\theta}$ & Returned observation \\ \hline
n32 & \texttt{B+} & 0.116 & Red car [(2, 2), (2, 3)]. Blocking car B [(2, 4), (3, 4), (4, 4)]. \\
n33 & \texttt{C-} & 0.440 & Red car [(2, 2), (2, 3)]. Blocking car B [(1, 4), (2, 4), (3, 4)]. \\
n34 & \texttt{D+} & 0.097 & Invalid move. Move blocked. \\
n35 & \texttt{D-} & 0.159 & Red car [(2, 2), (2, 3)]. Blocking car B [(1, 4), (2, 4), (3, 4)]. \\
\end{tabularx}
\end{examplebox}
\begin{examplebox}[frametitle={Round 10 (recovery)},linecolor=black,backgroundcolor=casegreen!9,frametitlebackgroundcolor=casegreen!20,nobreak=true]
\small Selected branch \texttt{n28}, selection probability 0.770.
\par\smallskip\noindent
\begin{tabularx}{\linewidth}{@{}l l r X@{}}
Branch & Action & $\pi_{\theta}$ & Returned observation \\ \hline
n36 & \texttt{B+} & 0.052 & Invalid move. Move blocked. \\
n37 & \texttt{B-} & 0.049 & Red car [(2, 2), (2, 3)]. Blocking car B [(2, 4), (3, 4), (4, 4)]. \\
n38 & \texttt{C-} & 0.429 & Red car [(2, 2), (2, 3)]. Blocking car B [(3, 4), (4, 4), (5, 4)]. \\
\textbf{n39} & \texttt{X+} & 0.320 & Red car [(2, 3), (2, 4)]. Blocking car B [(3, 4), (4, 4), (5, 4)]. \\
\end{tabularx}
\end{examplebox}
\begin{examplebox}[frametitle={Round 11},linecolor=black,backgroundcolor=caseblue!9,frametitlebackgroundcolor=caseblue!20,nobreak=true]
\small Selected branch \texttt{n39}, selection probability 0.902.
\par\smallskip\noindent
\begin{tabularx}{\linewidth}{@{}l l r X@{}}
Branch & Action & $\pi_{\theta}$ & Returned observation \\ \hline
\textbf{n40} & \texttt{B-} & 0.054 & Invalid move. Move blocked. \\
n41 & \texttt{C-} & 0.170 & Red car [(2, 3), (2, 4)]. Blocking car B [(3, 4), (4, 4), (5, 4)]. \\
n42 & \texttt{D-} & 0.016 & Red car [(2, 3), (2, 4)]. Blocking car B [(3, 4), (4, 4), (5, 4)]. \\
n43 & \texttt{X+} & 0.697 & Invalid move. Move blocked. \\
\end{tabularx}
\end{examplebox}
\begin{examplebox}[frametitle={Round 12},linecolor=black,backgroundcolor=casegreen!9,frametitlebackgroundcolor=casegreen!20,nobreak=true]
\small Selected branch \texttt{n40}, selection probability 0.144.
\par\smallskip\noindent
\begin{tabularx}{\linewidth}{@{}l l r X@{}}
Branch & Action & $\pi_{\theta}$ & Returned observation \\ \hline
\textbf{n44} & \texttt{C+} & 0.037 & Red car [(2, 3), (2, 4)]. Blocking car B [(3, 4), (4, 4), (5, 4)]. \\
n45 & \texttt{C-} & 0.543 & Red car [(2, 3), (2, 4)]. Blocking car B [(3, 4), (4, 4), (5, 4)]. \\
n46 & \texttt{D-} & 0.027 & Red car [(2, 3), (2, 4)]. Blocking car B [(3, 4), (4, 4), (5, 4)]. \\
n47 & \texttt{X+} & 0.300 & Move blocked. Reward 0.75. \\
\end{tabularx}
\end{examplebox}
\begin{examplebox}[frametitle={Round 13},linecolor=black,backgroundcolor=caseblue!9,frametitlebackgroundcolor=caseblue!20,nobreak=true]
\small Selected branch \texttt{n44}, selection probability 0.211.
\par\smallskip\noindent
\begin{tabularx}{\linewidth}{@{}l l r X@{}}
Branch & Action & $\pi_{\theta}$ & Returned observation \\ \hline
n48 & \texttt{B-} & 0.053 & Invalid move. Move blocked. \\
n49 & \texttt{C-} & 0.038 & Red car [(2, 3), (2, 4)]. Blocking car B [(3, 4), (4, 4), (5, 4)]. \\
n50 & \texttt{D-} & 0.029 & Red car [(2, 3), (2, 4)]. Blocking car B [(3, 4), (4, 4), (5, 4)]. \\
\textbf{n51} & \texttt{X+} & 0.782 & Puzzle solved! Reward 1.0. \\
\end{tabularx}
\end{examplebox}
\begin{examplebox}[frametitle={Round 14},linecolor=black,backgroundcolor=casegreen!9,frametitlebackgroundcolor=casegreen!20,nobreak=true]
\small Selected branch \texttt{n51}, selection probability 0.990.
\par Submitted actions \texttt{J+}, \texttt{B+}, \texttt{C-}, \texttt{B+}, \texttt{B+}, \texttt{X+}, \texttt{B-}, \texttt{C+}, \texttt{X+}.
\end{examplebox}
\clearpage
\subsection{Sokoban}
A declaration over four movement commands supports both navigation and box pushing. Exploration revisits earlier branches before selecting a solution with four moves that places both boxes on their goals.
\begin{examplebox}[frametitle={Task and outcome},linecolor=black]
Environment seed 88. JevSpawn score 1.
\par\smallskip\centering
\small\begin{adjustbox}{max width=\linewidth}\begin{tabular}{rrrrrrr}
LATS & LLMCompiler & AgentPrune & HiAgent & FoldAgent & DyFlow & LatentMAS \\
0 & 0 & 1 & 0 & 0 & 0 & 0 \\
\end{tabular}\end{adjustbox}
\end{examplebox}
\begin{center}\captionsetup{hypcap=false}
\includegraphics[width=\linewidth]{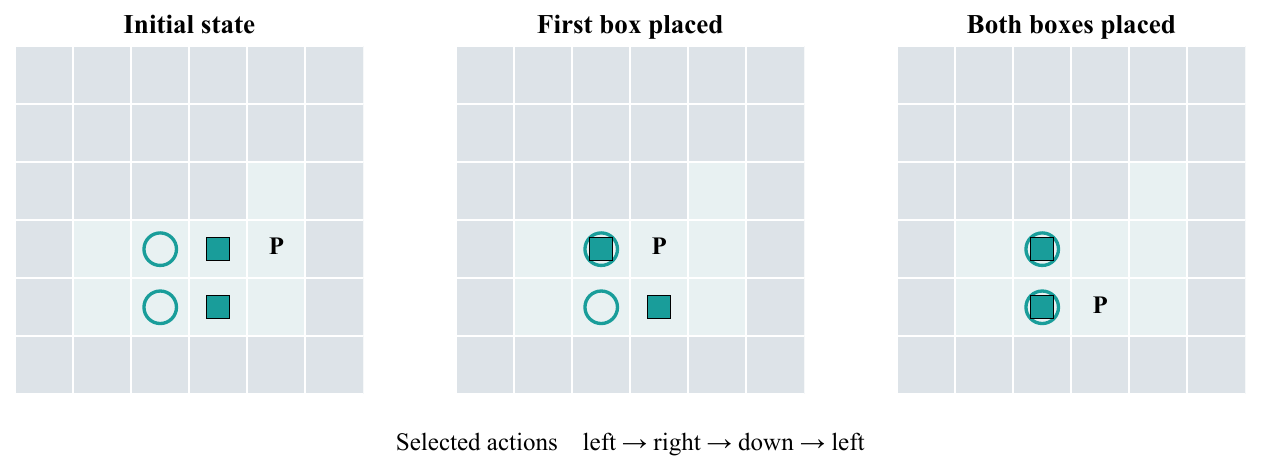}
\captionof{figure}{Sokoban solution in four moves. Squares denote boxes, rings denote goals, and P marks the player.}
\end{center}
\begin{examplebox}[frametitle={Round 1},linecolor=black,backgroundcolor=caseblue!9,frametitlebackgroundcolor=caseblue!20,nobreak=true]
\small Initial state.
\par\smallskip Fields \texttt{f0} (movement command) $\in$ \{\texttt{[up]}, \texttt{[down]}, \texttt{[left]}, \texttt{[right]}\}.
\par Action template \texttt{execute(action = \$\{f0\})}.
\par\smallskip\noindent
\begin{tabularx}{\linewidth}{@{}l l r X@{}}
Branch & Action & $\pi_{\theta}$ & Returned observation \\ \hline
n0 & \texttt{[down]} & 0.082 & Boxes on goals 0. \\
\textbf{n1} & \texttt{[left]} & 0.870 & Boxes on goals 1. \\
n2 & \texttt{[right]} & 0.011 & Invalid move. You cannot move into a wall! \\
n3 & \texttt{[up]} & 0.037 & Boxes on goals 0. \\
\end{tabularx}
\end{examplebox}
\begin{examplebox}[frametitle={Round 2},linecolor=black,backgroundcolor=casegreen!9,frametitlebackgroundcolor=casegreen!20,nobreak=true]
\small Selected branch \texttt{n1}, selection probability 0.921.
\par\smallskip\noindent
\begin{tabularx}{\linewidth}{@{}l l r X@{}}
Branch & Action & $\pi_{\theta}$ & Returned observation \\ \hline
n4 & \texttt{[down]} & 0.820 & Invalid move. You cannot move into a wall! \\
n5 & \texttt{[left]} & 0.085 & Boxes on goals 0. \\
\textbf{n6} & \texttt{[right]} & 0.019 & Boxes on goals 1. \\
n7 & \texttt{[up]} & 0.076 & Invalid move. You cannot move into a wall! \\
\end{tabularx}
\end{examplebox}
\begin{examplebox}[frametitle={Round 3},linecolor=black,backgroundcolor=caseblue!9,frametitlebackgroundcolor=caseblue!20,nobreak=true]
\small Selected branch \texttt{n6}, selection probability 0.830.
\par\smallskip\noindent
\begin{tabularx}{\linewidth}{@{}l l r X@{}}
Branch & Action & $\pi_{\theta}$ & Returned observation \\ \hline
\textbf{n8} & \texttt{[down]} & 0.614 & Boxes on goals 1. \\
n9 & \texttt{[left]} & 0.263 & Boxes on goals 1. \\
n10 & \texttt{[right]} & 0.009 & Invalid move. You cannot move into a wall! \\
n11 & \texttt{[up]} & 0.114 & Boxes on goals 1. \\
\end{tabularx}
\end{examplebox}
\begin{examplebox}[frametitle={Round 4},linecolor=black,backgroundcolor=casegreen!9,frametitlebackgroundcolor=casegreen!20,nobreak=true]
\small Selected branch \texttt{n11}, selection probability 0.460.
\par\smallskip\noindent
\begin{tabularx}{\linewidth}{@{}l l r X@{}}
Branch & Action & $\pi_{\theta}$ & Returned observation \\ \hline
n12 & \texttt{[down]} & 0.161 & Boxes on goals 1. \\
n13 & \texttt{[left]} & 0.818 & Invalid move. You cannot move into a wall! \\
n14 & \texttt{[right]} & 0.009 & Invalid move. You cannot move into a wall! \\
n15 & \texttt{[up]} & 0.012 & Invalid move. You cannot move into a wall! \\
\end{tabularx}
\end{examplebox}
\begin{examplebox}[frametitle={Round 5 (recovery)},linecolor=black,backgroundcolor=caseblue!9,frametitlebackgroundcolor=caseblue!20,nobreak=true]
\small Selected branch \texttt{n0}, selection probability 0.337.
\par\smallskip\noindent
\begin{tabularx}{\linewidth}{@{}l l r X@{}}
Branch & Action & $\pi_{\theta}$ & Returned observation \\ \hline
n16 & \texttt{[down]} & 0.041 & Invalid move. You cannot move into a wall! \\
n17 & \texttt{[left]} & 0.875 & Boxes on goals 1. \\
n18 & \texttt{[right]} & 0.014 & Invalid move. You cannot move into a wall! \\
n19 & \texttt{[up]} & 0.070 & Boxes on goals 0. \\
\end{tabularx}
\end{examplebox}
\begin{examplebox}[frametitle={Round 6 (recovery)},linecolor=black,backgroundcolor=casegreen!9,frametitlebackgroundcolor=casegreen!20,nobreak=true]
\small Selected branch \texttt{n8}, selection probability 0.458.
\par\smallskip\noindent
\begin{tabularx}{\linewidth}{@{}l l r X@{}}
Branch & Action & $\pi_{\theta}$ & Returned observation \\ \hline
n20 & \texttt{[down]} & 0.031 & Invalid move. You cannot move into a wall! \\
\textbf{n21} & \texttt{[left]} & 0.896 & Congratulations! You have solved the Sokoban puzzle! Reward 1. \\
n22 & \texttt{[right]} & 0.019 & Invalid move. You cannot move into a wall! \\
n23 & \texttt{[up]} & 0.055 & Boxes on goals 1. \\
\end{tabularx}
\end{examplebox}
\begin{examplebox}[frametitle={Round 7},linecolor=black,backgroundcolor=caseblue!9,frametitlebackgroundcolor=caseblue!20,nobreak=true]
\small Selected branch \texttt{n21}, selection probability 0.991.
\par Submitted actions \texttt{[left]}, \texttt{[right]}, \texttt{[down]}, \texttt{[left]}.
\end{examplebox}
\clearpage
\subsection{2048}
The selected trajectory contains 31 moves and is submitted at the round limit with a merge score of 1076. KORGym draws new tiles from powers of two up to half the largest tile, or inserts a two when the largest tile is below four. Newly inserted tiles do not contribute to the merge score.
\begin{examplebox}[frametitle={Task and outcome},linecolor=black]
Environment seed 141. JevSpawn score 1076.
\par\smallskip\centering
\small\begin{adjustbox}{max width=\linewidth}\begin{tabular}{rrrrrrr}
LATS & LLMCompiler & AgentPrune & HiAgent & FoldAgent & DyFlow & LatentMAS \\
0 & 0 & 0 & 0 & 0 & 0 & 0 \\
\end{tabular}\end{adjustbox}
\end{examplebox}
\begin{center}\captionsetup{hypcap=false}
\includegraphics[width=\linewidth]{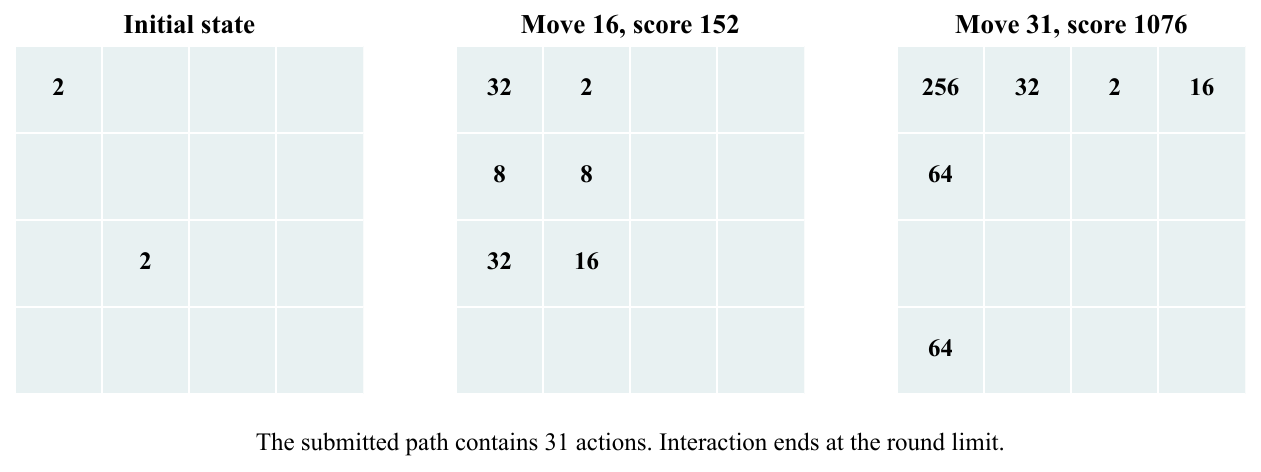}
\captionof{figure}{Selected 2048 states along the submitted trajectory. The final score is 1076 and the largest tile is 256.}
\end{center}
\begin{examplebox}[frametitle={Round 1},linecolor=black,backgroundcolor=caseblue!9,frametitlebackgroundcolor=caseblue!20,nobreak=true]
\small Initial state.
\par\smallskip Fields \texttt{f0} (movement direction) $\in$ \{\texttt{LEFT}, \texttt{RIGHT}, \texttt{UP}, \texttt{DOWN}\}.
\par Action template \texttt{execute(action = \$\{f0\})}.
\par\smallskip\noindent
\begin{tabularx}{\linewidth}{@{}l l r X@{}}
Branch & Action & $\pi_{\theta}$ & Returned observation \\ \hline
n0 & \texttt{DOWN} & 0.186 & Score 0. Largest tile 2. \\
n1 & \texttt{LEFT} & 0.605 & Score 0. Largest tile 2. \\
n2 & \texttt{RIGHT} & 0.115 & Score 0. Largest tile 2. \\
\textbf{n3} & \texttt{UP} & 0.094 & Score 0. Largest tile 2. \\
\end{tabularx}
\end{examplebox}
\begin{examplebox}[frametitle={Round 2},linecolor=black,backgroundcolor=casegreen!9,frametitlebackgroundcolor=casegreen!20,nobreak=true]
\small Selected branch \texttt{n3}, selection probability 0.770.
\par\smallskip\noindent
\begin{tabularx}{\linewidth}{@{}l l r X@{}}
Branch & Action & $\pi_{\theta}$ & Returned observation \\ \hline
n4 & \texttt{DOWN} & 0.021 & Score 0. Largest tile 2. \\
n5 & \texttt{LEFT} & 0.925 & Score 4. Largest tile 4. \\
\textbf{n6} & \texttt{RIGHT} & 0.028 & Score 4. Largest tile 4. \\
n7 & \texttt{UP} & 0.026 & Score 0. Largest tile 2. \\
\end{tabularx}
\end{examplebox}
\begin{examplebox}[frametitle={Round 3},linecolor=black,backgroundcolor=caseblue!9,frametitlebackgroundcolor=caseblue!20,nobreak=true]
\small Selected branch \texttt{n6}, selection probability 0.573.
\par\smallskip\noindent
\begin{tabularx}{\linewidth}{@{}l l r X@{}}
Branch & Action & $\pi_{\theta}$ & Returned observation \\ \hline
n8 & \texttt{DOWN} & 0.147 & Score 4. Largest tile 4. \\
\textbf{n9} & \texttt{LEFT} & 0.658 & Score 4. Largest tile 4. \\
n10 & \texttt{RIGHT} & 0.059 & Score 4. Largest tile 4. \\
n11 & \texttt{UP} & 0.136 & Score 4. Largest tile 4. \\
\end{tabularx}
\end{examplebox}
\begin{examplebox}[frametitle={Round 4 (recovery)},linecolor=black,backgroundcolor=casegreen!9,frametitlebackgroundcolor=casegreen!20,nobreak=true]
\small Selected branch \texttt{n0}, selection probability 0.513.
\par\smallskip\noindent
\begin{tabularx}{\linewidth}{@{}l l r X@{}}
Branch & Action & $\pi_{\theta}$ & Returned observation \\ \hline
n12 & \texttt{DOWN} & 0.014 & Score 0. Largest tile 2. \\
n13 & \texttt{LEFT} & 0.961 & Score 4. Largest tile 4. \\
n14 & \texttt{RIGHT} & 0.015 & Score 4. Largest tile 4. \\
n15 & \texttt{UP} & 0.011 & Score 0. Largest tile 2. \\
\end{tabularx}
\end{examplebox}
\begin{examplebox}[frametitle={Round 5 (recovery)},linecolor=black,backgroundcolor=caseblue!9,frametitlebackgroundcolor=caseblue!20,nobreak=true]
\small Selected branch \texttt{n9}, selection probability 0.221.
\par\smallskip\noindent
\begin{tabularx}{\linewidth}{@{}l l r X@{}}
Branch & Action & $\pi_{\theta}$ & Returned observation \\ \hline
n16 & \texttt{DOWN} & 0.218 & Score 8. Largest tile 4. \\
n17 & \texttt{LEFT} & 0.313 & Score 8. Largest tile 4. \\
n18 & \texttt{RIGHT} & 0.025 & Score 8. Largest tile 4. \\
\textbf{n19} & \texttt{UP} & 0.444 & Score 8. Largest tile 4. \\
\end{tabularx}
\end{examplebox}
\begin{examplebox}[frametitle={Intermediate rounds},linecolor=black,backgroundcolor=gray!6]
Rounds 6 through 33 are condensed. The action declaration is retained and 112 actions are executed across these rounds. Several earlier branches are revisited before the final high scoring path is submitted.
\end{examplebox}
\begin{examplebox}[frametitle={Round 34},linecolor=black,backgroundcolor=casegreen!9,frametitlebackgroundcolor=casegreen!20,nobreak=true]
\small Selected branch \texttt{n129}, selection probability 0.364.
\par\smallskip\noindent
\begin{tabularx}{\linewidth}{@{}l l r X@{}}
Branch & Action & $\pi_{\theta}$ & Returned observation \\ \hline
n132 & \texttt{DOWN} & 0.146 & Score 996. Largest tile 256. \\
n133 & \texttt{LEFT} & 0.540 & Score 708. Largest tile 128. \\
n134 & \texttt{RIGHT} & 0.008 & Score 708. Largest tile 128. \\
\textbf{n135} & \texttt{UP} & 0.306 & Score 996. Largest tile 256. \\
\end{tabularx}
\end{examplebox}
\begin{examplebox}[frametitle={Round 35},linecolor=black,backgroundcolor=caseblue!9,frametitlebackgroundcolor=caseblue!20,nobreak=true]
\small Selected branch \texttt{n135}, selection probability 0.923.
\par\smallskip\noindent
\begin{tabularx}{\linewidth}{@{}l l r X@{}}
Branch & Action & $\pi_{\theta}$ & Returned observation \\ \hline
n136 & \texttt{DOWN} & 0.144 & Score 1012. Largest tile 256. \\
n137 & \texttt{LEFT} & 0.799 & Score 996. Largest tile 256. \\
n138 & \texttt{RIGHT} & 0.007 & Score 996. Largest tile 256. \\
\textbf{n139} & \texttt{UP} & 0.051 & Score 1012. Largest tile 256. \\
\end{tabularx}
\end{examplebox}
\begin{examplebox}[frametitle={Round 36},linecolor=black,backgroundcolor=casegreen!9,frametitlebackgroundcolor=casegreen!20,nobreak=true]
\small Selected branch \texttt{n139}, selection probability 0.909.
\par\smallskip\noindent
\begin{tabularx}{\linewidth}{@{}l l r X@{}}
Branch & Action & $\pi_{\theta}$ & Returned observation \\ \hline
n140 & \texttt{DOWN} & 0.066 & Score 1012. Largest tile 256. \\
\textbf{n141} & \texttt{LEFT} & 0.892 & Score 1076. Largest tile 256. \\
n142 & \texttt{RIGHT} & 0.012 & Score 1076. Largest tile 256. \\
n143 & \texttt{UP} & 0.031 & Score 1012. Largest tile 256. \\
\end{tabularx}
\end{examplebox}
\begin{examplebox}[frametitle={Round 37},linecolor=black,backgroundcolor=caseblue!9,frametitlebackgroundcolor=caseblue!20,nobreak=true]
\small Selected branch \texttt{n141}, selection probability 0.436.
\par Submitted actions \texttt{UP}, \texttt{RIGHT}, \texttt{LEFT}, \texttt{UP}, \texttt{UP}, \texttt{UP}, \texttt{UP}, \texttt{LEFT}, \texttt{DOWN}, \texttt{UP}, \texttt{LEFT}, \texttt{LEFT}, \texttt{LEFT}, \texttt{UP}, \texttt{UP}, \texttt{UP}, \texttt{LEFT}, \texttt{LEFT}, \texttt{UP}, \texttt{UP}, \texttt{LEFT}, \texttt{UP}, \texttt{LEFT}, \texttt{UP}, \texttt{LEFT}, \texttt{UP}, \texttt{UP}, \texttt{LEFT}, \texttt{UP}, \texttt{UP}, \texttt{LEFT}.
\end{examplebox}
\clearpage
\subsection{Nullify}
Two index fields select units for arithmetic operations, with the goal of eliminating all units. A solution with seven actions is found after revisiting retained branches. The initial domain omits unit 8, valued at -7. Reindexing after the first action brings that unit into the declared range, and the third action eliminates the unit.
\begin{examplebox}[frametitle={Task and outcome},linecolor=black]
Environment seed 129. JevSpawn score 1.
\par\smallskip\centering
\small\begin{adjustbox}{max width=\linewidth}\begin{tabular}{rrrrrrr}
LATS & LLMCompiler & AgentPrune & HiAgent & FoldAgent & DyFlow & LatentMAS \\
0 & 0 & 0 & 0 & 0 & 0 & 0 \\
\end{tabular}\end{adjustbox}
\end{examplebox}
\begin{center}\captionsetup{hypcap=false}
\includegraphics[width=\linewidth]{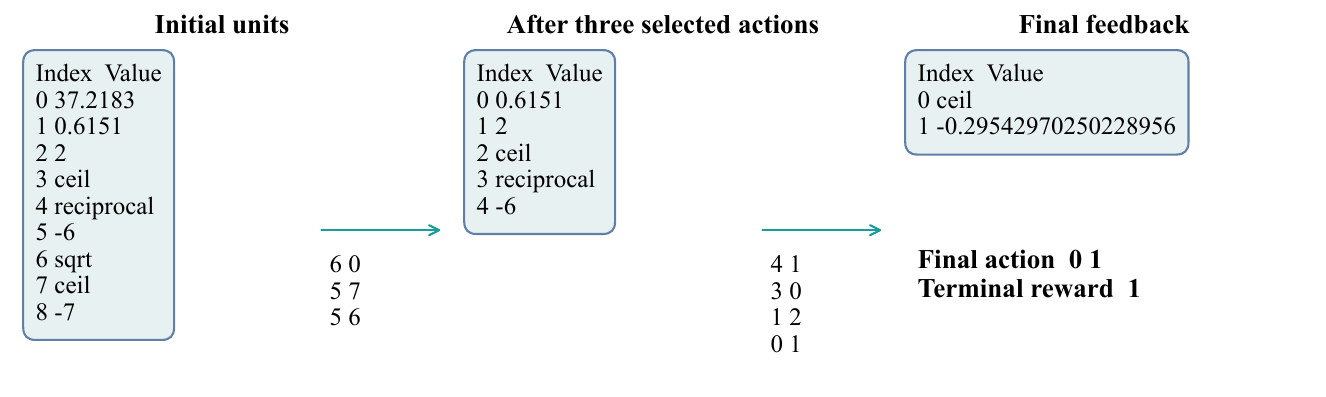}
\captionof{figure}{Nullify units during the selected trajectory. The final action rounds the remaining negative value up to zero. The terminal observation echoes the pre-action units and reports reward one.}
\end{center}
\begin{examplebox}[frametitle={Round 1},linecolor=black,backgroundcolor=caseblue!9,frametitlebackgroundcolor=caseblue!20,nobreak=true]
\small Initial state.
\par\smallskip Fields \texttt{f0} (first unit index) $\in$ \{\texttt{0}, \texttt{1}, \texttt{2}, \texttt{3}, \texttt{4}, \texttt{5}, \texttt{6}, \texttt{7}\}, \texttt{f1} (second unit index) $\in$ \{\texttt{0}, \texttt{1}, \texttt{2}, \texttt{3}, \texttt{4}, \texttt{5}, \texttt{6}, \texttt{7}\}.
\par Action template \texttt{execute(action = \$\{f0\} \$\{f1\})}.
\par\smallskip\noindent
\begin{tabularx}{\linewidth}{@{}l l r X@{}}
Branch & Action & $\pi_{\theta}$ & Returned observation \\ \hline
n0 & \texttt{0 2} & 0.068 & 0.6151, ceil, reciprocal, -6, sqrt, ceil, -7, 39.2183. \\
n1 & \texttt{0 6} & 0.081 & 0.6151, 2, ceil, reciprocal, -6, ceil, -7, 6.1006802899348855. \\
n2 & \texttt{5 7} & 0.181 & 37.2183, 0.6151, 2, ceil, reciprocal, sqrt, -7, -6. \\
\textbf{n3} & \texttt{6 0} & 0.079 & 0.6151, 2, ceil, reciprocal, -6, ceil, -7, 6.1006802899348855. \\
\end{tabularx}
\end{examplebox}
\begin{examplebox}[frametitle={Round 2},linecolor=black,backgroundcolor=casegreen!9,frametitlebackgroundcolor=casegreen!20,nobreak=true]
\small Selected branch \texttt{n3}, selection probability 0.779.
\par\smallskip\noindent
\begin{tabularx}{\linewidth}{@{}l l r X@{}}
Branch & Action & $\pi_{\theta}$ & Returned observation \\ \hline
n4 & \texttt{0 1} & 0.036 & ceil, reciprocal, -6, ceil, -7, 6.1006802899348855, 2.6151. \\
n5 & \texttt{4 6} & 0.035 & 0.6151, 2, ceil, reciprocal, ceil, 6.1006802899348855, -13. \\
\textbf{n6} & \texttt{5 7} & 0.078 & 0.6151, 2, ceil, reciprocal, -6, -7, 7. \\
n7 & \texttt{6 0} & 0.051 & 2, ceil, reciprocal, -6, ceil, 6.1006802899348855, -6.3849. \\
\end{tabularx}
\end{examplebox}
\begin{examplebox}[frametitle={Round 3},linecolor=black,backgroundcolor=caseblue!9,frametitlebackgroundcolor=caseblue!20,nobreak=true]
\small Selected branch \texttt{n7}, selection probability 0.254.
\par\smallskip\noindent
\begin{tabularx}{\linewidth}{@{}l l r X@{}}
Branch & Action & $\pi_{\theta}$ & Returned observation \\ \hline
n8 & \texttt{0 3} & 0.064 & ceil, reciprocal, ceil, 6.1006802899348855, -6.3849, -4. \\
n9 & \texttt{3 6} & 0.247 & 2, ceil, reciprocal, ceil, 6.1006802899348855, -12.3849. \\
n10 & \texttt{5 6} & 0.083 & 2, ceil, reciprocal, -6, ceil, -0.2842197100651145. \\
n11 & \texttt{6 3} & 0.050 & 2, ceil, reciprocal, ceil, 6.1006802899348855, -12.3849. \\
\end{tabularx}
\end{examplebox}
\begin{examplebox}[frametitle={Round 4 (recovery)},linecolor=black,backgroundcolor=casegreen!9,frametitlebackgroundcolor=casegreen!20,nobreak=true]
\small Selected branch \texttt{n6}, selection probability 0.208.
\par\smallskip\noindent
\begin{tabularx}{\linewidth}{@{}l l r X@{}}
Branch & Action & $\pi_{\theta}$ & Returned observation \\ \hline
n12 & \texttt{4 5} & 0.217 & 0.6151, 2, ceil, reciprocal, 7, -13. \\
n13 & \texttt{4 6} & 0.236 & 0.6151, 2, ceil, reciprocal, -7, 1. \\
n14 & \texttt{5 4} & 0.053 & 0.6151, 2, ceil, reciprocal, 7, -13. \\
\textbf{n15} & \texttt{5 6} & 0.093 & 0.6151, 2, ceil, reciprocal, -6. \\
\end{tabularx}
\end{examplebox}
\begin{examplebox}[frametitle={Round 5},linecolor=black,backgroundcolor=caseblue!9,frametitlebackgroundcolor=caseblue!20,nobreak=true]
\small Selected branch \texttt{n15}, selection probability 0.307.
\par\smallskip\noindent
\begin{tabularx}{\linewidth}{@{}l l r X@{}}
Branch & Action & $\pi_{\theta}$ & Returned observation \\ \hline
n16 & \texttt{0 1} & 0.085 & ceil, reciprocal, -6, 2.6151. \\
n17 & \texttt{1 4} & 0.073 & 0.6151, ceil, reciprocal, -4. \\
n18 & \texttt{2 4} & 0.078 & 0.6151, 2, reciprocal, -6. \\
\textbf{n19} & \texttt{4 1} & 0.089 & 0.6151, ceil, reciprocal, -4. \\
\end{tabularx}
\end{examplebox}
\begin{examplebox}[frametitle={Round 6 (recovery)},linecolor=black,backgroundcolor=casegreen!9,frametitlebackgroundcolor=casegreen!20,nobreak=true]
\small Selected branch \texttt{n8}, selection probability 0.328.
\par\smallskip\noindent
\begin{tabularx}{\linewidth}{@{}l l r X@{}}
Branch & Action & $\pi_{\theta}$ & Returned observation \\ \hline
n20 & \texttt{0 1} & 0.072 & ceil, reciprocal, ceil, 6.1006802899348855, -6.3849, -4. \\
n21 & \texttt{0 3} & 0.162 & reciprocal, ceil, -6.3849, -4, 7. \\
n22 & \texttt{4 5} & 0.232 & ceil, reciprocal, ceil, 6.1006802899348855, -10.3849. \\
n23 & \texttt{5 4} & 0.067 & ceil, reciprocal, ceil, 6.1006802899348855, -10.3849. \\
\end{tabularx}
\end{examplebox}
\begin{examplebox}[frametitle={Round 7 (recovery)},linecolor=black,backgroundcolor=caseblue!9,frametitlebackgroundcolor=caseblue!20,nobreak=true]
\small Selected branch \texttt{n19}, selection probability 0.249.
\par\smallskip\noindent
\begin{tabularx}{\linewidth}{@{}l l r X@{}}
Branch & Action & $\pi_{\theta}$ & Returned observation \\ \hline
n24 & \texttt{0 1} & 0.112 & reciprocal, -4, 1. \\
n25 & \texttt{0 3} & 0.199 & ceil, reciprocal, -3.3849. \\
n26 & \texttt{1 3} & 0.166 & 0.6151, reciprocal, -4. \\
\textbf{n27} & \texttt{3 0} & 0.114 & ceil, reciprocal, -3.3849. \\
\end{tabularx}
\end{examplebox}
\begin{examplebox}[frametitle={Round 8},linecolor=black,backgroundcolor=casegreen!9,frametitlebackgroundcolor=casegreen!20,nobreak=true]
\small Selected branch \texttt{n24}, selection probability 0.360.
\par\smallskip\noindent
\begin{tabularx}{\linewidth}{@{}l l r X@{}}
Branch & Action & $\pi_{\theta}$ & Returned observation \\ \hline
n28 & \texttt{0 1} & 0.244 & 1, -0.25. \\
n29 & \texttt{0 2} & 0.468 & -4, 1.0. \\
n30 & \texttt{1 2} & 0.172 & reciprocal, -3. \\
n31 & \texttt{2 0} & 0.047 & -4, 1.0. \\
\end{tabularx}
\end{examplebox}
\begin{examplebox}[frametitle={Round 9 (recovery)},linecolor=black,backgroundcolor=caseblue!9,frametitlebackgroundcolor=caseblue!20,nobreak=true]
\small Selected branch \texttt{n27}, selection probability 0.273.
\par\smallskip\noindent
\begin{tabularx}{\linewidth}{@{}l l r X@{}}
Branch & Action & $\pi_{\theta}$ & Returned observation \\ \hline
n32 & \texttt{0 1} & 0.380 & ceil, reciprocal, -3.3849. \\
n33 & \texttt{0 2} & 0.306 & reciprocal, -3. \\
n34 & \texttt{1 0} & 0.047 & ceil, reciprocal, -3.3849. \\
\textbf{n35} & \texttt{1 2} & 0.204 & ceil, -0.29542970250228956. \\
\end{tabularx}
\end{examplebox}
\begin{examplebox}[frametitle={Round 10 (recovery)},linecolor=black,backgroundcolor=casegreen!9,frametitlebackgroundcolor=casegreen!20,nobreak=true]
\small Selected branch \texttt{n31}, selection probability 0.427.
\par\smallskip\noindent
\begin{tabularx}{\linewidth}{@{}l l r X@{}}
Branch & Action & $\pi_{\theta}$ & Returned observation \\ \hline
n36 & \texttt{0 1} & 0.939 & -4, 1.0. Terminal reward 0. \\
n37 & \texttt{0 2} & 0.002 & -4, 1.0. \\
n38 & \texttt{1 0} & 0.054 & -4, 1.0. Terminal reward 0. \\
n39 & \texttt{1 2} & 0.001 & -4, 1.0. \\
\end{tabularx}
\end{examplebox}
\begin{examplebox}[frametitle={Round 11 (recovery)},linecolor=black,backgroundcolor=caseblue!9,frametitlebackgroundcolor=caseblue!20,nobreak=true]
\small Selected branch \texttt{n35}, selection probability 0.594.
\par\smallskip\noindent
\begin{tabularx}{\linewidth}{@{}l l r X@{}}
Branch & Action & $\pi_{\theta}$ & Returned observation \\ \hline
\textbf{n40} & \texttt{0 1} & 0.956 & ceil, -0.29542970250228956. Terminal reward 1. \\
n41 & \texttt{0 2} & 0.001 & ceil, -0.29542970250228956. \\
n42 & \texttt{1 0} & 0.040 & ceil, -0.29542970250228956. Terminal reward 1. \\
n43 & \texttt{1 2} & 0.002 & ceil, -0.29542970250228956. \\
\end{tabularx}
\end{examplebox}
\begin{examplebox}[frametitle={Round 12 (recovery)},linecolor=black,backgroundcolor=casegreen!9,frametitlebackgroundcolor=casegreen!20,nobreak=true]
\small Selected branch \texttt{n30}, selection probability 0.426.
\par\smallskip\noindent
\begin{tabularx}{\linewidth}{@{}l l r X@{}}
Branch & Action & $\pi_{\theta}$ & Returned observation \\ \hline
n44 & \texttt{0 1} & 0.911 & reciprocal, -3. Terminal reward 0. \\
n45 & \texttt{0 2} & 0.008 & reciprocal, -3. \\
n46 & \texttt{1 0} & 0.069 & reciprocal, -3. Terminal reward 0. \\
n47 & \texttt{1 2} & 0.011 & reciprocal, -3. \\
\end{tabularx}
\end{examplebox}
\begin{examplebox}[frametitle={Round 13 (recovery)},linecolor=black,backgroundcolor=caseblue!9,frametitlebackgroundcolor=caseblue!20,nobreak=true]
\small Selected branch \texttt{n40}, selection probability 0.468.
\par Submitted actions \texttt{6 0}, \texttt{5 7}, \texttt{5 6}, \texttt{4 1}, \texttt{3 0}, \texttt{1 2}, \texttt{0 1}.
\end{examplebox}
\endgroup

%% file: references.bib
@string{June = "June"}

@inproceedings{arxiv231207104,
 author = {Zheng, Lianmin and Yin, Liangsheng and Xie, Zhiqiang and Sun, Chuyue and Huang, Jeff and Yu, Cody Hao and Cao, Shiyi and Kozyrakis, Christos and Stoica, Ion and Gonzalez, Joseph E. and Barrett, Clark and Sheng, Ying},
 booktitle = {Advances in Neural Information Processing Systems},
 doi = {10.52202/079017-2000},
 editor = {A. Globerson and L. Mackey and D. Belgrave and A. Fan and U. Paquet and J. Tomczak and C. Zhang},
 pages = {62557--62583},
 publisher = {Curran Associates, Inc.},
 title = {SGLang: Efficient Execution of Structured Language Model Programs},
 url = {https://proceedings.neurips.cc/paper_files/paper/2024/file/724be4472168f31ba1c9ac630f15dec8-Paper-Conference.pdf},
 volume = {37},
 year = {2024}
}

@inproceedings{hu-etal-2025-hiagent,
    title = "{H}i{A}gent: Hierarchical Working Memory Management for Solving Long-Horizon Agent Tasks with Large Language Model",
    author = "Hu, Mengkang  and
      Chen, Tianxing  and
      Chen, Qiguang  and
      Mu, Yao  and
      Shao, Wenqi  and
      Luo, Ping",
    editor = "Che, Wanxiang  and
      Nabende, Joyce  and
      Shutova, Ekaterina  and
      Pilehvar, Mohammad Taher",
    booktitle = "Proceedings of the 63rd Annual Meeting of the Association for Computational Linguistics (Volume 1: Long Papers)",
    month = jul,
    year = "2025",
    address = "Vienna, Austria",
    publisher = "Association for Computational Linguistics",
    url = "https://aclanthology.org/2025.acl-long.1575/",
    doi = "10.18653/v1/2025.acl-long.1575",
    pages = "32779--32798",
    ISBN = "979-8-89176-251-0"
}

@InProceedings{kim2024llmcompiler,
  title = 	 {An {LLM} Compiler for Parallel Function Calling},
  author =       {Kim, Sehoon and Moon, Suhong and Tabrizi, Ryan and Lee, Nicholas and Mahoney, Michael W. and Keutzer, Kurt and Gholami, Amir},
  booktitle = 	 {Proceedings of the 41st International Conference on Machine Learning},
  pages = 	 {24370--24391},
  year = 	 {2024},
  editor = 	 {Salakhutdinov, Ruslan and Kolter, Zico and Heller, Katherine and Weller, Adrian and Oliver, Nuria and Scarlett, Jonathan and Berkenkamp, Felix},
  volume = 	 {235},
  series = 	 {Proceedings of Machine Learning Research},
  month = 	 {21--27 Jul},
  publisher =    {PMLR},
  url = 	 {https://proceedings.mlr.press/v235/kim24y.html},

}

@inproceedings{korgym,
 author = {Shi, Jiajun and Yang, Jian and Liu, Jiaheng and Bu, Xingyuan and Chen, Jiangjie and Zhou, Junting and Ma, Kaijing and Wen, Zhoufutu and Wang, Bingli and He, Yancheng and Song, Liang and Zhu, Hualei and Li, Shilong and Wang, Xingjian and Zhang, Wei and Yuan, Ruibin and Yao, Yifan and Yang, Wenjun and Wang, Yunli and Fang, Siyuan and Yuan, Siyu and He, Qianyu and Tang, Robert and Tan, Yingshui and Zhou, Wangchunshu and ZHANG, ZHAO-XIANG and Li, Zhoujun and Huang, Wenhao and Zhang, Ge},
 booktitle = {Advances in Neural Information Processing Systems},
 doi = {10.52202/085713-5384},
 editor = {D. Belgrave and C. Zhang and H. Lin and R. Pascanu and P. Koniusz and M. Ghassemi and N. Chen},
 pages = {161286--161314},
 publisher = {Curran Associates, Inc.},
 title = {KORGym: A Dynamic Game Platform for LLM Reasoning Evaluation},
 url = {https://proceedings.neurips.cc/paper_files/paper/2025/file/ebfa4297cd6419f64efe86f657ba49d0-Paper-Conference.pdf},
 volume = {38, Main Conference},
 year = {2025}
}

@misc{llfbench,
  doi = {10.48550/ARXIV.2312.06853},
  url = {https://arxiv.org/abs/2312.06853},
  author = {Cheng, Ching-An and Kolobov, Andrey and Misra, Dipendra and Nie, Allen and Swaminathan, Adith},
  title = {LLF-Bench: Benchmark for Interactive Learning from Language Feedback},
  publisher = {arXiv},
  year = {2023},
  copyright = {arXiv.org perpetual, non-exclusive license}
}

@InProceedings{lmrlgym,
  title = 	 {{LMRL} Gym: Benchmarks for Multi-Turn Reinforcement Learning with Language Models},
  author =       {Abdulhai, Marwa and White, Isadora and Snell, Charlie Victor and Sun, Charles and Hong, Joey and Zhai, Yuexiang and Xu, Kelvin and Levine, Sergey},
  booktitle = 	 {Proceedings of the 42nd International Conference on Machine Learning},
  pages = 	 {126--153},
  year = 	 {2025},
  editor = 	 {Singh, Aarti and Fazel, Maryam and Hsu, Daniel and Lacoste-Julien, Simon and Berkenkamp, Felix and Maharaj, Tegan and Wagstaff, Kiri and Zhu, Jerry},
  volume = 	 {267},
  series = 	 {Proceedings of Machine Learning Research},
  month = 	 {13--19 Jul},
  publisher =    {PMLR},
  url = 	 {https://proceedings.mlr.press/v267/abdulhai25a.html},

}

@misc{ppnl,
  doi = {10.48550/ARXIV.2310.03249},
  url = {https://arxiv.org/abs/2310.03249},
  author = {Aghzal, Mohamed and Plaku, Erion and Yao, Ziyu},
  title = {Can Large Language Models be Good Path Planners? A Benchmark and Investigation on Spatial-temporal Reasoning},
  publisher = {arXiv},
  year = {2023},
  copyright = {Creative Commons Attribution 4.0 International}
}

@inproceedings{
sun2026scaling,
title={Scaling Long-Horizon Agent via Context Folding},
author={Weiwei Sun and Miao Lu and Zhan Ling and Kang Liu and Xuesong Yao and Yiming Yang and Jiecao Chen},
booktitle={Forty-third International Conference on Machine Learning},
year={2026},
url={https://openreview.net/forum?id=lNRgWoGfYg}
}

@misc{textarena,
  doi = {10.48550/ARXIV.2504.11442},
  url = {https://arxiv.org/abs/2504.11442},
  author = {Guertler, Leon and Cheng, Bobby and Yu, Simon and Liu, Bo and Choshen, Leshem and Tan, Cheston},
  title = {TextArena},
  publisher = {arXiv},
  year = {2025},
  copyright = {Creative Commons Attribution 4.0 International}
}

@inproceedings{wang2025dyflow,
 author = {Wang, Yanbo and Xu, Zixiang and Huang, Yue and Wang, Xiangqi and Song, Zirui and Gao, Lang and Wang, Chenxi and Tang, Robert and Zhao, Yue and Cohan, Arman and Zhang, Xiangliang and Chen, Xiuying},
 booktitle = {Advances in Neural Information Processing Systems},
 doi = {10.52202/085713-5793},
 editor = {D. Belgrave and C. Zhang and H. Lin and R. Pascanu and P. Koniusz and M. Ghassemi and N. Chen},
 pages = {174148--174181},
 publisher = {Curran Associates, Inc.},
 title = {DyFlow: Dynamic Workflow Framework for Agentic Reasoning},
 url = {https://proceedings.neurips.cc/paper_files/paper/2025/file/fe9910d2b03324faeb5371a9658277bb-Paper-Conference.pdf},
 volume = {38, Main Conference},
 year = {2025}
}

@inproceedings{
yao2023react,
title={ReAct: Synergizing Reasoning and Acting in Language Models},
author={Shunyu Yao and Jeffrey Zhao and Dian Yu and Nan Du and Izhak Shafran and Karthik R Narasimhan and Yuan Cao},
booktitle={The Eleventh International Conference on Learning Representations },
year={2023},
url={https://openreview.net/forum?id=WE_vluYUL-X}
}

@inproceedings{yao2023treeofthoughts,
 author = {Yao, Shunyu and Yu, Dian and Zhao, Jeffrey and Shafran, Izhak and Griffiths, Tom and Cao, Yuan and Narasimhan, Karthik},
 booktitle = {Advances in Neural Information Processing Systems},
 doi = {10.52202/075280-0517},
 editor = {A. Oh and T. Naumann and A. Globerson and K. Saenko and M. Hardt and S. Levine},
 pages = {11809--11822},
 publisher = {Curran Associates, Inc.},
 title = {Tree of Thoughts: Deliberate Problem Solving with Large Language Models},
 url = {https://proceedings.neurips.cc/paper_files/paper/2023/file/271db9922b8d1f4dd7aaef84ed5ac703-Paper-Conference.pdf},
 volume = {36},
 year = {2023}
}

@inproceedings{zhang2025agentprune,
 author = {Zhang, Guibin and Yue, Yanwei and Li, Zhixun and Yun, Sukwon and Wan, Guancheng and Wang, Kun and Cheng, Dawei and Yu, Jeffrey and Chen, Tianlong},
 booktitle = {International Conference on Learning Representations},
 editor = {Y. Yue and A. Garg and N. Peng and F. Sha and R. Yu},
 pages = {75389--75428},
 title = {Cut the Crap: An Economical Communication Pipeline for LLM-based Multi-Agent Systems},
 url = {https://proceedings.iclr.cc/paper_files/paper/2025/file/bbc461518c59a2a8d64e70e2c38c4a0e-Paper-Conference.pdf},
 volume = {2025},
 year = {2025}
}

@InProceedings{zhou2024lats,
  title = 	 {Language Agent Tree Search Unifies Reasoning, Acting, and Planning in Language Models},
  author =       {Zhou, Andy and Yan, Kai and Shlapentokh-Rothman, Michal and Wang, Haohan and Wang, Yu-Xiong},
  booktitle = 	 {Proceedings of the 41st International Conference on Machine Learning},
  pages = 	 {62138--62160},
  year = 	 {2024},
  editor = 	 {Salakhutdinov, Ruslan and Kolter, Zico and Heller, Katherine and Weller, Adrian and Oliver, Nuria and Scarlett, Jonathan and Berkenkamp, Felix},
  volume = 	 {235},
  series = 	 {Proceedings of Machine Learning Research},
  month = 	 {21--27 Jul},
  publisher =    {PMLR},
  url = 	 {https://proceedings.mlr.press/v235/zhou24r.html},

}

@inproceedings{
zou2026latent,
title={Latent Collaboration in Multi-Agent Systems},
author={Jiaru Zou and Ruizhong Qiu and Gaotang Li and Xiyuan Yang and Katherine Tieu and Pan Lu and Ke Shen and Hanghang Tong and Yejin Choi and Jingrui He and James Zou and Mengdi Wang and Ling Yang},
booktitle={Forty-third International Conference on Machine Learning},
year={2026},
url={https://openreview.net/forum?id=syG9I9ofd8}
}

@inproceedings{MLSYS2025_5c20ca4b,
 author = {Dong, Yixin and Ruan, Charlie F. and Cai, Yaxing and Xu, Ziyi and Zhao, Yilong and Lai, Ruihang and Chen, Tianqi},
 booktitle = {Proceedings of Machine Learning and Systems},
 editor = {M. Zaharia and G. Joshi and Y. Lin},
 pages = {},
 publisher = {MLSys},
 title = {XGrammar: Flexible and Efficient Structured Generation Engine for Large Language Models},
 url = {https://proceedings.mlsys.org/paper_files/paper/2025/file/5c20ca4b0b20b0bd2f1d839dc605e70f-Paper-Conference.pdf},
 volume = {7},
 year = {2025}
}

@inproceedings{NEURIPS2024_2bdc2267,
 author = {Park, Kanghee and Wang, Jiayu and Berg-Kirkpatrick, Taylor and Polikarpova, Nadia and D\textquotesingle Antoni, Loris},
 booktitle = {Advances in Neural Information Processing Systems},
 doi = {10.52202/079017-0774},
 editor = {A. Globerson and L. Mackey and D. Belgrave and A. Fan and U. Paquet and J. Tomczak and C. Zhang},
 pages = {24547--24568},
 publisher = {Curran Associates, Inc.},
 title = {Grammar-Aligned Decoding},
 url = {https://proceedings.neurips.cc/paper_files/paper/2024/file/2bdc2267c3d7d01523e2e17ac0a754f3-Paper-Conference.pdf},
 volume = {37},
 year = {2024}
}

@inproceedings{Guo_2024,
  title     = {Large Language Model Based Multi-agents: A Survey of Progress and Challenges},
  author    = {Guo, Taicheng and Chen, Xiuying and Wang, Yaqi and Chang, Ruidi and Pei, Shichao and Chawla, Nitesh V. and Wiest, Olaf and Zhang, Xiangliang},
  booktitle = {Proceedings of the Thirty-Third International Joint Conference on
               Artificial Intelligence, {IJCAI-24}},
  publisher = {International Joint Conferences on Artificial Intelligence Organization},
  editor    = {Kate Larson},
  pages     = {8048--8057},
  year      = {2024},
  month     = {8},
  note      = {Survey Track},
  doi       = {10.24963/ijcai.2024/890},
  url       = {https://doi.org/10.24963/ijcai.2024/890},
}

@misc{https://doi.org/10.48550/arxiv.2601.14192,
  doi = {10.48550/ARXIV.2601.14192},
  url = {https://arxiv.org/abs/2601.14192},
  author = {Yang, Xiaofang and Li, Lijun and Zhou, Heng and Zhu, Tong and Qu, Xiaoye and Fan, Yuchen and Wei, Qianshan and Ye, Rui and Kang, Li and Qin, Yiran and Liu, Daizong and Li, Qi and Ding, Ning and Chen, Siheng and Shao, Jing},
  title = {Toward Efficient Agents: Memory, Tool learning, and Planning},
  publisher = {arXiv},
  year = {2026},
  copyright = {Creative Commons Attribution 4.0 International}
}

@misc{https://doi.org/10.48550/arxiv.2502.14321,
  doi = {10.48550/ARXIV.2502.14321},
  url = {https://arxiv.org/abs/2502.14321},
  author = {Yan, Bingyu and Zhou, Zhibo and Zhang, Litian and Zhang, Lian and Zhou, Ziyi and Miao, Dezhuang and Li, Zhoujun and Li, Chaozhuo and Zhang, Xiaoming},
  title = {Beyond Self-Talk: A Communication-Centric Survey of LLM-Based Multi-Agent Systems},
  publisher = {arXiv},
  year = {2025},
  copyright = {arXiv.org perpetual, non-exclusive license}
}

@inproceedings{wang-etal-2025-agentdropout,
    title = "{A}gent{D}ropout: Dynamic Agent Elimination for Token-Efficient and High-Performance {LLM}-Based Multi-Agent Collaboration",
    author = "Wang, Zhexuan  and
      Wang, Yutong  and
      Liu, Xuebo  and
      Ding, Liang  and
      Zhang, Miao  and
      Liu, Jie  and
      Zhang, Min",
    editor = "Che, Wanxiang  and
      Nabende, Joyce  and
      Shutova, Ekaterina  and
      Pilehvar, Mohammad Taher",
    booktitle = "Proceedings of the 63rd Annual Meeting of the Association for Computational Linguistics (Volume 1: Long Papers)",
    month = jul,
    year = "2025",
    address = "Vienna, Austria",
    publisher = "Association for Computational Linguistics",
    url = "https://aclanthology.org/2025.acl-long.1170/",
    doi = "10.18653/v1/2025.acl-long.1170",
    pages = "24013--24035",
    ISBN = "979-8-89176-251-0"
}

@inproceedings{geng-etal-2023-grammar,
    title = "Grammar-Constrained Decoding for Structured {NLP} Tasks without Finetuning",
    author = "Geng, Saibo  and
      Josifoski, Martin  and
      Peyrard, Maxime  and
      West, Robert",
    editor = "Bouamor, Houda  and
      Pino, Juan  and
      Bali, Kalika",
    booktitle = "Proceedings of the 2023 Conference on Empirical Methods in Natural Language Processing",
    month = dec,
    year = "2023",
    address = "Singapore",
    publisher = "Association for Computational Linguistics",
    url = "https://aclanthology.org/2023.emnlp-main.674/",
    doi = "10.18653/v1/2023.emnlp-main.674",
    pages = "10932--10952"
}

@InProceedings{pmlr-v235-beurer-kellner24a,
  title = 	 {Guiding {LLM}s The Right Way: Fast, Non-Invasive Constrained Generation},
  author =       {Beurer-Kellner, Luca and Fischer, Marc and Vechev, Martin},
  booktitle = 	 {Proceedings of the 41st International Conference on Machine Learning},
  pages = 	 {3658--3673},
  year = 	 {2024},
  editor = 	 {Salakhutdinov, Ruslan and Kolter, Zico and Heller, Katherine and Weller, Adrian and Oliver, Nuria and Scarlett, Jonathan and Berkenkamp, Felix},
  volume = 	 {235},
  series = 	 {Proceedings of Machine Learning Research},
  month = 	 {21--27 Jul},
  publisher =    {PMLR},
  url = 	 {https://proceedings.mlr.press/v235/beurer-kellner24a.html},
}

@misc{almeida2026jev,
  author = {Diogo Almeida},
  title = {Introducing System One Models \& {Jev}},
  year = {2026},
  month = sep,
  howpublished = {TypeSafe AI. \url{https://typesafe.ai/blog/introducing-system-one-models-and-jev}}
}

@inproceedings{ICLR2025_5492ecbc,
 author = {Zhang, Jiayi and Xiang, Jinyu and Yu, Zhaoyang and Teng, Fengwei and Chen, XiongHui and Chen, Jiaqi and Zhuge, Mingchen and Cheng, Xin and Hong, Sirui and Wang, Jinlin and Zheng, Bingnan and Liu, Bang and Luo, Yuyu and Wu, Chenglin},
 booktitle = {International Conference on Learning Representations},
 editor = {Y. Yue and A. Garg and N. Peng and F. Sha and R. Yu},
 pages = {34040--34077},
 title = {AFlow: Automating Agentic Workflow Generation},
 url = {https://proceedings.iclr.cc/paper_files/paper/2025/file/5492ecbce4439401798dcd2c90be94cd-Paper-Conference.pdf},
 volume = {2025},
 year = {2025}
}

@inproceedings{ICLR2025_36b7acf6,
 author = {Hu, Shengran and Lu, Cong and Clune, Jeff},
 booktitle = {International Conference on Learning Representations},
 editor = {Y. Yue and A. Garg and N. Peng and F. Sha and R. Yu},
 pages = {21344--21377},
 title = {Automated Design of Agentic Systems},
 url = {https://proceedings.iclr.cc/paper_files/paper/2025/file/36b7acf6f6010652b3f2a433774a66fe-Paper-Conference.pdf},
 volume = {2025},
 year = {2025}
}

@inproceedings{NEURIPS2023_1b44b878,
 author = {Shinn, Noah and Cassano, Federico and Gopinath, Ashwin and Narasimhan, Karthik and Yao, Shunyu},
 booktitle = {Advances in Neural Information Processing Systems},
 doi = {10.52202/075280-0377},
 editor = {A. Oh and T. Naumann and A. Globerson and K. Saenko and M. Hardt and S. Levine},
 pages = {8634--8652},
 publisher = {Curran Associates, Inc.},
 title = {Reflexion: language agents with verbal reinforcement learning},
 url = {https://proceedings.neurips.cc/paper_files/paper/2023/file/1b44b878bb782e6954cd888628510e90-Paper-Conference.pdf},
 volume = {36},
 year = {2023}
}

@article{Beurer_Kellner_2023, title={Prompting Is Programming: A Query Language for Large Language Models}, volume={7}, ISSN={2475-1421}, url={http://dx.doi.org/10.1145/3591300}, DOI={10.1145/3591300}, number={PLDI}, journal={Proceedings of the ACM on Programming Languages}, publisher={Association for Computing Machinery (ACM)}, author={Beurer-Kellner, Luca and Fischer, Marc and Vechev, Martin}, year={2023}, month=June, pages={1946--1969} }

@misc{https://doi.org/10.48550/arxiv.2409.07429,
  doi = {10.48550/ARXIV.2409.07429},
  url = {https://arxiv.org/abs/2409.07429},
  author = {Wang, Zora Zhiruo and Mao, Jiayuan and Fried, Daniel and Neubig, Graham},
  title = {Agent Workflow Memory},
  publisher = {arXiv},
  year = {2024},
  copyright = {Creative Commons Attribution Share Alike 4.0 International}
}

@misc{https://doi.org/10.48550/arxiv.2308.08155,
  doi = {10.48550/ARXIV.2308.08155},
  url = {https://arxiv.org/abs/2308.08155},
  author = {Wu, Qingyun and Bansal, Gagan and Zhang, Jieyu and Wu, Yiran and Li, Beibin and Zhu, Erkang and Jiang, Li and Zhang, Xiaoyun and Zhang, Shaokun and Liu, Jiale and Awadallah, Ahmed Hassan and White, Ryen W and Burger, Doug and Wang, Chi},
  title = {AutoGen: Enabling Next-Gen LLM Applications via Multi-Agent Conversation},
  publisher = {arXiv},
  year = {2023},
  copyright = {Creative Commons Attribution 4.0 International}
}

@misc{https://doi.org/10.48550/arxiv.2402.14034,
  doi = {10.48550/ARXIV.2402.14034},
  url = {https://arxiv.org/abs/2402.14034},
  author = {Gao, Dawei and Li, Zitao and Pan, Xuchen and Kuang, Weirui and Ma, Zhijian and Qian, Bingchen and Wei, Fei and Zhang, Wenhao and Xie, Yuexiang and Chen, Daoyuan and Yao, Liuyi and Peng, Hongyi and Zhang, Zeyu and Zhu, Lin and Cheng, Chen and Shi, Hongzhu and Li, Yaliang and Ding, Bolin and Zhou, Jingren},
  title = {AgentScope: A Flexible yet Robust Multi-Agent Platform},
  publisher = {arXiv},
  year = {2024},
  copyright = {arXiv.org perpetual, non-exclusive license}
}

@misc{packer2024memgptllmsoperatingsystems,
  doi = {10.48550/ARXIV.2310.08560},
  url = {https://arxiv.org/abs/2310.08560},
  author = {Packer, Charles and Wooders, Sarah and Lin, Kevin and Fang, Vivian and Patil, Shishir G. and Stoica, Ion and Gonzalez, Joseph E.},
  title = {MemGPT: Towards LLMs as Operating Systems},
  publisher = {arXiv},
  year = {2023},
  copyright = {Creative Commons Attribution 4.0 International}
}

@misc{https://doi.org/10.48550/arxiv.2304.03442,
  doi = {10.48550/ARXIV.2304.03442},
  url = {https://arxiv.org/abs/2304.03442},
  author = {Park, Joon Sung and O'Brien, Joseph C. and Cai, Carrie J. and Morris, Meredith Ringel and Liang, Percy and Bernstein, Michael S.},
  title = {Generative Agents: Interactive Simulacra of Human Behavior},
  publisher = {arXiv},
  year = {2023},
  copyright = {arXiv.org perpetual, non-exclusive license}
}

@misc{https://doi.org/10.48550/arxiv.1706.03762,
  doi = {10.48550/ARXIV.1706.03762},
  url = {https://arxiv.org/abs/1706.03762},
  author = {Vaswani, Ashish and Shazeer, Noam and Parmar, Niki and Uszkoreit, Jakob and Jones, Llion and Gomez, Aidan N. and Kaiser, Lukasz and Polosukhin, Illia},
  title = {Attention Is All You Need},
  publisher = {arXiv},
  year = {2017},
  copyright = {arXiv.org perpetual, non-exclusive license}
}

@misc{https://doi.org/10.48550/arxiv.1901.02860,
  doi = {10.48550/ARXIV.1901.02860},
  url = {https://arxiv.org/abs/1901.02860},
  author = {Dai, Zihang and Yang, Zhilin and Yang, Yiming and Carbonell, Jaime and Le, Quoc V. and Salakhutdinov, Ruslan},
  title = {Transformer-XL: Attentive Language Models Beyond a Fixed-Length Context},
  publisher = {arXiv},
  year = {2019},
  copyright = {Creative Commons Attribution Non Commercial Share Alike 4.0 International}
}

@misc{https://doi.org/10.48550/arxiv.1911.05507,
  doi = {10.48550/ARXIV.1911.05507},
  url = {https://arxiv.org/abs/1911.05507},
  author = {Rae, Jack W. and Potapenko, Anna and Jayakumar, Siddhant M. and Lillicrap, Timothy P.},
  title = {Compressive Transformers for Long-Range Sequence Modelling},
  publisher = {arXiv},
  year = {2019},
  copyright = {arXiv.org perpetual, non-exclusive license}
}

@misc{https://doi.org/10.48550/arxiv.2004.05150,
  doi = {10.48550/ARXIV.2004.05150},
  url = {https://arxiv.org/abs/2004.05150},
  author = {Beltagy, Iz and Peters, Matthew E. and Cohan, Arman},
  title = {Longformer: The Long-Document Transformer},
  publisher = {arXiv},
  year = {2020},
  copyright = {arXiv.org perpetual, non-exclusive license}
}

@article{https://doi.org/10.48550/arxiv.2007.14062,
  doi = {10.48550/ARXIV.2007.14062},
  url = {https://arxiv.org/abs/2007.14062},
  author = {Zaheer, Manzil and Guruganesh, Guru and Dubey, Avinava and Ainslie, Joshua and Alberti, Chris and Ontanon, Santiago and Pham, Philip and Ravula, Anirudh and Wang, Qifan and Yang, Li and Ahmed, Amr},
  title = {Big Bird: Transformers for Longer Sequences},
  publisher = {arXiv},
  year = {2020},
  copyright = {arXiv.org perpetual, non-exclusive license}
}

@misc{https://doi.org/10.48550/arxiv.2001.04451,
  doi = {10.48550/ARXIV.2001.04451},
  url = {https://arxiv.org/abs/2001.04451},
  author = {Kitaev, Nikita and Kaiser, {\L{}}ukasz and Levskaya, Anselm},
  title = {Reformer: The Efficient Transformer},
  publisher = {arXiv},
  year = {2020},
  copyright = {arXiv.org perpetual, non-exclusive license}
}

@misc{https://doi.org/10.48550/arxiv.2009.14794,
  doi = {10.48550/ARXIV.2009.14794},
  url = {https://arxiv.org/abs/2009.14794},
  author = {Choromanski, Krzysztof and Likhosherstov, Valerii and Dohan, David and Song, Xingyou and Gane, Andreea and Sarlos, Tamas and Hawkins, Peter and Davis, Jared and Mohiuddin, Afroz and Kaiser, Lukasz and Belanger, David and Colwell, Lucy and Weller, Adrian},
  title = {Rethinking Attention with Performers},
  publisher = {arXiv},
  year = {2020},
  copyright = {arXiv.org perpetual, non-exclusive license}
}

@misc{https://doi.org/10.48550/arxiv.2006.04768,
  doi = {10.48550/ARXIV.2006.04768},
  url = {https://arxiv.org/abs/2006.04768},
  author = {Wang, Sinong and Li, Belinda Z. and Khabsa, Madian and Fang, Han and Ma, Hao},
  title = {Linformer: Self-Attention with Linear Complexity},
  publisher = {arXiv},
  year = {2020},
  copyright = {arXiv.org perpetual, non-exclusive license}
}

@misc{https://doi.org/10.48550/arxiv.2006.16236,
  doi = {10.48550/ARXIV.2006.16236},
  url = {https://arxiv.org/abs/2006.16236},
  author = {Katharopoulos, Angelos and Vyas, Apoorv and Pappas, Nikolaos and Fleuret, Fran{\c{c}}ois},
  title = {Transformers are RNNs: Fast Autoregressive Transformers with Linear Attention},
  publisher = {arXiv},
  year = {2020},
  copyright = {arXiv.org perpetual, non-exclusive license}
}

@misc{https://doi.org/10.48550/arxiv.2205.14135,
  doi = {10.48550/ARXIV.2205.14135},
  url = {https://arxiv.org/abs/2205.14135},
  author = {Dao, Tri and Fu, Daniel Y. and Ermon, Stefano and Rudra, Atri and R{\'e}, Christopher},
  title = {FlashAttention: Fast and Memory-Efficient Exact Attention with IO-Awareness},
  publisher = {arXiv},
  year = {2022},
  copyright = {arXiv.org perpetual, non-exclusive license}
}

@misc{https://doi.org/10.48550/arxiv.2307.08691,
  doi = {10.48550/ARXIV.2307.08691},
  url = {https://arxiv.org/abs/2307.08691},
  author = {Dao, Tri},
  title = {FlashAttention-2: Faster Attention with Better Parallelism and Work Partitioning},
  publisher = {arXiv},
  year = {2023},
  copyright = {arXiv.org perpetual, non-exclusive license}
}

@misc{https://doi.org/10.48550/arxiv.2407.08608,
  doi = {10.48550/ARXIV.2407.08608},
  url = {https://arxiv.org/abs/2407.08608},
  author = {Shah, Jay and Bikshandi, Ganesh and Zhang, Ying and Thakkar, Vijay and Ramani, Pradeep and Dao, Tri},
  title = {FlashAttention-3: Fast and Accurate Attention with Asynchrony and Low-precision},
  publisher = {arXiv},
  year = {2024},
  copyright = {Creative Commons Attribution 4.0 International}
}

@misc{https://doi.org/10.48550/arxiv.1911.02150,
  doi = {10.48550/ARXIV.1911.02150},
  url = {https://arxiv.org/abs/1911.02150},
  author = {Shazeer, Noam},
  title = {Fast Transformer Decoding: One Write-Head is All You Need},
  publisher = {arXiv},
  year = {2019},
  copyright = {arXiv.org perpetual, non-exclusive license}
}

@misc{https://doi.org/10.48550/arxiv.2305.13245,
  doi = {10.48550/ARXIV.2305.13245},
  url = {https://arxiv.org/abs/2305.13245},
  author = {Ainslie, Joshua and Lee-Thorp, James and de Jong, Michiel and Zemlyanskiy, Yury and Lebr{\'o}n, Federico and Sanghai, Sumit},
  title = {GQA: Training Generalized Multi-Query Transformer Models from Multi-Head Checkpoints},
  publisher = {arXiv},
  year = {2023},
  copyright = {Creative Commons Attribution 4.0 International}
}

@misc{https://doi.org/10.48550/arxiv.2401.09670,
  doi = {10.48550/ARXIV.2401.09670},
  url = {https://arxiv.org/abs/2401.09670},
  author = {Zhong, Yinmin and Liu, Shengyu and Chen, Junda and Hu, Jianbo and Zhu, Yibo and Liu, Xuanzhe and Jin, Xin and Zhang, Hao},
  title = {DistServe: Disaggregating Prefill and Decoding for Goodput-optimized Large Language Model Serving},
  publisher = {arXiv},
  year = {2024},
  copyright = {Creative Commons Attribution Share Alike 4.0 International}
}

@misc{https://doi.org/10.48550/arxiv.2311.18677,
  doi = {10.48550/ARXIV.2311.18677},
  url = {https://arxiv.org/abs/2311.18677},
  author = {Patel, Pratyush and Choukse, Esha and Zhang, Chaojie and Shah, Aashaka and Goiri, {\'I}{\~n}igo and Maleki, Saeed and Bianchini, Ricardo},
  title = {Splitwise: Efficient generative LLM inference using phase splitting},
  publisher = {arXiv},
  year = {2023},
  copyright = {Creative Commons Attribution Non Commercial Share Alike 4.0 International}
}

@misc{https://doi.org/10.48550/arxiv.2308.16369,
  doi = {10.48550/ARXIV.2308.16369},
  url = {https://arxiv.org/abs/2308.16369},
  author = {Agrawal, Amey and Panwar, Ashish and Mohan, Jayashree and Kwatra, Nipun and Gulavani, Bhargav S. and Ramjee, Ramachandran},
  title = {SARATHI: Efficient LLM Inference by Piggybacking Decodes with Chunked Prefills},
  publisher = {arXiv},
  year = {2023},
  copyright = {arXiv.org perpetual, non-exclusive license}
}

@misc{https://doi.org/10.48550/arxiv.2403.02310,
  doi = {10.48550/ARXIV.2403.02310},
  url = {https://arxiv.org/abs/2403.02310},
  author = {Agrawal, Amey and Kedia, Nitin and Panwar, Ashish and Mohan, Jayashree and Kwatra, Nipun and Gulavani, Bhargav S. and Tumanov, Alexey and Ramjee, Ramachandran},
  title = {Taming Throughput-Latency Tradeoff in LLM Inference with Sarathi-Serve},
  publisher = {arXiv},
  year = {2024},
  copyright = {arXiv.org perpetual, non-exclusive license}
}

@misc{https://doi.org/10.48550/arxiv.2405.04437,
  doi = {10.48550/ARXIV.2405.04437},
  url = {https://arxiv.org/abs/2405.04437},
  author = {Prabhu, Ramya and Nayak, Ajay and Mohan, Jayashree and Ramjee, Ramachandran and Panwar, Ashish},
  title = {vAttention: Dynamic Memory Management for Serving LLMs without PagedAttention},
  publisher = {arXiv},
  year = {2024},
  copyright = {Creative Commons Attribution 4.0 International}
}

@misc{https://doi.org/10.48550/arxiv.2501.01005,
  doi = {10.48550/ARXIV.2501.01005},
  url = {https://arxiv.org/abs/2501.01005},
  author = {Ye, Zihao and Chen, Lequn and Lai, Ruihang and Lin, Wuwei and Zhang, Yineng and Wang, Stephanie and Chen, Tianqi and Kasikci, Baris and Grover, Vinod and Krishnamurthy, Arvind and Ceze, Luis},
  title = {FlashInfer: Efficient and Customizable Attention Engine for LLM Inference Serving},
  publisher = {arXiv},
  year = {2025},
  copyright = {Creative Commons Attribution Share Alike 4.0 International}
}

@misc{https://doi.org/10.48550/arxiv.2207.00032,
  doi = {10.48550/ARXIV.2207.00032},
  url = {https://arxiv.org/abs/2207.00032},
  author = {Aminabadi, Reza Yazdani and Rajbhandari, Samyam and Zhang, Minjia and Awan, Ammar Ahmad and Li, Cheng and Li, Du and Zheng, Elton and Rasley, Jeff and Smith, Shaden and Ruwase, Olatunji and He, Yuxiong},
  title = {DeepSpeed Inference: Enabling Efficient Inference of Transformer Models at Unprecedented Scale},
  publisher = {arXiv},
  year = {2022},
  copyright = {arXiv.org perpetual, non-exclusive license}
}

@misc{https://doi.org/10.48550/arxiv.2303.06865,
  doi = {10.48550/ARXIV.2303.06865},
  url = {https://arxiv.org/abs/2303.06865},
  author = {Sheng, Ying and Zheng, Lianmin and Yuan, Binhang and Li, Zhuohan and Ryabinin, Max and Fu, Daniel Y. and Xie, Zhiqiang and Chen, Beidi and Barrett, Clark and Gonzalez, Joseph E. and Liang, Percy and R{\'e}, Christopher and Stoica, Ion and Zhang, Ce},
  title = {FlexGen: High-Throughput Generative Inference of Large Language Models with a Single GPU},
  publisher = {arXiv},
  year = {2023},
  copyright = {arXiv.org perpetual, non-exclusive license}
}

@misc{https://doi.org/10.48550/arxiv.2312.12456,
  doi = {10.48550/ARXIV.2312.12456},
  url = {https://arxiv.org/abs/2312.12456},
  author = {Song, Yixin and Mi, Zeyu and Xie, Haotong and Chen, Haibo},
  title = {PowerInfer: Fast Large Language Model Serving with a Consumer-grade GPU},
  publisher = {arXiv},
  year = {2023},
  copyright = {Creative Commons Attribution 4.0 International}
}

@misc{https://doi.org/10.48550/arxiv.2407.00023,
  doi = {10.48550/ARXIV.2407.00023},
  url = {https://arxiv.org/abs/2407.00023},
  author = {Srivatsa, Vikranth and He, Zijian and Abhyankar, Reyna and Li, Dongming and Zhang, Yiying},
  title = {Preble: Efficient Distributed Prompt Scheduling for LLM Serving},
  publisher = {arXiv},
  year = {2024},
  copyright = {Creative Commons Attribution 4.0 International}
}

@misc{https://doi.org/10.48550/arxiv.2309.17453,
  doi = {10.48550/ARXIV.2309.17453},
  url = {https://arxiv.org/abs/2309.17453},
  author = {Xiao, Guangxuan and Tian, Yuandong and Chen, Beidi and Han, Song and Lewis, Mike},
  title = {Efficient Streaming Language Models with Attention Sinks},
  publisher = {arXiv},
  year = {2023},
  copyright = {Creative Commons Attribution 4.0 International}
}

@misc{https://doi.org/10.48550/arxiv.2306.14048,
  doi = {10.48550/ARXIV.2306.14048},
  url = {https://arxiv.org/abs/2306.14048},
  author = {Zhang, Zhenyu and Sheng, Ying and Zhou, Tianyi and Chen, Tianlong and Zheng, Lianmin and Cai, Ruisi and Song, Zhao and Tian, Yuandong and R{\'e}, Christopher and Barrett, Clark and Wang, Zhangyang and Chen, Beidi},
  title = {H$_2$O: Heavy-Hitter Oracle for Efficient Generative Inference of Large Language Models},
  publisher = {arXiv},
  year = {2023},
  copyright = {arXiv.org perpetual, non-exclusive license}
}

@misc{https://doi.org/10.48550/arxiv.2305.17118,
  doi = {10.48550/ARXIV.2305.17118},
  url = {https://arxiv.org/abs/2305.17118},
  author = {Liu, Zichang and Desai, Aditya and Liao, Fangshuo and Wang, Weitao and Xie, Victor and Xu, Zhaozhuo and Kyrillidis, Anastasios and Shrivastava, Anshumali},
  title = {Scissorhands: Exploiting the Persistence of Importance Hypothesis for LLM KV Cache Compression at Test Time},
  publisher = {arXiv},
  year = {2023},
  copyright = {arXiv.org perpetual, non-exclusive license}
}

@misc{https://doi.org/10.48550/arxiv.2404.14469,
  doi = {10.48550/ARXIV.2404.14469},
  url = {https://arxiv.org/abs/2404.14469},
  author = {Li, Yuhong and Huang, Yingbing and Yang, Bowen and Venkitesh, Bharat and Locatelli, Acyr and Ye, Hanchen and Cai, Tianle and Lewis, Patrick and Chen, Deming},
  title = {SnapKV: LLM Knows What You are Looking for Before Generation},
  publisher = {arXiv},
  year = {2024},
  copyright = {Creative Commons Attribution 4.0 International}
}

@misc{https://doi.org/10.48550/arxiv.2406.02069,
  doi = {10.48550/ARXIV.2406.02069},
  url = {https://arxiv.org/abs/2406.02069},
  author = {Cai, Zefan and Zhang, Yichi and Gao, Bofei and Liu, Yuliang and Li, Yucheng and Liu, Tianyu and Lu, Keming and Xiong, Wayne and Dong, Yue and Hu, Junjie and Xiao, Wen},
  title = {PyramidKV: Dynamic KV Cache Compression based on Pyramidal Information Funneling},
  publisher = {arXiv},
  year = {2024},
  copyright = {arXiv.org perpetual, non-exclusive license}
}

@misc{https://doi.org/10.48550/arxiv.2406.10774,
  doi = {10.48550/ARXIV.2406.10774},
  url = {https://arxiv.org/abs/2406.10774},
  author = {Tang, Jiaming and Zhao, Yilong and Zhu, Kan and Xiao, Guangxuan and Kasikci, Baris and Han, Song},
  title = {Quest: Query-Aware Sparsity for Efficient Long-Context LLM Inference},
  publisher = {arXiv},
  year = {2024},
  copyright = {arXiv.org perpetual, non-exclusive license}
}

@article{https://doi.org/10.48550/arxiv.2402.02750,
  doi = {10.48550/ARXIV.2402.02750},
  url = {https://arxiv.org/abs/2402.02750},
  author = {Liu, Zirui and Yuan, Jiayi and Jin, Hongye and Zhong, Shaochen and Xu, Zhaozhuo and Braverman, Vladimir and Chen, Beidi and Hu, Xia},
  title = {KIVI: A Tuning-Free Asymmetric 2bit Quantization for KV Cache},
  publisher = {arXiv},
  year = {2024},
  copyright = {Creative Commons Attribution 4.0 International}
}

@misc{https://doi.org/10.48550/arxiv.2405.14366,
  doi = {10.48550/ARXIV.2405.14366},
  url = {https://arxiv.org/abs/2405.14366},
  author = {Liu, Akide and Liu, Jing and Pan, Zizheng and He, Yefei and Haffari, Gholamreza and Zhuang, Bohan},
  title = {MiniCache: KV Cache Compression in Depth Dimension for Large Language Models},
  publisher = {arXiv},
  year = {2024},
  copyright = {Creative Commons Attribution 4.0 International}
}

@misc{https://doi.org/10.48550/arxiv.2310.07240,
  doi = {10.48550/ARXIV.2310.07240},
  url = {https://arxiv.org/abs/2310.07240},
  author = {Liu, Yuhan and Li, Hanchen and Cheng, Yihua and Ray, Siddhant and Huang, Yuyang and Zhang, Qizheng and Du, Kuntai and Yao, Jiayi and Lu, Shan and Ananthanarayanan, Ganesh and Maire, Michael and Hoffmann, Henry and Holtzman, Ari and Jiang, Junchen},
  title = {CacheGen: KV Cache Compression and Streaming for Fast Large Language Model Serving},
  publisher = {arXiv},
  year = {2023},
  copyright = {arXiv.org perpetual, non-exclusive license}
}

@misc{https://doi.org/10.48550/arxiv.2405.16444,
  doi = {10.48550/ARXIV.2405.16444},
  url = {https://arxiv.org/abs/2405.16444},
  author = {Yao, Jiayi and Li, Hanchen and Liu, Yuhan and Ray, Siddhant and Cheng, Yihua and Zhang, Qizheng and Du, Kuntai and Lu, Shan and Jiang, Junchen},
  title = {CacheBlend: Fast Large Language Model Serving for RAG with Cached Knowledge Fusion},
  publisher = {arXiv},
  year = {2024},
  copyright = {Creative Commons Attribution 4.0 International}
}

@misc{https://doi.org/10.48550/arxiv.2311.04934,
  doi = {10.48550/ARXIV.2311.04934},
  url = {https://arxiv.org/abs/2311.04934},
  author = {Gim, In and Chen, Guojun and Lee, Seung-seob and Sarda, Nikhil and Khandelwal, Anurag and Zhong, Lin},
  title = {Prompt Cache: Modular Attention Reuse for Low-Latency Inference},
  publisher = {arXiv},
  year = {2023},
  copyright = {arXiv.org perpetual, non-exclusive license}
}

@misc{https://doi.org/10.48550/arxiv.2402.05099,
  doi = {10.48550/ARXIV.2402.05099},
  url = {https://arxiv.org/abs/2402.05099},
  author = {Juravsky, Jordan and Brown, Bradley and Ehrlich, Ryan and Fu, Daniel Y. and R{\'e}, Christopher and Mirhoseini, Azalia},
  title = {Hydragen: High-Throughput LLM Inference with Shared Prefixes},
  publisher = {arXiv},
  year = {2024},
  copyright = {Creative Commons Attribution 4.0 International}
}

@misc{https://doi.org/10.48550/arxiv.2402.15220,
  doi = {10.48550/ARXIV.2402.15220},
  url = {https://arxiv.org/abs/2402.15220},
  author = {Ye, Lu and Tao, Ze and Huang, Yong and Li, Yang},
  title = {ChunkAttention: Efficient Self-Attention with Prefix-Aware KV Cache and Two-Phase Partition},
  publisher = {arXiv},
  year = {2024},
  copyright = {Creative Commons Attribution 4.0 International}
}

@misc{https://doi.org/10.48550/arxiv.2404.04793,
  doi = {10.48550/ARXIV.2404.04793},
  url = {https://arxiv.org/abs/2404.04793},
  author = {Wang, Zihao and Cui, Bin and Gan, Shaoduo},
  title = {SqueezeAttention: 2D Management of KV-Cache in LLM Inference via Layer-wise Optimal Budget},
  publisher = {arXiv},
  year = {2024},
  copyright = {Creative Commons Attribution Share Alike 4.0 International}
}

@misc{https://doi.org/10.48550/arxiv.2310.01801,
  doi = {10.48550/ARXIV.2310.01801},
  url = {https://arxiv.org/abs/2310.01801},
  author = {Ge, Suyu and Zhang, Yunan and Liu, Liyuan and Zhang, Minjia and Han, Jiawei and Gao, Jianfeng},
  title = {Model Tells You What to Discard: Adaptive KV Cache Compression for LLMs},
  publisher = {arXiv},
  year = {2023},
  copyright = {Creative Commons Attribution 4.0 International}
}

@misc{https://doi.org/10.48550/arxiv.2312.04985,
  doi = {10.48550/ARXIV.2312.04985},
  url = {https://arxiv.org/abs/2312.04985},
  author = {Ribar, Luka and Chelombiev, Ivan and Hudlass-Galley, Luke and Blake, Charlie and Luschi, Carlo and Orr, Douglas},
  title = {SparQ Attention: Bandwidth-Efficient LLM Inference},
  publisher = {arXiv},
  year = {2023},
  copyright = {Creative Commons Attribution Share Alike 4.0 International}
}

@misc{https://doi.org/10.48550/arxiv.2211.17192,
  doi = {10.48550/ARXIV.2211.17192},
  url = {https://arxiv.org/abs/2211.17192},
  author = {Leviathan, Yaniv and Kalman, Matan and Matias, Yossi},
  title = {Fast Inference from Transformers via Speculative Decoding},
  publisher = {arXiv},
  year = {2022},
  copyright = {Creative Commons Attribution 4.0 International}
}

@misc{https://doi.org/10.48550/arxiv.2302.01318,
  doi = {10.48550/ARXIV.2302.01318},
  url = {https://arxiv.org/abs/2302.01318},
  author = {Chen, Charlie and Borgeaud, Sebastian and Irving, Geoffrey and Lespiau, Jean-Baptiste and Sifre, Laurent and Jumper, John},
  title = {Accelerating Large Language Model Decoding with Speculative Sampling},
  publisher = {arXiv},
  year = {2023},
  copyright = {Creative Commons Attribution 4.0 International}
}

@misc{https://doi.org/10.48550/arxiv.2401.10774,
  doi = {10.48550/ARXIV.2401.10774},
  url = {https://arxiv.org/abs/2401.10774},
  author = {Cai, Tianle and Li, Yuhong and Geng, Zhengyang and Peng, Hongwu and Lee, Jason D. and Chen, Deming and Dao, Tri},
  title = {Medusa: Simple LLM Inference Acceleration Framework with Multiple Decoding Heads},
  publisher = {arXiv},
  year = {2024},
  copyright = {Creative Commons Attribution 4.0 International}
}

@misc{https://doi.org/10.48550/arxiv.2401.15077,
  doi = {10.48550/ARXIV.2401.15077},
  url = {https://arxiv.org/abs/2401.15077},
  author = {Li, Yuhui and Wei, Fangyun and Zhang, Chao and Zhang, Hongyang},
  title = {EAGLE: Speculative Sampling Requires Rethinking Feature Uncertainty},
  publisher = {arXiv},
  year = {2024},
  copyright = {arXiv.org perpetual, non-exclusive license}
}

@misc{https://doi.org/10.48550/arxiv.2406.16858,
  doi = {10.48550/ARXIV.2406.16858},
  url = {https://arxiv.org/abs/2406.16858},
  author = {Li, Yuhui and Wei, Fangyun and Zhang, Chao and Zhang, Hongyang},
  title = {EAGLE-2: Faster Inference of Language Models with Dynamic Draft Trees},
  publisher = {arXiv},
  year = {2024},
  copyright = {arXiv.org perpetual, non-exclusive license}
}

@misc{https://doi.org/10.48550/arxiv.2503.01840,
  doi = {10.48550/ARXIV.2503.01840},
  url = {https://arxiv.org/abs/2503.01840},
  author = {Li, Yuhui and Wei, Fangyun and Zhang, Chao and Zhang, Hongyang},
  title = {EAGLE-3: Scaling up Inference Acceleration of Large Language Models via Training-Time Test},
  publisher = {arXiv},
  year = {2025},
  copyright = {arXiv.org perpetual, non-exclusive license}
}

@misc{https://doi.org/10.48550/arxiv.2402.02057,
  doi = {10.48550/ARXIV.2402.02057},
  url = {https://arxiv.org/abs/2402.02057},
  author = {Fu, Yichao and Bailis, Peter and Stoica, Ion and Zhang, Hao},
  title = {Break the Sequential Dependency of LLM Inference Using Lookahead Decoding},
  publisher = {arXiv},
  year = {2024},
  copyright = {arXiv.org perpetual, non-exclusive license}
}

@inproceedings{https://doi.org/10.48550/arxiv.2305.09781, series={ASPLOS '24}, title={SpecInfer: Accelerating Large Language Model Serving with Tree-based Speculative Inference and Verification}, url={http://dx.doi.org/10.1145/3620666.3651335}, DOI={10.1145/3620666.3651335}, booktitle={Proceedings of the 29th ACM International Conference on Architectural Support for Programming Languages and Operating Systems, Volume 3}, publisher={ACM}, author={Miao, Xupeng and Oliaro, Gabriele and Zhang, Zhihao and Cheng, Xinhao and Wang, Zeyu and Zhang, Zhengxin and Wong, Rae Ying Yee and Zhu, Alan and Yang, Lijie and Shi, Xiaoxiang and Shi, Chunan and Chen, Zhuoming and Arfeen, Daiyaan and Abhyankar, Reyna and Jia, Zhihao}, year={2024}, month=Apr, pages={932--949}, collection={ASPLOS '24} }

@misc{https://doi.org/10.48550/arxiv.2311.08252,
  doi = {10.48550/ARXIV.2311.08252},
  url = {https://arxiv.org/abs/2311.08252},
  author = {He, Zhenyu and Zhong, Zexuan and Cai, Tianle and Lee, Jason D. and He, Di},
  title = {REST: Retrieval-Based Speculative Decoding},
  publisher = {arXiv},
  year = {2023},
  copyright = {arXiv.org perpetual, non-exclusive license}
}

@inproceedings{https://doi.org/10.48550/arxiv.2309.08168, title={Draft\& Verify: Lossless Large Language Model Acceleration via Self-Speculative Decoding}, url={http://dx.doi.org/10.18653/v1/2024.acl-long.607}, DOI={10.18653/v1/2024.acl-long.607}, booktitle={Proceedings of the 62nd Annual Meeting of the Association for Computational Linguistics (Volume 1: Long Papers)}, publisher={Association for Computational Linguistics}, author={Zhang, Jun and Wang, Jue and Li, Huan and Shou, Lidan and Chen, Ke and Chen, Gang and Mehrotra, Sharad}, year={2024}, pages={11263--11282} }

@inproceedings{https://doi.org/10.48550/arxiv.2404.16710, title={LayerSkip: Enabling Early Exit Inference and Self-Speculative Decoding}, url={http://dx.doi.org/10.18653/v1/2024.acl-long.681}, DOI={10.18653/v1/2024.acl-long.681}, booktitle={Proceedings of the 62nd Annual Meeting of the Association for Computational Linguistics (Volume 1: Long Papers)}, publisher={Association for Computational Linguistics}, author={Elhoushi, Mostafa and Shrivastava, Akshat and Liskovich, Diana and Hosmer, Basil and Wasti, Bram and Lai, Liangzhen and Mahmoud, Anas and Acun, Bilge and Agarwal, Saurabh and Roman, Ahmed and Aly, Ahmed and Chen, Beidi and Wu, Carole-Jean}, year={2024}, pages={12622--12642} }

@misc{https://doi.org/10.48550/arxiv.2402.11131,
  doi = {10.48550/ARXIV.2402.11131},
  url = {https://arxiv.org/abs/2402.11131},
  author = {Bhendawade, Nikhil and Belousova, Irina and Fu, Qichen and Mason, Henry and Rastegari, Mohammad and Najibi, Mahyar},
  title = {Speculative Streaming: Fast LLM Inference without Auxiliary Models},
  publisher = {arXiv},
  year = {2024},
  copyright = {arXiv.org perpetual, non-exclusive license}
}

@misc{https://doi.org/10.48550/arxiv.2403.09919,
  doi = {10.48550/ARXIV.2403.09919},
  url = {https://arxiv.org/abs/2403.09919},
  author = {Cheng, Yunfei and Zhang, Aonan and Zhang, Xuanyu and Wang, Chong and Wang, Yi},
  title = {Recurrent Drafter for Fast Speculative Decoding in Large Language Models},
  publisher = {arXiv},
  year = {2024},
  copyright = {arXiv.org perpetual, non-exclusive license}
}

@misc{https://doi.org/10.48550/arxiv.1711.02281,
  doi = {10.48550/ARXIV.1711.02281},
  url = {https://arxiv.org/abs/1711.02281},
  author = {Gu, Jiatao and Bradbury, James and Xiong, Caiming and Li, Victor O. K. and Socher, Richard},
  title = {Non-Autoregressive Neural Machine Translation},
  publisher = {arXiv},
  year = {2017},
  copyright = {arXiv.org perpetual, non-exclusive license}
}

@misc{https://doi.org/10.48550/arxiv.1904.09324,
  doi = {10.48550/ARXIV.1904.09324},
  url = {https://arxiv.org/abs/1904.09324},
  author = {Ghazvininejad, Marjan and Levy, Omer and Liu, Yinhan and Zettlemoyer, Luke},
  title = {Mask-Predict: Parallel Decoding of Conditional Masked Language Models},
  publisher = {arXiv},
  year = {2019},
  copyright = {arXiv.org perpetual, non-exclusive license}
}

@misc{https://doi.org/10.48550/arxiv.1905.11006,
  doi = {10.48550/ARXIV.1905.11006},
  url = {https://arxiv.org/abs/1905.11006},
  author = {Gu, Jiatao and Wang, Changhan and Zhao, Jake},
  title = {Levenshtein Transformer},
  publisher = {arXiv},
  year = {2019},
  copyright = {arXiv.org perpetual, non-exclusive license}
}

@misc{https://doi.org/10.48550/arxiv.1902.03249,
  doi = {10.48550/ARXIV.1902.03249},
  url = {https://arxiv.org/abs/1902.03249},
  author = {Stern, Mitchell and Chan, William and Kiros, Jamie and Uszkoreit, Jakob},
  title = {Insertion Transformer: Flexible Sequence Generation via Insertion Operations},
  publisher = {arXiv},
  year = {2019},
  copyright = {arXiv.org perpetual, non-exclusive license}
}

@misc{https://doi.org/10.48550/arxiv.1808.08583,
  doi = {10.48550/ARXIV.1808.08583},
  url = {https://arxiv.org/abs/1808.08583},
  author = {Wang, Chunqi and Zhang, Ji and Chen, Haiqing},
  title = {Semi-Autoregressive Neural Machine Translation},
  publisher = {arXiv},
  year = {2018},
  copyright = {arXiv.org perpetual, non-exclusive license}
}

@misc{https://doi.org/10.48550/arxiv.2205.14217,
  doi = {10.48550/ARXIV.2205.14217},
  url = {https://arxiv.org/abs/2205.14217},
  author = {Li, Xiang Lisa and Thickstun, John and Gulrajani, Ishaan and Liang, Percy and Hashimoto, Tatsunori B.},
  title = {Diffusion-LM Improves Controllable Text Generation},
  publisher = {arXiv},
  year = {2022},
  copyright = {Creative Commons Attribution 4.0 International}
}

@misc{https://doi.org/10.48550/arxiv.2310.16834,
  doi = {10.48550/ARXIV.2310.16834},
  url = {https://arxiv.org/abs/2310.16834},
  author = {Lou, Aaron and Meng, Chenlin and Ermon, Stefano},
  title = {Discrete Diffusion Modeling by Estimating the Ratios of the Data Distribution},
  publisher = {arXiv},
  year = {2023},
  copyright = {Creative Commons Attribution 4.0 International}
}

@misc{https://doi.org/10.48550/arxiv.2406.07524,
  doi = {10.48550/ARXIV.2406.07524},
  url = {https://arxiv.org/abs/2406.07524},
  author = {Sahoo, Subham Sekhar and Arriola, Marianne and Schiff, Yair and Gokaslan, Aaron and Marroquin, Edgar and Chiu, Justin T and Rush, Alexander and Kuleshov, Volodymyr},
  title = {Simple and Effective Masked Diffusion Language Models},
  publisher = {arXiv},
  year = {2024},
  copyright = {Creative Commons Attribution 4.0 International}
}

@misc{https://doi.org/10.48550/arxiv.2502.09992,
  doi = {10.48550/ARXIV.2502.09992},
  url = {https://arxiv.org/abs/2502.09992},
  author = {Nie, Shen and Zhu, Fengqi and You, Zebin and Zhang, Xiaolu and Ou, Jingyang and Hu, Jun and Zhou, Jun and Lin, Yankai and Wen, Ji-Rong and Li, Chongxuan},
  title = {Large Language Diffusion Models},
  publisher = {arXiv},
  year = {2025},
  copyright = {arXiv.org perpetual, non-exclusive license}
}

@misc{https://doi.org/10.48550/arxiv.2109.05093,
  doi = {10.48550/ARXIV.2109.05093},
  url = {https://arxiv.org/abs/2109.05093},
  author = {Scholak, Torsten and Schucher, Nathan and Bahdanau, Dzmitry},
  title = {PICARD: Parsing Incrementally for Constrained Auto-Regressive Decoding from Language Models},
  publisher = {arXiv},
  year = {2021},
  copyright = {Creative Commons Attribution Non Commercial Share Alike 4.0 International}
}

@misc{https://doi.org/10.48550/arxiv.2201.11227,
  doi = {10.48550/ARXIV.2201.11227},
  url = {https://arxiv.org/abs/2201.11227},
  author = {Poesia, Gabriel and Polozov, Oleksandr and Le, Vu and Tiwari, Ashish and Soares, Gustavo and Meek, Christopher and Gulwani, Sumit},
  title = {Synchromesh: Reliable code generation from pre-trained language models},
  publisher = {arXiv},
  year = {2022},
  copyright = {Creative Commons Attribution Non Commercial No Derivatives 4.0 International}
}

@misc{https://doi.org/10.48550/arxiv.2501.10868,
  doi = {10.48550/ARXIV.2501.10868},
  url = {https://arxiv.org/abs/2501.10868},
  author = {Geng, Saibo and Cooper, Hudson and Moskal, Micha{\l{}} and Jenkins, Samuel and Berman, Julian and Ranchin, Nathan and West, Robert and Horvitz, Eric and Nori, Harsha},
  title = {JSONSchemaBench: A Rigorous Benchmark of Structured Outputs for Language Models},
  publisher = {arXiv},
  year = {2025},
  copyright = {Creative Commons Attribution 4.0 International}
}

@misc{https://doi.org/10.48550/arxiv.2310.11324,
  doi = {10.48550/ARXIV.2310.11324},
  url = {https://arxiv.org/abs/2310.11324},
  author = {Sclar, Melanie and Choi, Yejin and Tsvetkov, Yulia and Suhr, Alane},
  title = {Quantifying Language Models' Sensitivity to Spurious Features in Prompt Design or: How I learned to start worrying about prompt formatting},
  publisher = {arXiv},
  year = {2023},
  copyright = {Creative Commons Attribution 4.0 International}
}

@misc{https://doi.org/10.48550/arxiv.2408.02442,
  doi = {10.48550/ARXIV.2408.02442},
  url = {https://arxiv.org/abs/2408.02442},
  author = {Tam, Zhi Rui and Wu, Cheng-Kuang and Tsai, Yi-Lin and Lin, Chieh-Yen and Lee, Hung-yi and Chen, Yun-Nung},
  title = {Let Me Speak Freely? A Study on the Impact of Format Restrictions on Performance of Large Language Models},
  publisher = {arXiv},
  year = {2024},
  copyright = {Creative Commons Attribution Share Alike 4.0 International}
}

@misc{https://doi.org/10.48550/arxiv.2307.03172,
  doi = {10.48550/ARXIV.2307.03172},
  url = {https://arxiv.org/abs/2307.03172},
  author = {Liu, Nelson F. and Lin, Kevin and Hewitt, John and Paranjape, Ashwin and Bevilacqua, Michele and Petroni, Fabio and Liang, Percy},
  title = {Lost in the Middle: How Language Models Use Long Contexts},
  publisher = {arXiv},
  year = {2023},
  copyright = {arXiv.org perpetual, non-exclusive license}
}

@misc{https://doi.org/10.48550/arxiv.2102.09690,
  doi = {10.48550/ARXIV.2102.09690},
  url = {https://arxiv.org/abs/2102.09690},
  author = {Zhao, Tony Z. and Wallace, Eric and Feng, Shi and Klein, Dan and Singh, Sameer},
  title = {Calibrate Before Use: Improving Few-Shot Performance of Language Models},
  publisher = {arXiv},
  year = {2021},
  copyright = {arXiv.org perpetual, non-exclusive license}
}

@misc{https://doi.org/10.48550/arxiv.2104.08786,
  doi = {10.48550/ARXIV.2104.08786},
  url = {https://arxiv.org/abs/2104.08786},
  author = {Lu, Yao and Bartolo, Max and Moore, Alastair and Riedel, Sebastian and Stenetorp, Pontus},
  title = {Fantastically Ordered Prompts and Where to Find Them: Overcoming Few-Shot Prompt Order Sensitivity},
  publisher = {arXiv},
  year = {2021},
  copyright = {Creative Commons Attribution 4.0 International}
}

@misc{https://doi.org/10.48550/arxiv.2309.03882,
  doi = {10.48550/ARXIV.2309.03882},
  url = {https://arxiv.org/abs/2309.03882},
  author = {Zheng, Chujie and Zhou, Hao and Meng, Fandong and Zhou, Jie and Huang, Minlie},
  title = {Large Language Models Are Not Robust Multiple Choice Selectors},
  publisher = {arXiv},
  year = {2023},
  copyright = {arXiv.org perpetual, non-exclusive license}
}
